\pdfoutput=1 
\documentclass[11pt]{article}

\usepackage[font=libertinus, citestyle=numeric, status={Accepted manuscript}]{kurbanlab}

\DeclareAffiliation{hbku}{%
  College of Science and Engineering, Hamad Bin Khalifa University, Doha, Qatar}

\DeclareAffiliation{tamu}{%
  Department of Computer and Electrical Engineering,
  Texas A\&M University, College Station, TX, USA}

\DeclareAffiliation{iub}{%
  Luddy School of Informatics, Computing, and Engineering,
  Indiana University Bloomington, Bloomington, IN, USA}

\DeclareAffiliation{ankara}{%
  Ankara University, Ankara, Turkey}

\usepackage{float}
\usepackage{ragged2e}
\graphicspath{{figures/oracle_examples/}{figures/}{./}}
\usetikzlibrary{arrows.meta,positioning,calc,shapes.geometric,fit,backgrounds,%
                decorations.pathreplacing}

\definecolor{npblue}{HTML}{0072B2}
\definecolor{nporange}{HTML}{E69F00}
\definecolor{npgreen}{HTML}{009E73}
\definecolor{npred}{HTML}{D55E00}
\definecolor{nppurple}{HTML}{CC79A7}
\definecolor{npslate}{HTML}{3A4A5A}
\tikzset{
  proc/.style={rounded corners=2pt,draw=npblue!55,fill=npblue!8,line width=0.6pt,align=center,inner sep=3.5pt,font=\small},
  data/.style={rounded corners=2pt,draw=black!35,fill=black!4,line width=0.6pt,align=center,inner sep=3.5pt,font=\small},
  good/.style={rounded corners=2pt,draw=npgreen!65,fill=npgreen!12,line width=0.6pt,align=center,inner sep=3pt,font=\small},
  bad/.style={rounded corners=2pt,draw=npred!65,fill=npred!10,line width=0.6pt,align=center,inner sep=3pt,font=\small},
  acc/.style={rounded corners=2pt,draw=nporange!70,fill=nporange!12,line width=0.6pt,align=center,inner sep=3pt,font=\small},
  flow/.style={-{Latex[length=2.0mm,width=1.6mm]},line width=0.7pt,draw=black!60},
  plab/.style={font=\sffamily\bfseries\large},
  card/.style={draw=black!22,line width=0.6pt,fill=white,rounded corners=2pt},
  ax/.style={->,line width=0.5pt,draw=black!50,>=stealth},
}

\title{Physics as the label for measuring and correcting materials reasoning in
multimodal models}
\RunningTitle{MatPCR: physics as the label for materials reasoning}

\Author{Hasan Kurban}{hbku}
\Author{Rasul Khanbayov}{hbku}
\Author[corresponding=kurbanm@ankara.edu.tr]{Mustafa Kurban}{ankara}

\CodeURL{https://github.com/KurbanIntelligenceLab/matpcr}
\ArxivID{2609.12181}
\Venue{Knowledge-Based Systems 355 (2027) 117295}

\begin{document}
\maketitle
\begingroup\renewcommand{\thefootnote}{}\footnotetext{Accepted manuscript of the article
published in \emph{Knowledge-Based Systems} 355 (2027) 117295,
\url{https://doi.org/10.1016/j.knosys.2026.117295}. \textcopyright~2026. This manuscript version is
made available under the CC-BY-NC-ND 4.0 license
\url{https://creativecommons.org/licenses/by-nc-nd/4.0/}.}\endgroup

\begin{abstract}
Vision--language and language models increasingly interpret materials data, yet most benchmarks score final answers rather than whether the quantities a model states while reasoning obey established physical constraints. We introduce MatPCR, a label-free benchmark that uses six programmatic oracle families spanning diffraction geometry, microscopy scale, Raman peak structure, and Materials Project-grounded checks of near-hull stability, computed band-gap class, and net magnetization to test the physical admissibility of each checkable quantity a reasoning chain states, summarized as the Physical-Consistency Rate (PCR). The oracles are necessary-condition checks; their precision is audited by human review with interval estimates, and their recall on real responses is not assumed. Across nine models and 4,334 checkable records, overall PCR ranges from 0.556 to 0.778, and paired item-level tests separate only differences larger than about 0.07. Constraint-Grounded Self-Verification reduces oracle-detected violations in all four tested models; paired before--after scoring shows that it also raises answer correctness, but that most revisions which clear the oracle move to an admissible rather than an exact value. An open verifier reaches ROC AUC = 0.850 in distribution, varies strongly across constraint families, and does not transfer at a useful level to any held-out constraint family. An exact identity quantifies how oracle error displaces the reported PCR.
\end{abstract}

\printkeywords

\noindent\begin{tcolorbox}[colback=gray!6,colframe=gray!50,boxrule=0.5pt]
\textbf{Key contributions.} (1) We reframe materials-model evaluation from answer accuracy to the physical
admissibility of the quantities a reasoning chain states, checkable without human labels for a large,
well-defined class of physical errors, across vision--language and language inputs; in the present benchmark each chain exposes one checkable quantity, so PCR scores claim-level admissibility within the chain, not the validity of every inferential step. (2) We release MatPCR: programmatic,
modality-specific physical oracles including DFT-grounded property checks. (3) We
define PCR, a metric with bootstrap confidence intervals and a paired significance
test, and prove an exact identity for the bias an imperfect oracle induces, which fixes
what the oracle-validity gate must establish for the reported rate to be trustworthy.
(4) We fine-tune and release an open physical-consistency verifier, a useful
in-distribution detector for some constraint families but not others, whose transfer to six held-out constraint types is at or near
chance,
and benchmark it against zero-shot LLM-judge and self-confidence baselines. (5) We
propose CGSV, a label-free agentic self-verification loop, isolate its effect from
unaided self-refinement and equal-compute re-prompting controls, and measure, on the same items before and after revision, both its effect on violations and its effect on answer correctness.
\end{tcolorbox}

\section{Introduction}
Materials data is governed by physics that is, in large part, programmatically
checkable. A micrograph's scale bar fixes its spatial calibration; a powder
diffraction pattern's peak positions follow from a crystal structure and the source
through Bragg's law; a crystal's thermodynamic stability is set by its energy
relative to a convex hull; and a computed property has a definite sign and order of
magnitude. Vision--language models (VLMs), built on open instruction-tuned
architectures\cite{liu2024llava} as well as closed frontier systems, and language
models (LLMs) are now
routinely asked to reason over such data step by step, in the free-form
chain-of-thought style now standard for eliciting multi-step reasoning from these
models,\cite{wei2022cot} and a large literature evaluates them on
materials tasks;\cite{lai2025materials,npj2026vlmeval,choudhary2025microscopygpt,li2024mmsci}
classical computer-vision pipelines for microstructure and defect analysis predate
them.\cite{holm2020cv,li2018defect}

The problem this creates is now well documented. Materials benchmarks report that
general-purpose models hallucinate invalid property values\cite{llm4matbench2024,lai2025materials}
and benchmark their spatial and structural reasoning over crystals,\cite{atomworld2025}
motivating purpose-built multimodal materials models.\cite{matterchat2026} Beyond
materials, recent work shows vision--language models violate basic physical law,
for example failing conservation under transformation;\cite{conservationbench2026}
in response, open vision--language models now add physics-grounded reinforcement
learning and inference-time self-verification to strengthen physics-problem
reasoning.\cite{p1vl2026} Across domains, a shared response has emerged:
a model's reasoning should be audited against laws or logic in a
reference-free way, whether through conservation-grounded metrics or deterministic
formal checks.\cite{conservationbench2026,neurosymclinical2026}

That literature, like multimodal evaluation generally, scores the final answer
against human labels. A parallel and fast-growing line builds discovery agents that
couple language models with simulation and first-principles tools to propose and
screen candidates;\cite{dreams2025,crystalyse2025,mapps2025,atomagents2024,llmatdesign2024,matpc2025,kang2024chatmof,zhang2024honeycomb,survey2025aimat,eaa2026,matty2026}
some verify the physics of final proposals, for example through physics-aware
rejection sampling,\cite{physrej2025} mitigate hallucination during DFT
execution,\cite{dreams2025} or pre-screen candidates with machine-learned
interatomic potentials before DFT.\cite{genbaselines2025,chgnet2023} General-domain
methods verify chains of thought via visual grounding or generic
programs,\cite{corgi2025,mmverify2025,vlmr3_2026,prove2024} but against visual
evidence or generic programs rather than domain physics, and on general benchmarks
rather than materials. A third line builds verifiers and process-reward models that
score reasoning steps;\cite{processbench2025,compassverifier2025,t1verify2025} these
have driven reinforcement learning from verifiable rewards in mathematics and code,
where correctness is programmatically checkable,\cite{deepseekr1_2025} but
verification has remained under-specified for scientific reasoning, where answers are
free-form and references are scarce.\cite{scirlvr2026} Across all three lines, whether
the quantities a model's reasoning chain states about materials data are physically
admissible is neither
measured nor made reusable. Yet that is what
determines whether a model can be trusted inside an analysis or screening pipeline:
a chain that reads a 100~nm scale bar and then reports a 10~$\mu$m feature, assigns
a peak to a $2\theta$ the stated phase cannot produce, or calls a structure stable
that lies far above the convex hull, is wrong on its face. Checking the chain against
external physics is also less susceptible than outcome- or judge-based verification to a
plausible but physically inadmissible chain: a model's stated reasoning need not
faithfully reflect its computation, so an unfaithful trace can satisfy a final-answer or
judge check while resting on physics that does not hold, whereas a programmatic oracle
tests the admissibility of the stated claims directly.\cite{faithflow2026}

Neighbouring lines of work help delimit what such a check can and cannot establish, and we
draw on them as conceptual background rather than as comparators for the metric. The faithfulness of
a stated chain to the computation behind an answer is now measured directly: FaithCoT-Bench frames
instance-level faithfulness as a detection task over expert-annotated trajectories and finds it
hardest in knowledge-intensive domains.\cite{faithcotbench2026} MatPCR asks a different question,
whether the quantities a chain states are physically admissible, and this question remains
informative whether or not the chain is faithful, so the two signals complement each other. Work that
improves scientific reasoning chains rather than evaluating them suggests a further use of such a
signal: CoT-Evo evolves and distils chain-of-thought trajectories under a fitness that rewards
correctness, coherence and the use of domain knowledge,\cite{cotevo2026} and a physics-grounded
admissibility check could in principle serve as one component of that fitness, although we do not
test this here. Studies of how vision--language models perceive images beyond aggregate benchmark
scores, such as the strong prompt and model dependence reported for multimodal sarcasm
perception,\cite{sarcasm2026lvlm} parallel our own observation that which responses are checkable,
and therefore scored, shifts with prompt wording (Results). Surveys of knowledge infusion in
specialised domains such as law\cite{legalinfusion2026} likewise distinguish injecting domain
knowledge into a model from verifying how it is used; MatPCR belongs to the verification stage for
materials physics and injects no knowledge into the evaluated model. None of these works measures the
physical admissibility of the quantities a model states about materials characterization or
first-principles data, and none provides a materials-specific oracle, a correction loop driven by
such an oracle, or a verifier trained on its judgments.

We make one observation: materials data contains directly checkable physical
constraints for a large class of reasoning errors, and first-principles databases
extend these constraints from geometry to computed materials properties. These constraints are
checkable with no human annotation, yielding a label-free signal that scales with the
available public data and that applies equally to image inputs (VLMs) and
structure or text inputs (LLMs). This brings to materials the reference-free correctness
signal that has driven progress in mathematics and code, though here the oracles are
imperfect-precision necessary conditions rather than the effectively exact checkers of
those domains. We turn this into MatPCR, a benchmark and a set of
oracles including DFT-grounded property checks; the Physical-Consistency Rate (PCR);
Constraint-Grounded Self-Verification (CGSV); and an open, fine-tuned verifier
released as a reusable artifact (Fig.~\ref{fig:teaser}).

The distinction from each neighboring line is precise. Against materials benchmarks
that score answers or property predictions,\cite{lai2025materials,llm4matbench2024,atomworld2025,matterchat2026}
we score the physical admissibility of the quantities stated in the reasoning chain rather than the final
answer, with a label-free oracle. Against physical-law evaluations of model
reasoning,\cite{conservationbench2026,p1vl2026} we ground the check in materials
characterization and first-principles DFT, span both image and structure inputs, and
release a verifier and a label-free self-verification loop rather than only a diagnostic.
Against discovery agents,\cite{dreams2025,physrej2025} we measure the admissibility of
quantities stated within the chain and act on oracle-flagged violations, rather than screen final proposals or
run DFT. Against general verifiers
and process-reward models,\cite{processbench2025,compassverifier2025,t1verify2025}
including recent step-level verifiers for scientific or structured reasoning,\cite{sciprm2026,vprm2026}
we supply a domain-specific, physics- and DFT-grounded, label-free signal for a setting
where verification was previously under-specified. Those verifiers need an answer key, a
tool call, or gold reasoning annotation; our oracle needs none of these. The two closest contemporaneous
efforts use physics as a programmatic signal for materials models but stop short of the
reasoning chain: one benchmarks vision--language models on X-ray-diffraction peak
indexing, grading the final indexed answer against structure-derived labels on the
diffraction modality alone,\cite{crystalxrdbench2026} and the other uses DFT-computed
energies and band gaps as a label-free reward, but as a binary exact-match signal on the
final answer during reinforcement-learning training rather than a chain-level
consistency measure at evaluation and inference time.\cite{matpref2026} MatPCR differs
from both in scoring quantities stated in the intermediate chain rather than only the final answer, spanning four
oracle modalities, and closing a label-free self-verification loop with a released
verifier. To our knowledge no single prior effort combines a label-free,
chain-level (stated-quantity), multi-oracle, materials-specific consistency benchmark with a
self-verification loop and a released verifier. We note the scope this implies: every item
in the present benchmark exposes one checkable quantity, so PCR tests the admissibility of that
stated quantity, not the logical validity of every step that led to it; chains that state several
interdependent quantities are supported by the metric's definition (Methods) but are left to
future benchmark versions.

This work also extends a line of research by the present authors on multimodal
resources and benchmarks for crystalline
materials,\cite{polat2025crysmtm,polat2026extrapolation} on machine-learning property
prediction for nanomaterials,\cite{polat2025quantumshellnet,polat2024multimodalnp} on
benchmarking model hallucination,\cite{abdaljalil2025halluverse} and on
physics-grounded benchmarks for dynamic systems.\cite{khanbayov2026iris} MatPCR is
distinct from each: rather than providing a dataset, scoring generation or property
prediction, or measuring text factuality, it measures the physical consistency of the
reasoning chain itself and acts on the violations it finds, with label-free, DFT-grounded
oracles.

\begin{figure}[t]
\centering
\adjustbox{max size={\linewidth}{0.66\textheight}}{%
\begin{tikzpicture}[font=\sffamily]
\node[plab,anchor=west] at (-0.10,8.85){a)};
\node[font=\sffamily\small\bfseries,anchor=west] at (0.55,8.85){Materials data fixes quantities a reasoning chain can be checked against};
\node[card,minimum width=3.15cm,minimum height=2.65cm,anchor=south west] at (0.0,5.7){};
\node[card,minimum width=3.15cm,minimum height=2.65cm,anchor=south west] at (3.3,5.7){};
\node[card,minimum width=3.15cm,minimum height=2.65cm,anchor=south west] at (6.6,5.7){};
\node[card,minimum width=3.15cm,minimum height=2.65cm,anchor=south west] at (9.9,5.7){};
\node[font=\sffamily\fontsize{7}{8}\selectfont\bfseries,anchor=north west] at (0.16,8.25){\textcolor{npblue}{\textbullet}\ Diffraction};
\draw[ax] (0.55,6.55)--(3.0,6.55);
\draw[ax] (0.55,6.55)--(0.55,7.800000000000001);
\draw[npblue,line width=1.4pt] (0.95,6.55)--(0.95,7.27);
\draw[npblue,line width=1.4pt] (1.4,6.55)--(1.4,7.57);
\draw[npblue,line width=1.4pt] (2.1,6.55)--(2.1,7.109999999999999);
\draw[npblue,line width=1.4pt] (2.6,6.55)--(2.6,6.95);
\draw[npred,line width=0.8pt,dash pattern=on 1.8pt off 1.4pt] (1.75,6.55)--(1.75,7.42);
\node[npred,font=\sffamily\fontsize{6}{7}\selectfont] at (1.75,7.550000000000001){$\times$};
\node[npred,font=\sffamily\fontsize{5.2}{6}\selectfont,anchor=south] at (2.2,7.6){claimed $33^\circ$};
\node[font=\sffamily\fontsize{6}{7}\selectfont,anchor=north] at (1.8,6.5){$2\theta$};
\node[font=\sffamily\fontsize{5.6}{7}\selectfont,rotate=90,anchor=south,black!65] at (0.45,7.15){intensity};
\node[font=\sffamily\fontsize{5.2}{6.4}\selectfont,anchor=north,text=black!62] at (1.575,6.08){admissible $2\theta$ from Bragg's law};
\node[font=\sffamily\fontsize{7}{8}\selectfont\bfseries,anchor=north west] at (3.46,8.25){\textcolor{npslate}{\textbullet}\ Micrograph};
\fill[npslate!48,rounded corners=1pt] (3.5,6.55) rectangle (6.279999999999999,7.82);
\fill[npslate!74] (3.9,7.5) circle (0.13);
\fill[npslate!74] (4.3,6.92) circle (0.1);
\fill[npslate!74] (4.85,7.5600000000000005) circle (0.16);
\fill[npslate!74] (5.35,7.08) circle (0.12);
\fill[npslate!74] (5.88,7.5600000000000005) circle (0.11);
\fill[npslate!74] (6.16,6.98) circle (0.08);
\fill[npslate!74] (4.55,7.24) circle (0.07);
\fill[npslate!74] (5.68,6.78) circle (0.09);
\fill[npslate!74] (4.08,6.75) circle (0.06);
\fill[npslate!74] (5.12,6.88) circle (0.07);
\draw[white!85,line width=0.5pt,fill=npslate!22] (4.88,7.220000000000001) ellipse (0.32 and 0.20);
\draw[{Latex[length=1.2mm]}-{Latex[length=1.2mm]},white,line width=0.5pt] (4.58,7.220000000000001)--(5.18,7.220000000000001);
\node[npred,font=\sffamily\fontsize{5.2}{6}\selectfont,fill=white,rounded corners=0.4pt,inner sep=0.8pt] at (5.02,7.68){claimed $d=10\,\mu$m};
\draw[white,line width=1.6pt] (3.88,6.720000000000001)--(4.43,6.720000000000001);
\node[white,font=\sffamily\fontsize{5.4}{6}\selectfont,anchor=south] at (4.16,6.75){100\,nm};
\node[font=\sffamily\fontsize{5.2}{6.4}\selectfont,anchor=north,text=black!62] at (4.875,6.08){field of view from the scale bar};
\node[font=\sffamily\fontsize{7}{8}\selectfont\bfseries,anchor=north west] at (6.76,8.25){\textcolor{nppurple}{\textbullet}\ Spectrum};
\draw[ax] (7.1499999999999995,6.55)--(9.6,6.55);
\draw[ax] (7.1499999999999995,6.55)--(7.1499999999999995,7.800000000000001);
\draw[nppurple,line width=1.1pt] (7.199999999999999,6.67) .. controls (7.6499999999999995,6.71) and (7.779999999999999,7.6) .. (7.9799999999999995,7.6) .. controls (8.18,7.6) and (8.33,6.7700000000000005) .. (8.68,6.75) .. controls (8.879999999999999,6.74) and (8.93,7.1) .. (9.129999999999999,7.07) .. controls (9.28,7.050000000000001) and (9.43,6.71) .. (9.58,6.69);
\fill[nppurple] (7.9799999999999995,7.6) circle (1.5pt);
\node[npgreen!55!black,font=\sffamily\fontsize{5.4}{6}\selectfont,anchor=south] at (7.9799999999999995,7.66){real maximum};
\draw[npred,line width=0.8pt,dash pattern=on 1.8pt off 1.4pt] (8.85,6.55)--(8.85,6.92);
\node[npred,font=\sffamily\fontsize{6}{7}\selectfont] at (8.85,7.03){$\times$};
\node[npred,font=\sffamily\fontsize{5.4}{6}\selectfont,anchor=south] at (8.899999999999999,7.12){claimed peak};
\node[font=\sffamily\fontsize{6}{7}\selectfont,anchor=north] at (8.4,6.5){wavenumber};
\node[font=\sffamily\fontsize{5.6}{7}\selectfont,rotate=90,anchor=south,black!65] at (7.05,7.15){intensity};
\node[font=\sffamily\fontsize{5.2}{6.4}\selectfont,anchor=north,text=black!62] at (8.174999999999999,6.08){a peak needs a local maximum};
\node[font=\sffamily\fontsize{7}{8}\selectfont\bfseries,anchor=north west] at (10.06,8.25){\textcolor{npgreen}{\textbullet}\ Convex hull};
\draw[ax] (10.48,6.55)--(12.9,6.55);
\draw[ax] (10.48,6.55)--(10.48,7.800000000000001);
\fill[npgreen!10] (10.65,7.42)--(11.25,6.75)--(12.05,6.9)--(12.780000000000001,7.36)--(12.780000000000001,6.55)--(10.65,6.55)--cycle;
\draw[npgreen,line width=1.0pt] (10.65,7.42)--(11.25,6.75)--(12.05,6.9)--(12.780000000000001,7.36);
\fill[npgreen] (10.65,7.42) circle (1.4pt);
\fill[npgreen] (11.25,6.75) circle (1.4pt);
\fill[npgreen] (12.05,6.9) circle (1.4pt);
\fill[npgreen] (12.780000000000001,7.36) circle (1.4pt);
\fill[npred] (11.5,7.5200000000000005) circle (1.7pt);
\draw[npred,line width=0.6pt,dash pattern=on 1.6pt off 1.2pt] (11.5,7.5200000000000005)--(11.5,6.83);
\node[npred,font=\sffamily\fontsize{5.4}{6}\selectfont,anchor=west] at (11.620000000000001,7.5600000000000005){claimed stable};
\node[font=\sffamily\fontsize{6}{7}\selectfont,anchor=north] at (11.700000000000001,6.5){composition};
\node[font=\sffamily\fontsize{5.6}{7}\selectfont,rotate=90,anchor=south,black!65] at (10.38,7.15){energy};
\node[font=\sffamily\fontsize{5.2}{6.4}\selectfont,anchor=north,text=black!62] at (11.475,6.08){stable only if $E_{\mathrm{hull}}\le\epsilon$};
\node[plab,anchor=west] at (-0.10,5.05){b)};
\node[font=\sffamily\fontsize{7.6}{9}\selectfont\bfseries,anchor=west] at (0.55,5.05){Each claim is tested against its own physics};
\node[data,text width=5.90cm,minimum height=0.62cm,align=justify,font=\sffamily\fontsize{5.7}{7}\selectfont\itshape,inner sep=3pt] (chain) at (3.42,4.28){\hyphenchar\font=-1\relax ``\dots the strongest reflection sits near $2\theta{=}33^\circ$, the particle is about $10\,\mu$m across, and the computed gap is small, so the phase is metallic.''};
\node[font=\sffamily\fontsize{5.6}{7}\selectfont,black!55,anchor=east] at (6.45,4.86){model reasoning chain};
\node[acc,text width=1.75cm,minimum height=0.52cm,align=center,font=\sffamily\fontsize{5.9}{7}\selectfont,inner sep=2pt] (cA) at (1.35,3.32){peak at $2\theta=33^\circ$};
\node[proc,fill=nppurple!10,draw=nppurple!55,text width=1.75cm,minimum height=0.62cm,align=center,font=\sffamily\fontsize{5.9}{7}\selectfont,inner sep=2pt] (oA) at (1.35,2.30){no reflection within\\$0.5^\circ$ of $33^\circ$};
\node[bad,text width=1.75cm,minimum height=0.42cm,align=center,font=\sffamily\fontsize{5.9}{7}\selectfont,inner sep=2pt] (vA) at (1.35,1.40){\textcolor{npred}{$\times$}\ violation};
\draw[flow](cA)--(oA);\draw[flow](oA)--(vA);
\draw[flow,black!45](chain.south)--(1.35,3.72)--(cA.north);
\node[acc,text width=1.75cm,minimum height=0.52cm,align=center,font=\sffamily\fontsize{5.9}{7}\selectfont,inner sep=2pt] (cB) at (3.42,3.32){feature $\approx10\,\mu$m};
\node[proc,fill=nppurple!10,draw=nppurple!55,text width=1.75cm,minimum height=0.62cm,align=center,font=\sffamily\fontsize{5.9}{7}\selectfont,inner sep=2pt] (oB) at (3.42,2.30){scale bar sets a\\$2.1\,\mu$m field of view};
\node[bad,text width=1.75cm,minimum height=0.42cm,align=center,font=\sffamily\fontsize{5.9}{7}\selectfont,inner sep=2pt] (vB) at (3.42,1.40){\textcolor{npred}{$\times$}\ violation};
\draw[flow](cB)--(oB);\draw[flow](oB)--(vB);
\draw[flow,black!45](chain.south)--(3.42,3.72)--(cB.north);
\node[acc,text width=1.75cm,minimum height=0.52cm,align=center,font=\sffamily\fontsize{5.9}{7}\selectfont,inner sep=2pt] (cC) at (5.49,3.32){phase is metallic};
\node[proc,fill=nppurple!10,draw=nppurple!55,text width=1.75cm,minimum height=0.62cm,align=center,font=\sffamily\fontsize{5.9}{7}\selectfont,inner sep=2pt] (oC) at (5.49,2.30){computed gap\\$0.02\,\mathrm{eV}\le0.05$};
\node[good,text width=1.75cm,minimum height=0.42cm,align=center,font=\sffamily\fontsize{5.9}{7}\selectfont,inner sep=2pt] (vC) at (5.49,1.40){\textcolor{npgreen!55!black}{\checkmark}\ consistent};
\draw[flow](cC)--(oC);\draw[flow](oC)--(vC);
\draw[flow,black!45](chain.south)--(5.49,3.72)--(cC.north);
\draw[nporange,line width=0.7pt](1.35,1.19)--(1.35,0.95);
\draw[flow,nporange](1.35,0.95)--(6.78,0.95)--(6.78,4.28)--(chain.east);
\node[font=\sffamily\fontsize{5.6}{7}\selectfont,black,anchor=north] at (5.10,0.86){violations returned (CGSV, diffraction and spectra only)};
\node[font=\sffamily\fontsize{5.5}{7}\selectfont,black!55,anchor=west,align=left] at (0.15,0.24){Schematic: the chain and the numbers illustrate the\\checks; they are not a model output.};
\begin{scope}[shift={(7.35,0)}]
\node[plab,anchor=west] at (-0.10,5.05){c)};
\node[font=\sffamily\fontsize{7.6}{9}\selectfont\bfseries,anchor=west] at (0.55,5.05){What the oracles actually catch};
\def\xb#1{2.62 + #1/80*2.55}
\node[font=\sffamily\fontsize{6}{8}\selectfont\itshape,black!65,anchor=west] at (0.02,4.52){Diffraction, 78 human-confirmed violations};
\node[font=\sffamily\fontsize{6}{8}\selectfont,anchor=east] at (2.52,4.12){gridline or scale misreading};
\fill[npblue]({\xb{0}},4.01) rectangle ({\xb{70.5}},4.23);
\node[font=\sffamily\fontsize{5.9}{7}\selectfont,anchor=west,text=black!65] at ({\xb{70.5}+0.08},4.12){70.5\%};
\node[font=\sffamily\fontsize{6}{8}\selectfont,anchor=east] at (2.52,3.76){peak misidentification};
\fill[npblue]({\xb{0}},3.65) rectangle ({\xb{14.1}},3.87);
\node[font=\sffamily\fontsize{5.9}{7}\selectfont,anchor=west,text=black!65] at ({\xb{14.1}+0.08},3.76){14.1\%};
\node[font=\sffamily\fontsize{6}{8}\selectfont,anchor=east] at (2.52,3.40){other};
\fill[black!35]({\xb{0}},3.29) rectangle ({\xb{15.4}},3.51);
\node[font=\sffamily\fontsize{5.9}{7}\selectfont,anchor=west,text=black!65] at ({\xb{15.4}+0.08},3.40){15.4\%};
\node[font=\sffamily\fontsize{6}{8}\selectfont\itshape,black!65,anchor=west] at (0.02,3.04){Microscopy, 93 human-confirmed violations};
\node[font=\sffamily\fontsize{6}{8}\selectfont,anchor=east] at (2.52,2.64){scale-bar unit error};
\fill[npslate]({\xb{0}},2.53) rectangle ({\xb{37.6}},2.75);
\node[font=\sffamily\fontsize{5.9}{7}\selectfont,anchor=west,text=black!65] at ({\xb{37.6}+0.08},2.64){37.6\%};
\node[font=\sffamily\fontsize{6}{8}\selectfont,anchor=east] at (2.52,2.28){unit conversion};
\fill[npslate]({\xb{0}},2.17) rectangle ({\xb{23.7}},2.39);
\node[font=\sffamily\fontsize{5.9}{7}\selectfont,anchor=west,text=black!65] at ({\xb{23.7}+0.08},2.28){23.7\%};
\node[font=\sffamily\fontsize{6}{8}\selectfont,anchor=east] at (2.52,1.92){scale-factor arithmetic};
\fill[npslate]({\xb{0}},1.81) rectangle ({\xb{14.0}},2.03);
\node[font=\sffamily\fontsize{5.9}{7}\selectfont,anchor=west,text=black!65] at ({\xb{14.0}+0.08},1.92){14.0\%};
\node[font=\sffamily\fontsize{6}{8}\selectfont,anchor=east] at (2.52,1.56){other};
\fill[black!35]({\xb{0}},1.45) rectangle ({\xb{24.7}},1.67);
\node[font=\sffamily\fontsize{5.9}{7}\selectfont,anchor=west,text=black!65] at ({\xb{24.7}+0.08},1.56){24.7\%};
\draw[black!45,line width=0.5pt]({\xb{0}},1.32)--({\xb{80}},1.32);
\draw[black!45,line width=0.4pt]({\xb{0}},1.32)--({\xb{0}},1.25);
\node[font=\sffamily\fontsize{6}{7}\selectfont,anchor=north] at ({\xb{0}},1.23){0\%};
\draw[black!45,line width=0.4pt]({\xb{20}},1.32)--({\xb{20}},1.25);
\node[font=\sffamily\fontsize{6}{7}\selectfont,anchor=north] at ({\xb{20}},1.23){20\%};
\draw[black!45,line width=0.4pt]({\xb{40}},1.32)--({\xb{40}},1.25);
\node[font=\sffamily\fontsize{6}{7}\selectfont,anchor=north] at ({\xb{40}},1.23){40\%};
\draw[black!45,line width=0.4pt]({\xb{60}},1.32)--({\xb{60}},1.25);
\node[font=\sffamily\fontsize{6}{7}\selectfont,anchor=north] at ({\xb{60}},1.23){60\%};
\draw[black!45,line width=0.4pt]({\xb{80}},1.32)--({\xb{80}},1.25);
\node[font=\sffamily\fontsize{6}{7}\selectfont,anchor=north] at ({\xb{80}},1.23){80\%};
\node[font=\sffamily\fontsize{6.2}{8}\selectfont,anchor=north] at ({\xb{40}},0.92){share of confirmed violations};
\node[font=\sffamily\fontsize{5.5}{7}\selectfont,black!55,anchor=west,align=left] at (0.02,0.24){Spectra and the three DFT families: Supplementary Table~S5.};
\end{scope}
\end{tikzpicture}}
\caption{\justifying\textbf{Physics is the label.} \textbf{a)} Materials data fixes quantities that a stated claim must respect: which $2\theta$ can carry a reflection, through Bragg's law and the reference pattern of the named phase; the field of view, through the calibration a scale bar sets; whether a claimed spectral peak coincides with a genuine local maximum of sufficient prominence; and whether a composition lies within tolerance of the computed convex hull. Each constraint is checkable with no human annotation. \textbf{b)} A chain is parsed into its checkable claims and each claim is tested against the oracle for its own constraint type, returning a per-claim verdict rather than a single score for the answer; detected violations are returned to the model by Constraint-Grounded Self-Verification, which is applied to the diffraction and spectra constraint types only (Methods). The chain and the values shown are schematic and illustrate the checks; they are not a model output. \textbf{c)} The violations these oracles flag are legible physical mistakes, not parsing artifacts. An open-ended categorization of every human-confirmed violation, shown here for the two image oracles that read a quantity off a plot or micrograph, gives a small taxonomy: diffraction failures are dominated by misreading the gridlines or the axis scale, and microscopy failures by scale-bar unit and unit-conversion errors, rather than by errors in the arithmetic that follows. Percentages are shares of the 78 and 93 confirmed violations respectively (Methods); the same categorization for spectra and the three DFT families, on their final-oracle audits, is in Supplementary Table~S5.}
\label{fig:teaser}
\end{figure}
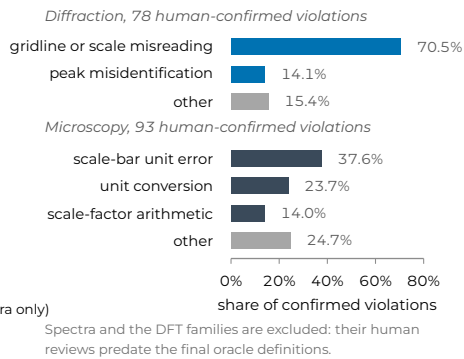

\section{Results}
\emph{All quantitative results below are measured, not projected. DFT-grounded
numbers (stability, metallicity, and magnetism-magnitude constraints, and every column or
cell marked with $^{*}$) rest on Materials Project reference values.}

\subsection{Oracle validity: positive control and precision audit}
Every downstream number rests on the oracles firing on genuine physical violations
rather than on parsing artifacts, so we validated the oracles before running any model
(Table~\ref{tab:exp1}). In a positive control we injected synthetic violations into real
reference data at planted rates of 10\%, 30\%, and 50\%; all six oracle families recovered
every planted violation at every rate (maximum false-negative rate 0.0000). This establishes
completeness on the \emph{injected} violation types only. It does not measure how many genuine
violations in real model responses the oracles miss: the recall $\hat d$ that enters
Corollary~\ref{cor:onesided} is not estimated by this experiment, so the one-sided relation
$\mathrm{PCR}\le\mathrm{PCR}^\star$ holds only to the extent that $\hat d=1$ also holds on real
responses (Methods). We therefore read the reported PCR as exact under that condition, not as a
guaranteed lower bound. To check this condition on real responses, an annotator audited 30 randomly
sampled oracle-silent responses per family (responses the oracle had judged consistent) for genuine
violations the oracle missed. The miss rate is 1/30 for diffraction, 2/30 for microscopy, 1/30 for spectra
and 0/30, 0/30 and 0/29 for stability, metallicity and magnetism (Wilson 95\% upper bounds 11--21\%);
scaled to the oracle-silent share of each audited population, the implied recall on real responses is
0.97 for diffraction, 0.81 for microscopy, 0.995 for spectra and 1.00 for the three DFT-grounded families
(Supplementary Table~S8). Microscopy, whose oracle checks only the field of view, is the least complete. We then audited oracle precision by human review of every oracle-fired
case in a collection of real Qwen2.5-VL-7B and Gemini 2.5 Flash Lite outputs, at least 100 per
family, all sampled and judged under the final oracle definitions (Table~\ref{tab:exp1};
Supplementary Fig.~S1 gives the full breakdown). Audited precision with Wilson 95\% intervals is 77.2\% [68.1, 84.3] for diffraction
($n=101$), 88.6\% [81.1, 93.3] for microscopy ($n=105$), 92.7\% [90.1, 94.7] for spectra
($n=493$ decided of 495), and 100\% [93.7, 100], 94.7\% [89.0, 97.6] and 92.2\% [85.8, 95.8]
for stability, metallicity and magnetism ($n=57$, 114 and 115; of the 107 audited stability cases,
the 50 near-hull ones are not scored under the final oracle, Methods). Only the spectrum audit
contained cases the annotator marked unsure (2 of 495); counting them as false positives or as
genuine violations gives 92.3\% or 92.7\%, so their exclusion does not move the estimate. Each
audit was performed by one annotator, and a second annotator independently re-labelled a stratified
random subsample of 30 audited cases per family, blind to the first labels. Agreement is high where the
criterion is sharp: Cohen's $\kappa$ is 1.00 for metallicity, 0.71 for microscopy and 0.63 for
magnetism, and the two annotators agree on all 14 subsampled stability cases that the final oracle
scores. It is low at the tolerance boundaries of diffraction ($\kappa=0.33$) and spectra ($\kappa=0.13$):
every disagreement there is a claim 0.5--0.7$^\circ$ from a reference reflection, or 7--17
$\mathrm{cm}^{-1}$ from a genuine Raman peak, which one annotator judged a violation and the other a
reading on the flank of a real peak. Weighting the second annotator's labels by the sampling strata
gives precisions of 0.60 [0.44, 0.75] for diffraction and 0.61 [0.43, 0.76] for spectra, against 0.77 and
0.93 from the first annotator (Supplementary Table~S8). Both annotators then adjudicated together the
24 subsampled cases on which they disagreed (9 diffraction, 2 microscopy, 11 spectra and 2 magnetism).
The consensus sided with the second annotator in all 22 diffraction, spectrum and magnetism disputes and
with the first in both microscopy disputes. Weighted back to the audited population by stratum, consensus
precision is 0.60 [0.44, 0.75] for diffraction, 0.89 [0.82, 0.94] for microscopy, 0.61 [0.44, 0.76] for
spectra, 1.00 [0.84, 1.00] for stability, 0.95 [0.87, 0.97] for metallicity and 0.91 [0.80, 0.97] for
magnetism (Table~\ref{tab:exp1}). For diffraction and spectra the consensus is therefore well below the
single-annotator estimate: roughly four in ten fired cases in these two families are boundary readings that
both annotators, after discussion, do not count as violations, so the PCR of these families understates the
true rate by the larger false-positive fraction of Methods (Corollary~\ref{cor:onesided}). The second annotation also exposed a conflict in
the original stability oracle, corrected in the final protocol (Methods). The false positives of the diffraction and spectrum oracles
are predominantly a tolerance effect: 16 of 23 diffraction false positives lie within
$1.0^\circ$ of a real reference peak, just outside the $0.5^\circ$ tolerance
(Fig.~\ref{fig:oracle-examples}a), and 34 of 36 spectrum false positives lie within
$10\,\mathrm{cm}^{-1}$ of a real, independently detected local maximum, just outside the
$5\,\mathrm{cm}^{-1}$ tolerance (Fig.~\ref{fig:oracle-examples}b); only 2 of 36 are farther than
$30\,\mathrm{cm}^{-1}$ from any real peak. How these tolerances compare with the sampling and
peak widths of the underlying measurements is quantified in Methods.

\begin{figure}[!htbp]
    \centering

    \includegraphics[width=\linewidth,height=0.235\textheight,keepaspectratio]
    {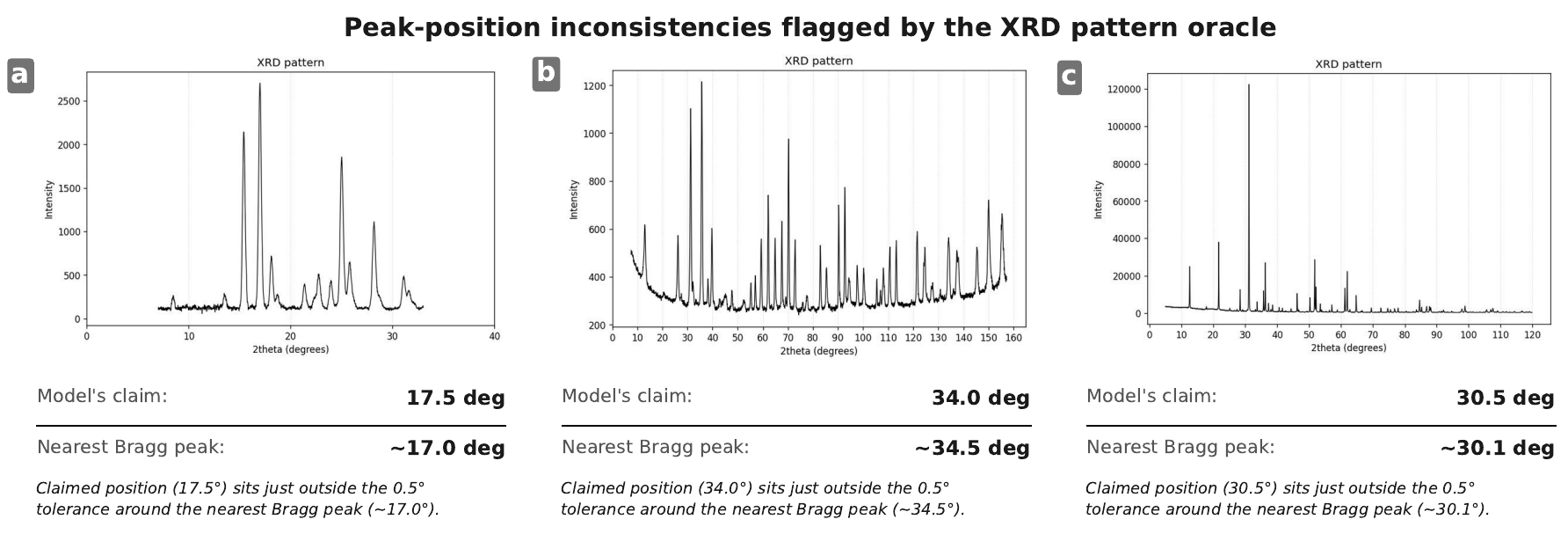}

    \vspace{1mm}

    \includegraphics[width=\linewidth,height=0.235\textheight,keepaspectratio]
    {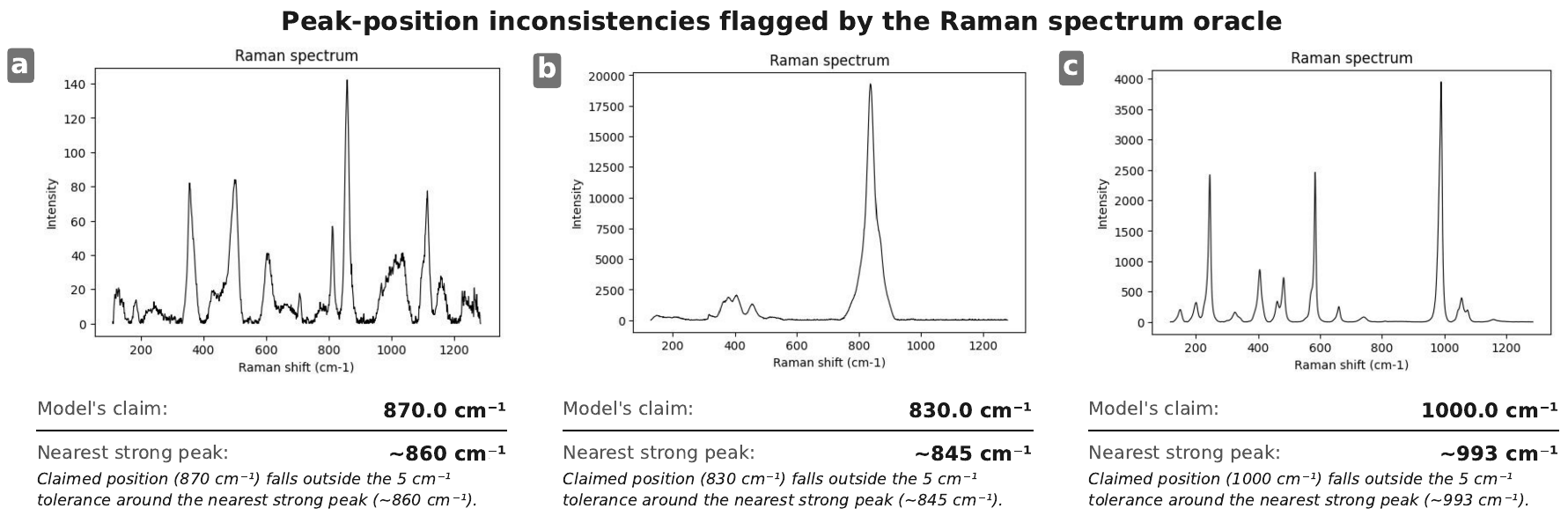}

    \vspace{1mm}

    \includegraphics[width=\linewidth,height=0.235\textheight,keepaspectratio]
    {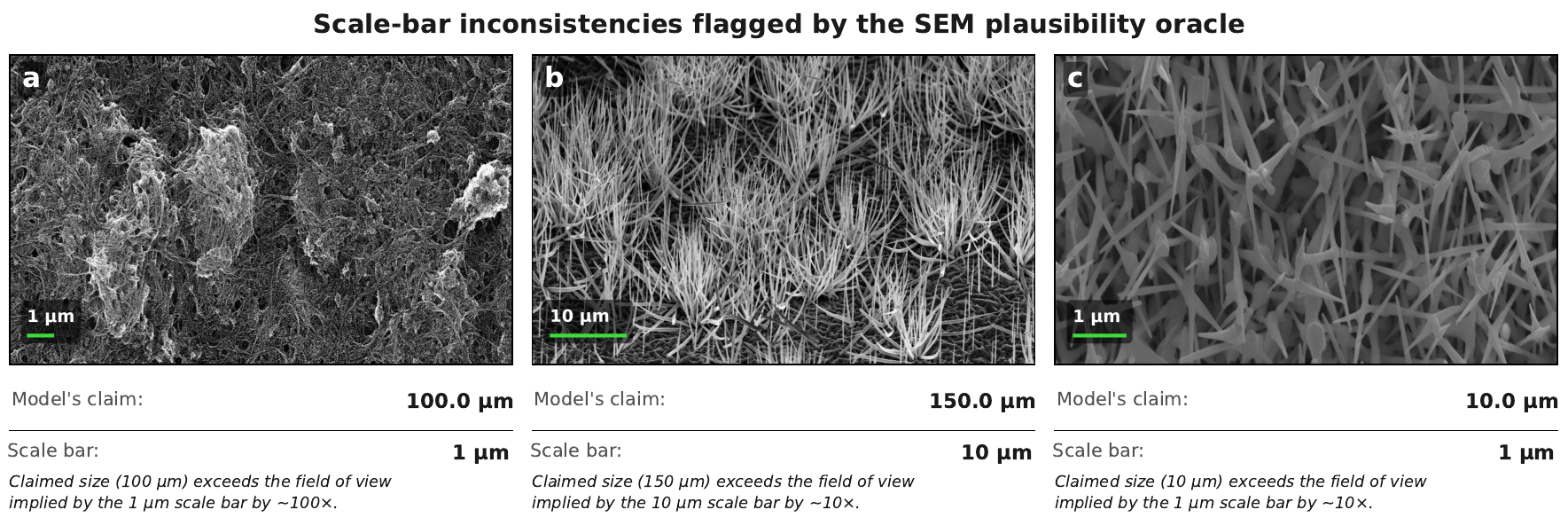}

    \caption{\justifying\textbf{Representative inconsistencies identified by the multimodal
    oracles.}
    \textbf{a}, X-ray diffraction cases in which the model-reported
    $2\theta$ position lies outside the $0.5^\circ$ tolerance around the nearest
    genuine Bragg reflection.
    \textbf{b}, Raman spectroscopy cases in which the model-reported Raman shift
    exceeds the $5\,\mathrm{cm}^{-1}$ tolerance relative to the nearest genuine
    spectral peak.
    \textbf{c}, Microscopy cases in which the model-reported feature size is
    inconsistent with the physical scale indicated by the embedded scale bar,
    producing discrepancies of approximately $10$--$100\times$.
    All examples are human-confirmed oracle activations and illustrate
    representative tolerance, unit-conversion, and scale-factor failure modes.}

    \label{fig:oracle-examples}
\end{figure}

\begin{table}[t]
\centering
\caption{\justifying\textbf{Oracle validity and audited precision.} The positive-control
false-negative rate is the maximum across three planted violation rates (10\%, 30\%, and 50\%);
all three rates yielded 0.0000 for every oracle family, which measures completeness on the
injected violation types only. Audited precision is the fraction of oracle-fired cases from
Qwen2.5-VL-7B and Gemini 2.5 Flash Lite outputs that human review confirmed as genuine
violations of the stated constraint rather than parsing or oracle artifacts, with Wilson
95\% intervals; $n$ counts decided cases (2 spectrum cases marked unsure are excluded; counting
them either way gives 92.3--92.7\%). Every audit was sampled and judged under the final oracle
definitions used for all reported results (Methods): spectrum responses were collected once and the fired
set re-scored with the final oracle before the audit, and DFT responses were re-collected under the final
definitions before the audit; the diffraction and microscopy audits were carried out in July 2026 and the others in August 2026. Sources: NFFA, NFFA-EUROPE; MP, Materials Project.
$^{a}$Final stability oracle with the near-hull abstention band: the 50 audited cases with
$0<E_{\mathrm{hull}}\le0.05$ eV\,atom$^{-1}$ are not scored; under the original rule the audit gave 95.3\%
[89.5, 98.0] ($n=107$). Consensus: a second annotator, blind to the first labels, re-labelled a random
subsample of 30 audited cases per family stratified by the first annotator's verdict, and the two annotators
adjudicated together every case on which they disagreed (24 in total); the consensus precision is weighted back
to the audited population by stratum, with 95\% intervals from per-stratum Jeffreys posteriors (30 subsampled
cases per family; $^{b}$14, the subsampled stability cases scored under the final oracle). Agreement statistics and the oracle-silent audit are in Supplementary Table~S8.
The full audits were performed by one annotator (R.K.).}
\label{tab:exp1}
{\scriptsize\setlength{\tabcolsep}{2.5pt}\renewcommand{\arraystretch}{1.12}
\begin{tabular}{l l S[table-format=1.4] c c}
\toprule
Oracle family & Source & {FN rate} & \multicolumn{2}{c}{Audited precision, \% [95\% CI]} \\
\cmidrule(lr){3-3}\cmidrule(lr){4-5}
& & {\itshape positive control} & {\itshape one annotator ($n$)} & \itshape two-annotator consensus \\
\midrule
\multicolumn{5}{l}{\itshape Geometric oracles}\\
Diffraction & opXRD & 0.0000 & 77.2 [68.1, 84.3] (101) & 60.2 [44.1, 75.4] \\
Microscopy & NFFA       & 0.0000 & 88.6 [81.1, 93.3] (105) & 88.6 [81.8, 93.9] \\
Raman spectra & RRUFF & 0.0000 & 92.7 [90.1, 94.7] (493) & 60.8 [43.6, 76.3] \\
\addlinespace
\multicolumn{5}{l}{\itshape DFT-grounded oracles}\\
Stability & MP & 0.0000 & 100 [93.7, 100] (57)$^{a}$ & 100 [84.3, 100]$^{b}$ \\
Metallicity & MP & 0.0000 & 94.7 [89.0, 97.6] (114) & 94.7 [86.8, 97.0] \\
Magnetism & MP & 0.0000 & 92.2 [85.8, 95.8] (115) & 91.4 [79.9, 96.5] \\
\bottomrule
\end{tabular}}
\end{table}

\subsection{Physical consistency of current models}
We measure the Physical-Consistency Rate and per-constraint-type violation rates for
the four originally evaluated models, four later-generation frontier or open models, and
one flagship-tier addition
across diffraction, microscopy, spectra, and
DFT-grounded constraints, separately for image (VLM) and structure or text (LLM)
inputs (Table~\ref{tab:main}, Fig.~\ref{fig:results}a). Each (model, constraint-type) cell
sends 30 items per seed to the model, with $S=3$ seeds (90 items per cell); 4{,}334 of these
responses are scored across the nine models. For the image constraint types all nine models
receive the same items in each seed, so their comparisons are paired; for the DFT constraint
types each model and seed draws its own block of materials, so those comparisons are unpaired.
Not every attempted item can be scored. On diffraction, 64 of the 90 items per cell have a
reference pattern the oracle can use (the rest lack a usable structure); on metallicity, 1--8
items per cell fall in the oracle's abstention band, and on stability 11--19 items per cell are
near-hull materials, which the final stability oracle does not score (Methods). Within the items the oracle can score, the
share of responses that state a checkable quantity (coverage) is at least 0.93 in every cell
except GPT-4o on spectra (0.49; Supplementary Table~S2). Overall PCR ranges from 0.556
(Qwen3-VL-8B-Instruct) to 0.778 (GPT-5.6-Sol, with GPT-5.6-Terra at 0.777); models are
markedly more consistent on microscopy (0.67--1.00) than on spectra (0.07--0.83) or
diffraction (0.47--0.82), and the DFT-grounded column is more tightly clustered
(0.65--0.76) than any single geometric constraint.

Because PCR excludes responses that state no checkable quantity, a model could in
principle raise its score by stating less. We therefore also scored every cell under two
alternative conventions: counting a non-checkable response as a violation (penalised) and as
consistent (lenient). The overall ordering of the nine models is nearly unchanged (Kendall
$\tau=0.83$ between the reported and penalised conventions and $1.00$ between the reported and
lenient ones); the main change is GPT-4o, from third to fifth under the penalised
convention, because of its low coverage on spectra (Supplementary Table~S2).

We test model differences at the item level rather than by comparing interval overlap:
exact McNemar tests on the paired image-arm items, Fisher exact tests on the unpaired DFT items,
and, for overall PCR, a bootstrap that resamples paired items jointly and DFT items within
model (Methods). Holm correction is applied within each family of comparisons (Supplementary
Table~S3). GPT-5.6-Terra exceeds GPT-4o overall by 0.072 (95\% CI $[0.018,0.125]$, Holm
$p=0.045$). Claude Opus 4.6 versus Claude Sonnet 4.5 ($-0.003$, $[-0.050,0.044]$),
Qwen3-VL-8B-Thinking versus Qwen3-VL-8B-Instruct ($+0.034$, $[-0.016,0.083]$) and
GPT-5.6-Sol versus GPT-5.6-Terra ($+0.002$, $[-0.042,0.046]$) are not significant. With
roughly 500 checkable items per model, the smallest overall difference the design detects with
80\% power is 0.064--0.079 depending on the pair, so these three comparisons exclude only
differences larger than about 0.07; they do not show that the models are equivalent. Across all
36 model pairs, 18 differences in overall PCR survive Holm correction; within single
constraint types the detectable difference is 0.11--0.20, so most per-type rankings are not
resolved. On microscopy, the locally run Qwen2.5-VL-7B is the weakest model (0.667, 95\% CI
$[0.563,0.759]$), and seven of the other eight models exceed it after Holm correction (paired
differences $+0.22$ to $+0.33$), while Qwen3-VL-8B-Instruct ($+0.14$, Holm $p=0.12$) does not. We do not
know whether the gap reflects genuine scale-bar-reading skill or a difference in how each model handles the
info-bar layout variety in the source micrographs, and do not claim more than the
measured difference. The high absolute microscopy rates should also be read against the
looseness of this oracle, which flags only feature sizes exceeding the field of view (a
necessary, not sufficient, condition), so a near-ceiling rate partly reflects an easily
satisfied check. Diffraction comparisons in particular should be treated as provisional:
the diffraction ranking changes both when the oracle tolerance is varied and when the prompt is
paraphrased (Methods).

The DFT column is also sensitive to class composition. The sampled materials are
imbalanced (for stability, only 3--9 of the 70--79 scored materials per model lie on the hull), so a
model that defaults to the majority answer scores well on pooled PCR. Qwen2.5-VL-7B reaches a pooled
stability PCR of 0.917 but answers none of its 4 on-hull materials consistently, a class-balanced rate
of 0.485, whereas GPT-5.6-Terra has the lowest pooled stability PCR (0.653) but a class-balanced rate of
0.814 (Supplementary Table~S4); with so few on-hull materials per model, these class-balanced rates are
imprecise. The ordering of models on the DFT constraint types therefore depends on whether PCR is pooled
or class-balanced, and we report both. To make the DFT comparisons paired as well, we evaluated the
models on one common item set of 90 materials per DFT constraint type, drawn 45 from each reference
class (Supplementary Table~S6); for stability, 34 of the 45 near-hull draws fall in the band the final
oracle does not score, leaving 56 scored materials (11 on the hull, 45 above it). Eight models are
included; Qwen3-VL-8B-Thinking could not be run because its provider endpoint had been withdrawn by the
time of this analysis, and the locally run Qwen2.5-VL-7B was run on CPU in single precision, because no
GPU with bfloat16 support was available then. On this paired set, restricted to the DFT constraint
types, none of the comparisons stated above is significant: GPT-5.6-Terra versus GPT-4o $+0.021$ (95\% CI
$[-0.042,0.085]$), Claude Opus 4.6 versus Claude Sonnet 4.5 $+0.021$ ($[-0.030,0.072]$) and GPT-5.6-Sol
versus GPT-5.6-Terra $+0.030$ ($[-0.013,0.072]$); 3 of the 28 model pairs differ after Holm correction,
all of them Gemini 2.5 Flash Lite scoring lower. The difference between GPT-5.6-Terra and GPT-4o in the
main evaluation therefore comes from the image constraint types, chiefly spectra. The balanced set also
shows a shared failure that the imbalanced sample hid: every model is far less consistent on
computed-metallic compositions (PCR 0.21--0.40) than on gapped ones (0.84--0.96), that is, models
default to calling a composition non-metallic.

\begin{table}[t]
\centering
\caption{\justifying Main evaluation: Physical-Consistency Rate by constraint type and overall,
over the checkable responses among 30 attempted items per (model, constraint-type) cell per seed, $S=3$ seeds pooled per cell
(bootstrap 95\% CIs omitted from this compact view; see Supplementary Table~S1 for the
full per-cell confidence intervals and checkable counts, Supplementary Table~S2 for coverage
and alternative conventions, Supplementary Table~S3 for paired, Holm-corrected model comparisons,
and Supplementary Table~S4 for class-balanced DFT rates, and the Methods for why $\rho_t=1-\mathrm{PCR}$ here). DFT column combines the stability, metallicity,
and magnitude constraint types; near-hull stability materials are not scored (Methods; the original
rule's values are in Supplementary Table~S7).  The first four rows are the
originally evaluated models; the next four are the later-generation refresh
(Methods, Baselines subsection). $^{\ddagger}$GPT-5.6-Sol is OpenAI's flagship tier, added
to test whether it exceeds the balanced-tier GPT-5.6-Terra within the same
generation; the overall difference is $+0.002$ (95\% CI $[-0.042,0.046]$), not significant, so
the design detects no difference between flagship and balanced tier larger than about 0.07
(Discussion). Bold marks the column maximum; it does not imply a significant difference from
the next model.}
\label{tab:main}
{\scriptsize\setlength{\tabcolsep}{5pt}\renewcommand{\arraystretch}{1.12}
\begin{tabular}{l *{5}{S[table-format=1.3]}}
\toprule
& \multicolumn{5}{c}{Physical-Consistency Rate (higher is more consistent)} \\
\cmidrule(lr){2-6}
Model & {Diffraction} & {Microscopy} & {Spectra} & {DFT$^{*}$} & {All$^{*}$} \\
\midrule
Qwen2.5-VL-7B \textit{(open)} & 0.500 & 0.667 & 0.148 & 0.731 & 0.584 \\
Gemini 2.5 Flash Lite & 0.517 & \bfseries 1.000 & 0.278 & 0.658 & 0.630 \\
GPT-4o                & 0.625 & 0.954 & 0.068 & 0.740 & 0.699 \\
Claude Sonnet 4.5     & 0.547 & 0.977 & 0.189 & \bfseries 0.762 & 0.667 \\
\addlinespace
Claude Opus 4.6            & 0.578 & 0.954 & 0.244 & 0.734 & 0.663 \\
GPT-5.6-Terra              & 0.797 & 0.951 & \bfseries 0.833 & 0.696 & 0.777 \\
Qwen3-VL-8B-Instruct       & 0.468 & 0.802 & 0.133 & 0.649 & 0.556 \\
Qwen3-VL-8B-Thinking       & 0.578 & 0.885 & 0.092 & 0.671 & 0.594 \\
\addlinespace
GPT-5.6-Sol$^{\ddagger}$   & \bfseries 0.820 & 0.915 & 0.764 & 0.727 & \bfseries 0.778 \\
\bottomrule
\end{tabular}}
\end{table}

\subsection{Accuracy versus consistency}
On three constraint types with a genuine, non-circular accuracy criterion distinct
from the oracle's own tolerance (diffraction: agreement with the single nearest true
peak to $0.1^{\circ}$, versus the oracle's looser $0.5^{\circ}$ tolerance; spectra:
agreement with the nearest real measured peak to $2\,\mathrm{cm}^{-1}$, versus the
oracle's prominence-based check; stability: agreement with the Materials
Project \texttt{is\_stable} label, evaluated separately from the oracle's
operational $0.05\,\mathrm{eV\,atom^{-1}}$ near-hull threshold), we report
accuracy alongside PCR and the four-cell breakdown on the same
838 scored items ($n=30$ per (model, constraint-type) cell, $S=3$ seeds, the four originally evaluated
models; this sample is drawn separately from the main evaluation). Microscopy is excluded: its
prompt elicits only a feature-size claim, so no non-circular accuracy ground truth is
available without a different prompt design. On the eight diffraction and spectra cells,
accuracy and PCR disagree on a mean of 23.6\% of items, and on stability they coincide
(Fig.~\ref{fig:pipeline}b,c). Part of the diffraction and spectra disagreement is definitional rather
than evidence of reasoning errors that accuracy misses: the accuracy criteria are deliberately tighter
than the oracle tolerances, so the two scores are expected to differ on items that fall between the two
thresholds. What the comparison does establish is that
the two scores measure different properties of the same response, not that PCR detects errors that
accuracy cannot. Diffraction is dominated by items
that are wrong by the tight accuracy criterion yet still oracle-consistent (up to 36
of 62 items in one diffraction cell, gap 58.1\%): models are frequently close enough to satisfy the
oracle's looser tolerance while measurably off the exact peak. On spectra, most responses
are wrong by both criteria, so the two scores disagree on few items (4.2--11.2\% across models). On stability the two scores coincide under the final protocol,
and this is informative in itself. The oracle's $0.05\,\mathrm{eV\,atom^{-1}}$ near-hull threshold and
Materials Project's own \texttt{is\_stable} label (true only on the hull) disagree on 571 of 3,000 real
entries (19.0\%), and every one of those is a near-hull material, $0<E_{\mathrm{hull}}\le0.05$. The final
oracle does not score that band (Methods), so outside it the two definitions agree, and the stability gap
of 18.0--27.8\% reported under the original rule came entirely from the near-hull materials. The stability
arm still shows the base-rate confound: Qwen2.5-VL-7B is accurate on 61 of 62 unstable materials but on
none of the 8 stable ones, so its near-ceiling raw accuracy reflects the base rate rather than
discriminative skill. Balanced accuracy (the mean of per-class accuracy) is 0.492 (Qwen2.5-VL-7B), 0.690
(Gemini 2.5 Flash Lite), 0.492 (GPT-4o) and 0.992 (Claude Sonnet 4.5); two of the four models are at
chance once the imbalance is corrected for ($n=65$--73 scored materials per model, 5--8 of them stable).

What PCR adds to answer accuracy is therefore not a finer measure of correctness but a
different, label-free property. First, it needs no answer key: it is defined on microscopy and
on net magnetization, where this benchmark has no non-circular accuracy criterion, and on any new
item for which a reference structure, scale bar or computed value exists. Second, it is a floor
rather than a target: a violation is a statement the measurement or the reference calculation
rules out, so a violation is informative about the reasoning that produced it (Fig.~\ref{fig:teaser}c
shows what these violations consist of), whereas an accuracy miss may only reflect reading
precision. Third, because it is a necessary condition, PCR cannot certify that a response is
right; a high PCR with low accuracy, as on diffraction, identifies responses that are physically
admissible but imprecise. We therefore present PCR as a complement to answer accuracy where labels
exist and as a substitute screening signal where they do not, not as a replacement for either.

\subsection{Effect of self-verification}
We measure PCR before and after CGSV on identical items and test the change with the
exact paired McNemar test, reporting fixes, regressions, and residual violations
(Table~\ref{tab:cgsv}, Fig.~\ref{fig:results}b), on diffraction and spectra, the two constraint types where
the oracle's feedback is an actionable physical constraint the model can re-examine
without the feedback message itself revealing the answer (a scale-bar violation
message would state the conversion factor outright, and the three DFT claim types are
single-shot binary judgments with no intermediate quantity to reconsider). $n=30$
items per (model, constraint-type) cell, $S=3$ seeds, $K=2$ maximum revision rounds, 545
checkable records in total. Because Algorithm~\ref{alg:cgsv} stops at the first round in which no
violation remains, an item that is consistent before revision is never revised. The absence of
regressions in consistency ($b=0$ in every cell) therefore follows from the stopping rule; it
confirms the implementation but is not evidence about the model. CGSV raises PCR for all four models pooled across both
constraint types (Qwen2.5-VL-7B $0.338\!\to\!0.439$, $p=3.1\times10^{-5}$; Gemini 2.5
Flash Lite $0.521\!\to\!0.632$, $p=2.4\times10^{-4}$; GPT-4o
$0.434\!\to\!0.593$, $p=7.6\times10^{-6}$; Claude Sonnet 4.5 $0.373\!\to\!0.639$,
$p=4.6\times10^{-13}$; all survive Holm correction).
Residual violations remain in every cell after $K=2$ rounds (roughly 40--90\% of
items on spectra), so CGSV reduces but does not eliminate violations.

Because CGSV is scored by the same oracles that supply its feedback, a rise in PCR alone cannot
show that revised answers are more correct. Every item exposes exactly one checkable claim in
every round, so a gain cannot come from the model dropping the claim. To see what the revisions
achieve, we scored each item's first-round and final-round claim against the same non-circular
accuracy criteria as the accuracy-versus-consistency analysis ($0.1^\circ$ for diffraction,
$2\,\mathrm{cm}^{-1}$ for spectra) and compared them on the same items (Table~\ref{tab:cgsvacc}).
Accuracy is a property the loop never observes, so unlike consistency it can regress. On spectra
($n=293$) accuracy rises from 5.8\% to 9.9\%: 14 items become accurate and 2 cease to be (exact
McNemar $p=0.004$). Of the 244 spectrum items the oracle flagged and the model revised, 48 cleared
the oracle; 11 of these moved to the accurate value and 37 to a value the oracle admits but that
is not within the accuracy tolerance, and the remaining 196 stayed inconsistent.
On diffraction ($n=245$; 7 items whose reference pattern could not be computed are omitted) accuracy
rises from 22.9\% to 28.6\%: 14 items become accurate and none ceases to be (exact McNemar $p=0.0001$).
Of the 79 diffraction items the oracle flagged and the model revised, 41 cleared the oracle; 14 of these
moved to the accurate value and 27 to an admissible but inexact one, and 38 stayed inconsistent. Revisions that clear
the oracle therefore mostly move a claim into the physically admissible region rather than onto the
exact value, while also producing a small but significant net gain in correctness. The evidence
supports this narrower statement, not a claim that CGSV corrects the underlying physics in general.

We separately ask whether the gain is
attributable to the oracle signal specifically, or to extra inference-time compute alone,
with three controls matched to CGSV's items, models, and decoding
(Baselines subsection): unaided self-refinement (the model is told to double-check its
answer with no oracle feedback) shows no significant change (before $0.349$, after $0.339$, $n=542$, $b=34$, $c=29$, $p=0.615$);
generic equal-compute re-prompting (the model is simply asked to reconsider) likewise shows
no significant change (before $0.358$, after $0.342$, $n=562$, $b=24$, $c=15$, $p=0.20$). Both
controls revise every item, including those already consistent, so their nonzero regression counts
are not comparable with CGSV's structural $b=0$; the comparison that matters is the net change,
which is significant for CGSV and absent for both controls. Best-of-3 independent resampling gives a
random-pick PCR of $0.363$ and an oracle-oracle upper bound (best of 3) of $0.419$
($n=614$). Both fall below CGSV's oracle-guided after-PCR for every model, though the
margin over the weakest, Qwen2.5-VL-7B at $0.439$, is only $0.02$, and the resampling run
draws its own item sample, so we do not test the difference. These controls show that the PCR gain
is not reproduced by generic self-reflection or by extra sampling at equal compute; they do not
by themselves show that the gain reflects better physics, which is what the paired correctness
analysis above addresses.

\begin{table}[t]
\centering
\caption{\justifying Effect of Constraint-Grounded Self-Verification, pooled across diffraction
and spectra, $n=30$ items per (model, constraint-type) cell, $S=3$ seeds, $K=2$ rounds.
$\Delta$: PCR gain (after minus before); $b$: regressions (before-consistent,
after-inconsistent); $c$: fixes (before-inconsistent, after-consistent); $p$: exact
two-sided McNemar test on $(b,c)$; $p_{\mathrm{Holm}}$: Holm-Bonferroni-adjusted $p$ over
this table's 6-test family (Methods, Statistics paragraph). All four CGSV effects survive
this correction; both controls remain non-significant. For CGSV, $b=0$ follows from the stopping
rule (an item consistent before revision is never revised); the controls revise every item, so
their $b$ is not comparable. Paired changes in correctness are in Table~\ref{tab:cgsvacc}.}
\label{tab:cgsv}
{\scriptsize\setlength{\tabcolsep}{3.4pt}\renewcommand{\arraystretch}{1.12}
\begin{tabular}{l S[table-format=1.3] S[table-format=1.3] S[table-format=+1.3] S[table-format=2.0] S[table-format=2.0] c c}
\toprule
& \multicolumn{3}{c}{PCR} & \multicolumn{2}{c}{Discordant} & & \\
\cmidrule(lr){2-4}\cmidrule(lr){5-6}
Model & {before} & {after} & {$\Delta$} & {$b$} & {$c$} & {$p$} & {$p_{\mathrm{Holm}}$} \\
\midrule
Qwen2.5-VL-7B         & 0.338 & 0.439 & 0.101 & 0 & 16 & $3.1\times10^{-5}$ & $1.2\times10^{-4}$ \\
Gemini 2.5 Flash Lite & 0.521 & 0.632 & 0.111 & 0 & 13 & $2.4\times10^{-4}$ & $7.3\times10^{-4}$ \\
GPT-4o                & 0.434 & 0.593 & \bfseries 0.159 & 0 & 18 & $7.6\times10^{-6}$ & $3.8\times10^{-5}$ \\
Claude Sonnet 4.5     & 0.373 & 0.639 & \bfseries 0.266 & 0 & 42 & $4.6\times10^{-13}$ & $2.7\times10^{-12}$ \\
\midrule
\multicolumn{8}{l}{\itshape Mechanism-isolation controls (pooled across models, both constraint types)}\\
Self-refinement (no oracle) & 0.349 & 0.339 & -0.010 & 34 & 29 & 0.615 & 0.615 \\
Generic re-prompting        & 0.358 & 0.342 & -0.016 & 24 & 15 & 0.20 & 0.399 \\
\bottomrule
\end{tabular}}
\end{table}

\begin{table}[t]
\centering
\caption{\justifying Paired change in answer correctness before and after CGSV, on the same items as
Table~\ref{tab:cgsv}. Accuracy uses the non-circular criteria of the accuracy-versus-consistency
analysis (diffraction: within $0.1^\circ$ of the nearest reference reflection; spectra: within
$2\,\mathrm{cm}^{-1}$ of the nearest detected local maximum), applied to each item's first-round
and final-round claim (accuracy before$\to$after). L/G: items that stop (lost) or start (gained) being
accurate; $p$: exact two-sided McNemar test on (L, G). Rev.: items the oracle flagged at round 1 and the
model revised; of these, A cleared the oracle with an accurate value, Ad cleared it with an admissible
value outside the accuracy tolerance, and U were still flagged after $K=2$ rounds.}
\label{tab:cgsvacc}
{\scriptsize\setlength{\tabcolsep}{2.5pt}\renewcommand{\arraystretch}{1.12}
\begin{tabular}{l l r c c r r c}
\toprule
Constraint & Model & $n$ & Accuracy & L/G & $p$ & Rev. & A/Ad/U \\
\midrule
Spectra & Qwen2.5-VL-7B & 89 & 0.067$\to$0.101 & 0/3 & 0.25 & 80 & 3/9/68 \\
 & Gemini 2.5 F.\ Lite & 69 & 0.043$\to$0.072 & 0/2 & 0.50 & 49 & 2/6/41 \\
 & GPT-4o & 45 & 0.044$\to$0.044 & 1/1 & 1.00 & 40 & 0/4/36 \\
 & Claude Sonnet 4.5 & 90 & 0.067$\to$0.144 & 1/8 & 0.039 & 75 & 6/18/51 \\
 & \textit{pooled} & 293 & 0.058$\to$0.099 & 2/14 & 0.004 & 244 & 11/37/196 \\
\addlinespace
Diffraction & Qwen2.5-VL-7B & 66 & 0.303$\to$0.318 & 0/1 & 1.00 & 24 & 1/3/20 \\
 & Gemini 2.5 F.\ Lite & 47 & 0.213$\to$0.255 & 0/2 & 0.50 & 7 & 2/3/2 \\
 & GPT-4o & 66 & 0.152$\to$0.212 & 0/4 & 0.12 & 24 & 4/10/10 \\
 & Claude Sonnet 4.5 & 66 & 0.242$\to$0.348 & 0/7 & 0.016 & 24 & 7/11/6 \\
 & \textit{pooled} & 245 & 0.229$\to$0.286 & 0/14 & 0.0001 & 79 & 14/27/38 \\
\bottomrule
\end{tabular}}
\end{table}

\subsection{Verifier for violation prediction}
We fine-tune Qwen2.5-0.5B-Instruct with LoRA on oracle-derived labels alone (no human
annotation) to predict, from the input description and chain text alone, whether a
physical violation is present, using 2{,}152 real oracle-checked records spanning all
six constraint types collected across the preceding evaluations. We report precision,
recall, $F_1$, ROC AUC, and expected calibration error (ECE) on an 80/20
in-distribution split and separately on six held-out-constraint-type splits (train on the
other five constraint types, test on one entirely unseen one). The in-distribution split
is disjoint at the parent source (item or structure) level, so no source item contributes
records to both sides. In distribution the verifier is a useful
detector overall (precision 0.77, recall 0.71, $F_1=0.74$, AUC $=0.850$, ECE $=0.096$).
This pooled figure overstates its usefulness for a given scientific task, for two reasons.
First, the six families have very different violation base rates (19\% for microscopy to 88\%
for spectra in the test split), so a score that only identified the family would already rank
violations above chance: a baseline that assigns each item its family's training-set violation
rate reaches AUC $=0.759$. Second, within single families the verifier's discrimination varies
widely (Table~\ref{tab:verifierfam}): it is strong on microscopy (AUC $0.922$, 95\% CI
$[0.852,0.977]$) and stability ($0.804$, $[0.683,0.903]$), moderate on magnetism ($0.728$,
$[0.562,0.877]$) and diffraction ($0.696$, $[0.493,0.895]$), and not distinguishable from chance on
metallicity ($0.582$, $[0.390,0.775]$) and spectra ($0.513$, $[0.353,0.702]$). The in-distribution
verifier is therefore a useful triage signal for microscopy and stability claims, a weak one for
magnetism and diffraction, and not useful for metallicity or spectra. Its
cross-constraint-type generalization is poor for all six constraint types: diffraction, microscopy, and stability
are indistinguishable from chance when held out entirely (AUC 0.50, 0.53, and 0.44,
bootstrap 95\% confidence intervals straddling 0.5), spectra is not separable from chance either (AUC
0.56, CI $[0.488,0.626]$), magnetism is at chance (AUC $0.503$, $[0.430,0.576]$), and metallicity is the single arguable exception, with a nominal
95\% interval that excludes chance (AUC 0.58, CI $[0.510,0.652]$) but that no longer does so
once the interval is widened for the six held-out comparisons (Bonferroni-corrected interval
$[0.487,0.673]$). We therefore treat metallicity as a nominal, uncorrected exception rather
than as evidence of transfer, and do not build any claim on it. No held-out constraint type shows transfer that survives correction for the six comparisons. We report the verifier as an
in-distribution tool for the families where it discriminates, and explicitly do \emph{not} claim it as a general-purpose,
constraint-agnostic violation detector. A pooled AUC across all six held-out constraint
types is not reported: the six held-out runs are separately fine-tuned models, so their raw
probability outputs are not on a directly comparable scale; we report only the six
per-constraint-type numbers (Table~\ref{tab:verifier}, Fig.~\ref{fig:results}c). The held-out results are sensitive to the training run. The released adapters, retrained
with the same data, splits and seed but in a different numerical precision (Supplementary Note~2),
reproduce the in-distribution AUC exactly ($0.850$) but move the held-out AUCs by up to $0.08$: held-out
microscopy rises to $0.602$, whose interval excludes chance even after Bonferroni correction
($[0.519,0.682]$), and held-out stability falls to $0.383$, below chance. No held-out family reaches a
level useful for triage in either run, which is the conclusion we draw. The verifier was trained on
oracle labels produced under the original stability rule; evaluated only on the test items that the final
oracle scores, its in-distribution AUC is $0.871$ ($[0.834,0.904]$; stability family $0.911$) and its
held-out stability AUC $0.372$, so the conclusions above do not depend on the near-hull band
(Supplementary Note~2).

\begin{table}[t]
\centering
\caption{\justifying Verifier precision/recall/$F_1$ (in-distribution and held-out-constraint-type
splits, both source-disjoint) and the detector's bootstrap 95\% CI on ROC AUC, predicting a
real oracle violation from held-out verifier scores. Expected calibration error is
reported in the text for the in-distribution split (ECE $=0.096$). No pooled
across-constraint-type AUC is reported (see text).}
\label{tab:verifier}
{\scriptsize\setlength{\tabcolsep}{5pt}\renewcommand{\arraystretch}{1.12}
\begin{tabular}{l S[table-format=1.3] S[table-format=1.3] S[table-format=1.3] S[table-format=1.3] c}
\toprule
& \multicolumn{3}{c}{Violation prediction} & \multicolumn{2}{c}{Detector (ROC)} \\
\cmidrule(lr){2-4}\cmidrule(lr){5-6}
Split & {Precision} & {Recall} & {$F_1$} & {AUC} & {95\% CI} \\
\midrule
\multicolumn{6}{l}{\itshape In distribution (80/20 split, source-disjoint)}\\
All constraint types & 0.768 & 0.706 & 0.735 & \bfseries 0.850 & {[}0.813, 0.883{]} \\
\addlinespace
\multicolumn{6}{l}{\itshape Held out entirely (train on the other five constraint types)}\\
Diffraction (XRD)      & 0.600 & 0.030 & 0.057 & 0.497 & {[}0.419, 0.578{]} \\
Microscopy (SEM)       & 0.239 & 0.905 & 0.378 & 0.527 & {[}0.468, 0.589{]} \\
Spectra (Raman)        & {--} & 0.000 & {--} & 0.558 & {[}0.488, 0.626{]} \\
Stability$^{*}$        & 0.355 & 0.103 & 0.159 & 0.436 & {[}0.370, 0.502{]} \\
Metallicity$^{*}$      & 0.486 & 0.158 & 0.238 & 0.582 & {[}0.510, 0.652{]} \\
Magnetism$^{*}$        & 0.468 & 0.635 & 0.539 & 0.503 & {[}0.430, 0.576{]} \\
\bottomrule
\end{tabular}}
\end{table}

\begin{table}[t]
\centering
\caption{\justifying In-distribution verifier performance within each constraint family (the
source-disjoint 80/20 test split of Table~\ref{tab:verifier}, broken down by family; scores are those of
the single in-distribution model). AUC 95\% CI: bootstrap over source items ($B=5000$).
Viol.: share of test records that are oracle violations. A baseline that scores each record by
its family's training-set violation rate reaches a pooled AUC of 0.759, against the verifier's
0.850. Labels follow the original stability rule; on the items the final rule scores, the stability AUC is
0.911 (Supplementary Note~2).}
\label{tab:verifierfam}
{\scriptsize\setlength{\tabcolsep}{5pt}\renewcommand{\arraystretch}{1.12}
\begin{tabular}{l r r S[table-format=1.3] S[table-format=1.3] S[table-format=1.3] S[table-format=1.3] c}
\toprule
Family & $n$ & {Viol.} & {Precision} & {Recall} & {$F_1$} & {AUC} & {95\% CI} \\
\midrule
Diffraction (XRD)  & 33  & 70\% & 0.846 & 0.478 & 0.611 & 0.696 & {[}0.493, 0.895{]} \\
Microscopy (SEM)   & 109 & 19\% & 0.750 & 0.714 & 0.732 & 0.922 & {[}0.852, 0.977{]} \\
Spectra (Raman)    & 100 & 88\% & 0.880 & 1.000 & 0.936 & 0.513 & {[}0.353, 0.702{]} \\
Stability$^{*}$    & 87  & 31\% & 0.727 & 0.296 & 0.421 & 0.804 & {[}0.683, 0.903{]} \\
Metallicity$^{*}$  & 55  & 35\% & 0.364 & 0.421 & 0.390 & 0.582 & {[}0.390, 0.775{]} \\
Magnetism$^{*}$    & 47  & 40\% & 0.600 & 0.474 & 0.529 & 0.728 & {[}0.562, 0.877{]} \\
\bottomrule
\end{tabular}}
\end{table}

\begin{figure}[tbp]
\centering
\adjustbox{max size={0.975\linewidth}{0.62\textheight}}{%
\begin{tikzpicture}[font=\sffamily]
\node[plab,anchor=west] at (-0.10,7.62){a)};
\node[font=\sffamily\small\bfseries,anchor=west] at (0.55,7.62){Physical consistency by model and constraint type};
\draw[black!45,line width=0.5pt](2.57,7.10)--(6.89,7.10);
\draw[black!45,line width=0.5pt](6.92,7.10)--(11.24,7.10);
\node[font=\sffamily\fontsize{6.5}{8}\selectfont,black!60,anchor=south] at (4.73,7.14){geometric oracles};
\node[font=\sffamily\fontsize{6.5}{8}\selectfont,black!60,anchor=south] at (9.08,7.14){DFT-grounded oracles};
\node[font=\sffamily\fontsize{6.8}{8}\selectfont,anchor=south] at (3.28,6.60){Diffraction};
\node[font=\sffamily\fontsize{6}{7}\selectfont,black!55,anchor=south] at (3.28,6.30){(XRD)};
\node[font=\sffamily\fontsize{6.8}{8}\selectfont,anchor=south] at (4.73,6.60){Microscopy};
\node[font=\sffamily\fontsize{6}{7}\selectfont,black!55,anchor=south] at (4.73,6.30){(SEM)};
\node[font=\sffamily\fontsize{6.8}{8}\selectfont,anchor=south] at (6.18,6.60){Spectra};
\node[font=\sffamily\fontsize{6}{7}\selectfont,black!55,anchor=south] at (6.18,6.30){(Raman)};
\node[font=\sffamily\fontsize{6.8}{8}\selectfont,anchor=south] at (7.63,6.60){Stability};
\node[font=\sffamily\fontsize{6}{7}\selectfont,black!55,anchor=south] at (7.63,6.30){$^{*}$};
\node[font=\sffamily\fontsize{6.8}{8}\selectfont,anchor=south] at (9.08,6.60){Metallicity};
\node[font=\sffamily\fontsize{6}{7}\selectfont,black!55,anchor=south] at (9.08,6.30){$^{*}$};
\node[font=\sffamily\fontsize{6.8}{8}\selectfont,anchor=south] at (10.53,6.60){Magnetism};
\node[font=\sffamily\fontsize{6}{7}\selectfont,black!55,anchor=south] at (10.53,6.30){$^{*}$};
\node[font=\sffamily\fontsize{6.8}{8}\selectfont\bfseries,anchor=south] at (12.36,6.60){Overall};
\node[font=\sffamily\fontsize{6}{7}\selectfont,black!55,anchor=south] at (12.36,6.30){(all six)};
\node[font=\sffamily\fontsize{6.8}{8}\selectfont,anchor=east] at (2.45,5.23){GPT-5.6-Terra$^{\dagger}$};
\node[minimum width=1.42cm,minimum height=0.56cm,fill=npblue!71,draw=white,line width=1pt,inner sep=0,font=\sffamily\fontsize{7.2}{8}\selectfont,text=white] at (3.28,5.23){0.797};
\node[minimum width=1.42cm,minimum height=0.56cm,fill=npblue!84,draw=white,line width=1pt,inner sep=0,font=\sffamily\fontsize{7.2}{8}\selectfont,text=white] at (4.73,5.23){0.951};
\node[minimum width=1.42cm,minimum height=0.56cm,fill=npblue!74,draw=white,line width=1pt,inner sep=0,font=\sffamily\fontsize{7.2}{8}\selectfont,text=white] at (6.18,5.23){0.833};
\node[minimum width=1.42cm,minimum height=0.56cm,fill=npblue!60,draw=white,line width=1pt,inner sep=0,font=\sffamily\fontsize{7.2}{8}\selectfont,text=white] at (7.63,5.23){0.653};
\node[minimum width=1.42cm,minimum height=0.56cm,fill=npblue!60,draw=white,line width=1pt,inner sep=0,font=\sffamily\fontsize{7.2}{8}\selectfont,text=white] at (9.08,5.23){0.659};
\node[minimum width=1.42cm,minimum height=0.56cm,fill=npblue!69,draw=white,line width=1pt,inner sep=0,font=\sffamily\fontsize{7.2}{8}\selectfont,text=white] at (10.53,5.23){0.767};
\node[minimum width=1.42cm,minimum height=0.56cm,fill=npblue!70,draw=black!45,line width=0.6pt,inner sep=0,font=\sffamily\fontsize{7.2}{8}\selectfont\bfseries,text=white] at (12.36,5.23){0.777};
\node[font=\sffamily\fontsize{6.8}{8}\selectfont,anchor=east] at (2.45,5.85){GPT-5.6-Sol$^{\ddagger}$};
\node[minimum width=1.42cm,minimum height=0.56cm,fill=npblue!73,draw=white,line width=1pt,inner sep=0,font=\sffamily\fontsize{7.2}{8}\selectfont,text=white] at (3.28,5.85){0.820};
\node[minimum width=1.42cm,minimum height=0.56cm,fill=npblue!81,draw=white,line width=1pt,inner sep=0,font=\sffamily\fontsize{7.2}{8}\selectfont,text=white] at (4.73,5.85){0.915};
\node[minimum width=1.42cm,minimum height=0.56cm,fill=npblue!69,draw=white,line width=1pt,inner sep=0,font=\sffamily\fontsize{7.2}{8}\selectfont,text=white] at (6.18,5.85){0.764};
\node[minimum width=1.42cm,minimum height=0.56cm,fill=npblue!71,draw=white,line width=1pt,inner sep=0,font=\sffamily\fontsize{7.2}{8}\selectfont,text=white] at (7.63,5.85){0.789};
\node[minimum width=1.42cm,minimum height=0.56cm,fill=npblue!61,draw=white,line width=1pt,inner sep=0,font=\sffamily\fontsize{7.2}{8}\selectfont,text=white] at (9.08,5.85){0.670};
\node[minimum width=1.42cm,minimum height=0.56cm,fill=npblue!66,draw=white,line width=1pt,inner sep=0,font=\sffamily\fontsize{7.2}{8}\selectfont,text=white] at (10.53,5.85){0.733};
\node[minimum width=1.42cm,minimum height=0.56cm,fill=npblue!70,draw=black!45,line width=0.6pt,inner sep=0,font=\sffamily\fontsize{7.2}{8}\selectfont\bfseries,text=white] at (12.36,5.85){0.778};
\node[font=\sffamily\fontsize{6.8}{8}\selectfont,anchor=east] at (2.45,4.62){GPT-4o};
\node[minimum width=1.42cm,minimum height=0.56cm,fill=npblue!57,draw=white,line width=1pt,inner sep=0,font=\sffamily\fontsize{7.2}{8}\selectfont,text=black] at (3.28,4.62){0.625};
\node[minimum width=1.42cm,minimum height=0.56cm,fill=npblue!84,draw=white,line width=1pt,inner sep=0,font=\sffamily\fontsize{7.2}{8}\selectfont,text=white] at (4.73,4.62){0.954};
\node[minimum width=1.42cm,minimum height=0.56cm,fill=npblue!12,draw=white,line width=1pt,inner sep=0,font=\sffamily\fontsize{7.2}{8}\selectfont,text=black] at (6.18,4.62){0.068};
\node[minimum width=1.42cm,minimum height=0.56cm,fill=npblue!79,draw=white,line width=1pt,inner sep=0,font=\sffamily\fontsize{7.2}{8}\selectfont,text=white] at (7.63,4.62){0.886};
\node[minimum width=1.42cm,minimum height=0.56cm,fill=npblue!61,draw=white,line width=1pt,inner sep=0,font=\sffamily\fontsize{7.2}{8}\selectfont,text=white] at (9.08,4.62){0.671};
\node[minimum width=1.42cm,minimum height=0.56cm,fill=npblue!62,draw=white,line width=1pt,inner sep=0,font=\sffamily\fontsize{7.2}{8}\selectfont,text=white] at (10.53,4.62){0.678};
\node[minimum width=1.42cm,minimum height=0.56cm,fill=npblue!63,draw=black!45,line width=0.6pt,inner sep=0,font=\sffamily\fontsize{7.2}{8}\selectfont\bfseries,text=white] at (12.36,4.62){0.699};
\node[font=\sffamily\fontsize{6.8}{8}\selectfont,anchor=east] at (2.45,4.00){Claude Sonnet 4.5};
\node[minimum width=1.42cm,minimum height=0.56cm,fill=npblue!51,draw=white,line width=1pt,inner sep=0,font=\sffamily\fontsize{7.2}{8}\selectfont,text=black] at (3.28,4.00){0.547};
\node[minimum width=1.42cm,minimum height=0.56cm,fill=npblue!86,draw=white,line width=1pt,inner sep=0,font=\sffamily\fontsize{7.2}{8}\selectfont,text=white] at (4.73,4.00){0.977};
\node[minimum width=1.42cm,minimum height=0.56cm,fill=npblue!21,draw=white,line width=1pt,inner sep=0,font=\sffamily\fontsize{7.2}{8}\selectfont,text=black] at (6.18,4.00){0.189};
\node[minimum width=1.42cm,minimum height=0.56cm,fill=npblue!76,draw=white,line width=1pt,inner sep=0,font=\sffamily\fontsize{7.2}{8}\selectfont,text=white] at (7.63,4.00){0.849};
\node[minimum width=1.42cm,minimum height=0.56cm,fill=npblue!63,draw=white,line width=1pt,inner sep=0,font=\sffamily\fontsize{7.2}{8}\selectfont,text=white] at (9.08,4.00){0.694};
\node[minimum width=1.42cm,minimum height=0.56cm,fill=npblue!68,draw=white,line width=1pt,inner sep=0,font=\sffamily\fontsize{7.2}{8}\selectfont,text=white] at (10.53,4.00){0.756};
\node[minimum width=1.42cm,minimum height=0.56cm,fill=npblue!61,draw=black!45,line width=0.6pt,inner sep=0,font=\sffamily\fontsize{7.2}{8}\selectfont\bfseries,text=white] at (12.36,4.00){0.667};
\node[font=\sffamily\fontsize{6.8}{8}\selectfont,anchor=east] at (2.45,3.39){Claude Opus 4.6$^{\dagger}$};
\node[minimum width=1.42cm,minimum height=0.56cm,fill=npblue!53,draw=white,line width=1pt,inner sep=0,font=\sffamily\fontsize{7.2}{8}\selectfont,text=black] at (3.28,3.39){0.578};
\node[minimum width=1.42cm,minimum height=0.56cm,fill=npblue!84,draw=white,line width=1pt,inner sep=0,font=\sffamily\fontsize{7.2}{8}\selectfont,text=white] at (4.73,3.39){0.954};
\node[minimum width=1.42cm,minimum height=0.56cm,fill=npblue!26,draw=white,line width=1pt,inner sep=0,font=\sffamily\fontsize{7.2}{8}\selectfont,text=black] at (6.18,3.39){0.244};
\node[minimum width=1.42cm,minimum height=0.56cm,fill=npblue!71,draw=white,line width=1pt,inner sep=0,font=\sffamily\fontsize{7.2}{8}\selectfont,text=white] at (7.63,3.39){0.789};
\node[minimum width=1.42cm,minimum height=0.56cm,fill=npblue!65,draw=white,line width=1pt,inner sep=0,font=\sffamily\fontsize{7.2}{8}\selectfont,text=white] at (9.08,3.39){0.721};
\node[minimum width=1.42cm,minimum height=0.56cm,fill=npblue!63,draw=white,line width=1pt,inner sep=0,font=\sffamily\fontsize{7.2}{8}\selectfont,text=white] at (10.53,3.39){0.700};
\node[minimum width=1.42cm,minimum height=0.56cm,fill=npblue!60,draw=black!45,line width=0.6pt,inner sep=0,font=\sffamily\fontsize{7.2}{8}\selectfont\bfseries,text=white] at (12.36,3.39){0.663};
\node[font=\sffamily\fontsize{6.8}{8}\selectfont,anchor=east] at (2.45,2.77){Gemini 2.5 Flash Lite};
\node[minimum width=1.42cm,minimum height=0.56cm,fill=npblue!48,draw=white,line width=1pt,inner sep=0,font=\sffamily\fontsize{7.2}{8}\selectfont,text=black] at (3.28,2.77){0.517};
\node[minimum width=1.42cm,minimum height=0.56cm,fill=npblue!88,draw=white,line width=1pt,inner sep=0,font=\sffamily\fontsize{7.2}{8}\selectfont,text=white] at (4.73,2.77){1.000};
\node[minimum width=1.42cm,minimum height=0.56cm,fill=npblue!29,draw=white,line width=1pt,inner sep=0,font=\sffamily\fontsize{7.2}{8}\selectfont,text=black] at (6.18,2.77){0.278};
\node[minimum width=1.42cm,minimum height=0.56cm,fill=npblue!73,draw=white,line width=1pt,inner sep=0,font=\sffamily\fontsize{7.2}{8}\selectfont,text=white] at (7.63,2.77){0.814};
\node[minimum width=1.42cm,minimum height=0.56cm,fill=npblue!64,draw=white,line width=1pt,inner sep=0,font=\sffamily\fontsize{7.2}{8}\selectfont,text=white] at (9.08,2.77){0.702};
\node[minimum width=1.42cm,minimum height=0.56cm,fill=npblue!47,draw=white,line width=1pt,inner sep=0,font=\sffamily\fontsize{7.2}{8}\selectfont,text=black] at (10.53,2.77){0.494};
\node[minimum width=1.42cm,minimum height=0.56cm,fill=npblue!58,draw=black!45,line width=0.6pt,inner sep=0,font=\sffamily\fontsize{7.2}{8}\selectfont\bfseries,text=black] at (12.36,2.77){0.630};
\node[font=\sffamily\fontsize{6.8}{8}\selectfont,anchor=east] at (2.45,2.16){Qwen3-VL-8B-Thinking$^{\dagger}$};
\node[minimum width=1.42cm,minimum height=0.56cm,fill=npblue!53,draw=white,line width=1pt,inner sep=0,font=\sffamily\fontsize{7.2}{8}\selectfont,text=black] at (3.28,2.16){0.578};
\node[minimum width=1.42cm,minimum height=0.56cm,fill=npblue!79,draw=white,line width=1pt,inner sep=0,font=\sffamily\fontsize{7.2}{8}\selectfont,text=white] at (4.73,2.16){0.885};
\node[minimum width=1.42cm,minimum height=0.56cm,fill=npblue!14,draw=white,line width=1pt,inner sep=0,font=\sffamily\fontsize{7.2}{8}\selectfont,text=black] at (6.18,2.16){0.092};
\node[minimum width=1.42cm,minimum height=0.56cm,fill=npblue!79,draw=white,line width=1pt,inner sep=0,font=\sffamily\fontsize{7.2}{8}\selectfont,text=white] at (7.63,2.16){0.895};
\node[minimum width=1.42cm,minimum height=0.56cm,fill=npblue!59,draw=white,line width=1pt,inner sep=0,font=\sffamily\fontsize{7.2}{8}\selectfont,text=black] at (9.08,2.16){0.644};
\node[minimum width=1.42cm,minimum height=0.56cm,fill=npblue!47,draw=white,line width=1pt,inner sep=0,font=\sffamily\fontsize{7.2}{8}\selectfont,text=black] at (10.53,2.16){0.506};
\node[minimum width=1.42cm,minimum height=0.56cm,fill=npblue!55,draw=black!45,line width=0.6pt,inner sep=0,font=\sffamily\fontsize{7.2}{8}\selectfont\bfseries,text=black] at (12.36,2.16){0.594};
\node[font=\sffamily\fontsize{6.8}{8}\selectfont,anchor=east] at (2.45,1.54){Qwen2.5-VL-7B};
\node[minimum width=1.42cm,minimum height=0.56cm,fill=npblue!47,draw=white,line width=1pt,inner sep=0,font=\sffamily\fontsize{7.2}{8}\selectfont,text=black] at (3.28,1.54){0.500};
\node[minimum width=1.42cm,minimum height=0.56cm,fill=npblue!61,draw=white,line width=1pt,inner sep=0,font=\sffamily\fontsize{7.2}{8}\selectfont,text=white] at (4.73,1.54){0.667};
\node[minimum width=1.42cm,minimum height=0.56cm,fill=npblue!18,draw=white,line width=1pt,inner sep=0,font=\sffamily\fontsize{7.2}{8}\selectfont,text=black] at (6.18,1.54){0.148};
\node[minimum width=1.42cm,minimum height=0.56cm,fill=npblue!81,draw=white,line width=1pt,inner sep=0,font=\sffamily\fontsize{7.2}{8}\selectfont,text=white] at (7.63,1.54){0.917};
\node[minimum width=1.42cm,minimum height=0.56cm,fill=npblue!50,draw=white,line width=1pt,inner sep=0,font=\sffamily\fontsize{7.2}{8}\selectfont,text=black] at (9.08,1.54){0.540};
\node[minimum width=1.42cm,minimum height=0.56cm,fill=npblue!69,draw=white,line width=1pt,inner sep=0,font=\sffamily\fontsize{7.2}{8}\selectfont,text=white] at (10.53,1.54){0.767};
\node[minimum width=1.42cm,minimum height=0.56cm,fill=npblue!54,draw=black!45,line width=0.6pt,inner sep=0,font=\sffamily\fontsize{7.2}{8}\selectfont\bfseries,text=black] at (12.36,1.54){0.584};
\node[font=\sffamily\fontsize{6.8}{8}\selectfont,anchor=east] at (2.45,0.92){Qwen3-VL-8B-Instruct$^{\dagger}$};
\node[minimum width=1.42cm,minimum height=0.56cm,fill=npblue!44,draw=white,line width=1pt,inner sep=0,font=\sffamily\fontsize{7.2}{8}\selectfont,text=black] at (3.28,0.92){0.468};
\node[minimum width=1.42cm,minimum height=0.56cm,fill=npblue!72,draw=white,line width=1pt,inner sep=0,font=\sffamily\fontsize{7.2}{8}\selectfont,text=white] at (4.73,0.92){0.802};
\node[minimum width=1.42cm,minimum height=0.56cm,fill=npblue!17,draw=white,line width=1pt,inner sep=0,font=\sffamily\fontsize{7.2}{8}\selectfont,text=black] at (6.18,0.92){0.133};
\node[minimum width=1.42cm,minimum height=0.56cm,fill=npblue!67,draw=white,line width=1pt,inner sep=0,font=\sffamily\fontsize{7.2}{8}\selectfont,text=white] at (7.63,0.92){0.747};
\node[minimum width=1.42cm,minimum height=0.56cm,fill=npblue!62,draw=white,line width=1pt,inner sep=0,font=\sffamily\fontsize{7.2}{8}\selectfont,text=white] at (9.08,0.92){0.688};
\node[minimum width=1.42cm,minimum height=0.56cm,fill=npblue!50,draw=white,line width=1pt,inner sep=0,font=\sffamily\fontsize{7.2}{8}\selectfont,text=black] at (10.53,0.92){0.533};
\node[minimum width=1.42cm,minimum height=0.56cm,fill=npblue!52,draw=black!45,line width=0.6pt,inner sep=0,font=\sffamily\fontsize{7.2}{8}\selectfont\bfseries,text=black] at (12.36,0.92){0.556};
\foreach \i in {0,...,39}{\pgfmathsetmacro{\ff}{6+\i*2.205}\fill[npblue!\ff] (4.3+\i*0.1,0.00) rectangle (4.3+\i*0.1+0.1,0.00+0.20);}
\draw[black!40,line width=0.4pt](4.3,0.00) rectangle (8.3,0.00+0.20);
\draw[black!40,line width=0.4pt](4.30,0.00)--(4.30,0.00-0.06);
\node[font=\sffamily\fontsize{6}{7}\selectfont,anchor=north] at (4.30,0.00-0.07){0.00};
\draw[black!40,line width=0.4pt](5.30,0.00)--(5.30,0.00-0.06);
\node[font=\sffamily\fontsize{6}{7}\selectfont,anchor=north] at (5.30,0.00-0.07){0.25};
\draw[black!40,line width=0.4pt](6.30,0.00)--(6.30,0.00-0.06);
\node[font=\sffamily\fontsize{6}{7}\selectfont,anchor=north] at (6.30,0.00-0.07){0.50};
\draw[black!40,line width=0.4pt](7.30,0.00)--(7.30,0.00-0.06);
\node[font=\sffamily\fontsize{6}{7}\selectfont,anchor=north] at (7.30,0.00-0.07){0.75};
\draw[black!40,line width=0.4pt](8.30,0.00)--(8.30,0.00-0.06);
\node[font=\sffamily\fontsize{6}{7}\selectfont,anchor=north] at (8.30,0.00-0.07){1.00};
\node[font=\sffamily\fontsize{6.5}{8}\selectfont,anchor=east] at (4.1499999999999995,0.00+0.10){Physical-Consistency Rate};
\node[font=\sffamily\fontsize{6}{7}\selectfont,black!55,anchor=west] at (8.600000000000001,0.00+0.10){$^{\dagger}$later-generation refresh\ \ $^{\ddagger}$flagship addition};
\begin{scope}[shift={(0,-7.17)}]
\node[plab,anchor=west] at (-0.10,5.62){b)};
\node[font=\sffamily\fontsize{7.6}{9}\selectfont\bfseries,anchor=west] at (0.52,5.62){Oracle feedback raises PCR};
\def\yb#1{0.62 + (#1-0.30)/0.40*3.85}
\draw[black!45,line width=0.5pt](1.05,0.62)--(1.05,4.47);
\node[font=\sffamily\fontsize{6.3}{7}\selectfont,anchor=east] at (0.99,{\yb{0.3}}){0.3};
\draw[black!45,line width=0.4pt](1.00,{\yb{0.3}})--(1.18,{\yb{0.3}});
\node[font=\sffamily\fontsize{6.3}{7}\selectfont,anchor=east] at (0.99,{\yb{0.4}}){0.4};
\draw[black!45,line width=0.4pt](1.00,{\yb{0.4}})--(1.18,{\yb{0.4}});
\node[font=\sffamily\fontsize{6.3}{7}\selectfont,anchor=east] at (0.99,{\yb{0.5}}){0.5};
\draw[black!45,line width=0.4pt](1.00,{\yb{0.5}})--(1.18,{\yb{0.5}});
\node[font=\sffamily\fontsize{6.3}{7}\selectfont,anchor=east] at (0.99,{\yb{0.6}}){0.6};
\draw[black!45,line width=0.4pt](1.00,{\yb{0.6}})--(1.18,{\yb{0.6}});
\node[font=\sffamily\fontsize{6.3}{7}\selectfont,anchor=east] at (0.99,{\yb{0.7}}){0.7};
\draw[black!45,line width=0.4pt](1.00,{\yb{0.7}})--(1.18,{\yb{0.7}});
\draw[black!10,line width=0.4pt](1.18,{\yb{0.4}})--(3.70,{\yb{0.4}});
\draw[black!10,line width=0.4pt](1.18,{\yb{0.5}})--(3.70,{\yb{0.5}});
\draw[black!10,line width=0.4pt](1.18,{\yb{0.6}})--(3.70,{\yb{0.6}});
\draw[black!10,line width=0.4pt](1.18,{\yb{0.7}})--(3.70,{\yb{0.7}});
\node[font=\sffamily\fontsize{6.5}{8}\selectfont,rotate=90,anchor=south] at (0.42,2.55){PCR};
\node[font=\sffamily\fontsize{6.8}{8}\selectfont,anchor=north] at (1.45,0.54){before};
\node[font=\sffamily\fontsize{6.8}{8}\selectfont,anchor=north] at (3.05,0.54){after};
\draw[black!45,line width=1.0pt,dash pattern=on 2pt off 1.4pt](1.45,{\yb{0.349}})--(3.05,{\yb{0.339}});
\fill[black!45](1.45,{\yb{0.349}})circle(2.0pt);\fill[black!45](3.05,{\yb{0.339}})circle(2.0pt);
\draw[black!45,line width=1.0pt,dash pattern=on 2pt off 1.4pt](1.45,{\yb{0.358}})--(3.05,{\yb{0.342}});
\fill[black!45](1.45,{\yb{0.358}})circle(2.0pt);\fill[black!45](3.05,{\yb{0.342}})circle(2.0pt);
\draw[npblue,line width=1.5pt](1.45,{\yb{0.338}})--(3.05,{\yb{0.439}});
\fill[npblue](1.45,{\yb{0.338}})circle(2.4pt);\draw[white,line width=0.5pt](1.45,{\yb{0.338}})circle(2.4pt);
\fill[npblue](3.05,{\yb{0.439}})circle(2.4pt);\draw[white,line width=0.5pt](3.05,{\yb{0.439}})circle(2.4pt);
\draw[nporange,line width=1.5pt](1.45,{\yb{0.521}})--(3.05,{\yb{0.632}});
\fill[nporange](1.45,{\yb{0.521}})circle(2.4pt);\draw[white,line width=0.5pt](1.45,{\yb{0.521}})circle(2.4pt);
\fill[nporange](3.05,{\yb{0.632}})circle(2.4pt);\draw[white,line width=0.5pt](3.05,{\yb{0.632}})circle(2.4pt);
\draw[npgreen,line width=1.5pt](1.45,{\yb{0.434}})--(3.05,{\yb{0.593}});
\fill[npgreen](1.45,{\yb{0.434}})circle(2.4pt);\draw[white,line width=0.5pt](1.45,{\yb{0.434}})circle(2.4pt);
\fill[npgreen](3.05,{\yb{0.593}})circle(2.4pt);\draw[white,line width=0.5pt](3.05,{\yb{0.593}})circle(2.4pt);
\draw[nppurple,line width=1.5pt](1.45,{\yb{0.373}})--(3.05,{\yb{0.639}});
\fill[nppurple](1.45,{\yb{0.373}})circle(2.4pt);\draw[white,line width=0.5pt](1.45,{\yb{0.373}})circle(2.4pt);
\fill[nppurple](3.05,{\yb{0.639}})circle(2.4pt);\draw[white,line width=0.5pt](3.05,{\yb{0.639}})circle(2.4pt);
\draw[nppurple!55,line width=0.4pt](3.20,{\yb{0.639}+0.34})--(3.11,{\yb{0.639}});
\node[font=\sffamily\fontsize{6.3}{7}\selectfont,anchor=west,text=nppurple] at (3.24,{\yb{0.639}+0.34}){Claude Sonnet 4.5$^{***}$};
\node[font=\sffamily\fontsize{6.3}{7}\selectfont,anchor=west,text=nporange] at (3.24,{\yb{0.632}-0.06}){Gemini 2.5 FL$^{***}$};
\node[font=\sffamily\fontsize{6.3}{7}\selectfont,anchor=west,text=npgreen] at (3.24,{\yb{0.593}-0.03}){GPT-4o$^{***}$};
\node[font=\sffamily\fontsize{6.3}{7}\selectfont,anchor=west,text=npblue] at (3.24,{\yb{0.439}}){Qwen2.5-VL-7B$^{***}$};
\node[font=\sffamily\fontsize{6.3}{7}\selectfont,anchor=west,text=black!55,align=left] at (3.24,{\yb{0.341}}){self-refinement\\generic re-prompt\\(both n.s.)};
\node[font=\sffamily\fontsize{5.8}{7}\selectfont,anchor=west,text=black!55] at (0.18,0.06){$^{***}$Holm-adjusted $p<10^{-3}$; CGSV $b=0$ by design (stopping rule)};
\end{scope}
\begin{scope}[shift={(6.05,-7.17)}]
\node[plab,anchor=west] at (-0.10,5.62){c)};
\node[font=\sffamily\fontsize{7.6}{9}\selectfont\bfseries,anchor=west] at (0.52,5.62){Violation detection versus baselines};
\def\xa#1{2.55 + (#1-0.25)/0.65*4.05}
\draw[npslate!65,dash pattern=on 2pt off 1.6pt,line width=0.6pt]({\xa{0.5}},0.42)--({\xa{0.5}},5.05);
\node[font=\sffamily\fontsize{6.3}{8}\selectfont\itshape,black!65,anchor=west] at (0.02,5.12){In distribution};
\node[font=\sffamily\fontsize{6.5}{8}\selectfont,anchor=east] at (2.42,4.70){MatPCR verifier};
\draw[npgreen,line width=1.1pt]({\xa{0.813}},4.70)--({\xa{0.883}},4.70);
\draw[npgreen,line width=1.0pt]({\xa{0.813}},4.63)--({\xa{0.813}},4.77);
\draw[npgreen,line width=1.0pt]({\xa{0.883}},4.63)--({\xa{0.883}},4.77);
\fill[npgreen]({\xa{0.85}},4.70)circle(2.4pt);\draw[white,line width=0.5pt]({\xa{0.85}},4.70)circle(2.4pt);
\node[font=\sffamily\fontsize{6.2}{7}\selectfont,anchor=west] at (6.72,4.70){0.850};
\node[font=\sffamily\fontsize{6.5}{8}\selectfont,anchor=east] at (2.42,4.34){Zero-shot LLM judge};
\draw[npslate,line width=1.1pt]({\xa{0.4}},4.34)--({\xa{0.493}},4.34);
\draw[npslate,line width=1.0pt]({\xa{0.4}},4.27)--({\xa{0.4}},4.41);
\draw[npslate,line width=1.0pt]({\xa{0.493}},4.27)--({\xa{0.493}},4.41);
\fill[npslate]({\xa{0.447}},4.34)circle(2.4pt);\draw[white,line width=0.5pt]({\xa{0.447}},4.34)circle(2.4pt);
\node[font=\sffamily\fontsize{6.2}{7}\selectfont,anchor=west] at (6.72,4.34){0.447};
\node[font=\sffamily\fontsize{6.5}{8}\selectfont,anchor=east] at (2.42,3.98){Model self-confidence};
\draw[npslate,line width=1.1pt]({\xa{0.264}},3.98)--({\xa{0.365}},3.98);
\draw[npslate,line width=1.0pt]({\xa{0.264}},3.91)--({\xa{0.264}},4.05);
\draw[npslate,line width=1.0pt]({\xa{0.365}},3.91)--({\xa{0.365}},4.05);
\fill[npslate]({\xa{0.315}},3.98)circle(2.4pt);\draw[white,line width=0.5pt]({\xa{0.315}},3.98)circle(2.4pt);
\node[font=\sffamily\fontsize{6.2}{7}\selectfont,anchor=west] at (6.72,3.98){0.315};
\node[font=\sffamily\fontsize{6.3}{8}\selectfont\itshape,black!65,anchor=west] at (0.02,3.62){Held out (unseen constraint type)};
\node[font=\sffamily\fontsize{6.5}{8}\selectfont,anchor=east] at (2.42,3.20){Metallicity$^{*}$};
\draw[nporange,line width=1.1pt]({\xa{0.51}},3.20)--({\xa{0.652}},3.20);
\draw[nporange,line width=1.0pt]({\xa{0.51}},3.13)--({\xa{0.51}},3.27);
\draw[nporange,line width=1.0pt]({\xa{0.652}},3.13)--({\xa{0.652}},3.27);
\fill[nporange]({\xa{0.582}},3.20)circle(2.4pt);\draw[white,line width=0.5pt]({\xa{0.582}},3.20)circle(2.4pt);
\node[font=\sffamily\fontsize{6.2}{7}\selectfont,anchor=west] at (6.72,3.20){0.582};
\node[font=\sffamily\fontsize{6.5}{8}\selectfont,anchor=east] at (2.42,2.84){Spectra (Raman)};
\draw[npslate,line width=1.1pt]({\xa{0.488}},2.84)--({\xa{0.626}},2.84);
\draw[npslate,line width=1.0pt]({\xa{0.488}},2.77)--({\xa{0.488}},2.91);
\draw[npslate,line width=1.0pt]({\xa{0.626}},2.77)--({\xa{0.626}},2.91);
\fill[npslate]({\xa{0.558}},2.84)circle(2.4pt);\draw[white,line width=0.5pt]({\xa{0.558}},2.84)circle(2.4pt);
\node[font=\sffamily\fontsize{6.2}{7}\selectfont,anchor=west] at (6.72,2.84){0.558};
\node[font=\sffamily\fontsize{6.5}{8}\selectfont,anchor=east] at (2.42,2.48){Microscopy (SEM)};
\draw[npslate,line width=1.1pt]({\xa{0.468}},2.48)--({\xa{0.589}},2.48);
\draw[npslate,line width=1.0pt]({\xa{0.468}},2.41)--({\xa{0.468}},2.55);
\draw[npslate,line width=1.0pt]({\xa{0.589}},2.41)--({\xa{0.589}},2.55);
\fill[npslate]({\xa{0.527}},2.48)circle(2.4pt);\draw[white,line width=0.5pt]({\xa{0.527}},2.48)circle(2.4pt);
\node[font=\sffamily\fontsize{6.2}{7}\selectfont,anchor=west] at (6.72,2.48){0.527};
\node[font=\sffamily\fontsize{6.5}{8}\selectfont,anchor=east] at (2.42,2.12){Magnetism$^{*}$};
\draw[npslate,line width=1.1pt]({\xa{0.43}},2.12)--({\xa{0.576}},2.12);
\draw[npslate,line width=1.0pt]({\xa{0.43}},2.05)--({\xa{0.43}},2.19);
\draw[npslate,line width=1.0pt]({\xa{0.576}},2.05)--({\xa{0.576}},2.19);
\fill[npslate]({\xa{0.503}},2.12)circle(2.4pt);\draw[white,line width=0.5pt]({\xa{0.503}},2.12)circle(2.4pt);
\node[font=\sffamily\fontsize{6.2}{7}\selectfont,anchor=west] at (6.72,2.12){0.503};
\node[font=\sffamily\fontsize{6.5}{8}\selectfont,anchor=east] at (2.42,1.76){Diffraction (XRD)};
\draw[npslate,line width=1.1pt]({\xa{0.419}},1.76)--({\xa{0.578}},1.76);
\draw[npslate,line width=1.0pt]({\xa{0.419}},1.69)--({\xa{0.419}},1.83);
\draw[npslate,line width=1.0pt]({\xa{0.578}},1.69)--({\xa{0.578}},1.83);
\fill[npslate]({\xa{0.497}},1.76)circle(2.4pt);\draw[white,line width=0.5pt]({\xa{0.497}},1.76)circle(2.4pt);
\node[font=\sffamily\fontsize{6.2}{7}\selectfont,anchor=west] at (6.72,1.76){0.497};
\node[font=\sffamily\fontsize{6.5}{8}\selectfont,anchor=east] at (2.42,1.40){Stability$^{*}$};
\draw[npslate,line width=1.1pt]({\xa{0.37}},1.40)--({\xa{0.502}},1.40);
\draw[npslate,line width=1.0pt]({\xa{0.37}},1.33)--({\xa{0.37}},1.47);
\draw[npslate,line width=1.0pt]({\xa{0.502}},1.33)--({\xa{0.502}},1.47);
\fill[npslate]({\xa{0.436}},1.40)circle(2.4pt);\draw[white,line width=0.5pt]({\xa{0.436}},1.40)circle(2.4pt);
\node[font=\sffamily\fontsize{6.2}{7}\selectfont,anchor=west] at (6.72,1.40){0.436};
\draw[black!45,line width=0.5pt](2.55,0.42)--(6.60,0.42);
\draw[black!45,line width=0.4pt]({\xa{0.3}},0.42)--({\xa{0.3}},0.35);
\node[font=\sffamily\fontsize{6.2}{7}\selectfont,anchor=north] at ({\xa{0.3}},0.33){0.3};
\draw[black!45,line width=0.4pt]({\xa{0.4}},0.42)--({\xa{0.4}},0.35);
\node[font=\sffamily\fontsize{6.2}{7}\selectfont,anchor=north] at ({\xa{0.4}},0.33){0.4};
\draw[black!45,line width=0.4pt]({\xa{0.5}},0.42)--({\xa{0.5}},0.35);
\node[font=\sffamily\fontsize{6.2}{7}\selectfont,anchor=north] at ({\xa{0.5}},0.33){0.5};
\draw[black!45,line width=0.4pt]({\xa{0.6}},0.42)--({\xa{0.6}},0.35);
\node[font=\sffamily\fontsize{6.2}{7}\selectfont,anchor=north] at ({\xa{0.6}},0.33){0.6};
\draw[black!45,line width=0.4pt]({\xa{0.7}},0.42)--({\xa{0.7}},0.35);
\node[font=\sffamily\fontsize{6.2}{7}\selectfont,anchor=north] at ({\xa{0.7}},0.33){0.7};
\draw[black!45,line width=0.4pt]({\xa{0.8}},0.42)--({\xa{0.8}},0.35);
\node[font=\sffamily\fontsize{6.2}{7}\selectfont,anchor=north] at ({\xa{0.8}},0.33){0.8};
\draw[black!45,line width=0.4pt]({\xa{0.9}},0.42)--({\xa{0.9}},0.35);
\node[font=\sffamily\fontsize{6.2}{7}\selectfont,anchor=north] at ({\xa{0.9}},0.33){0.9};
\node[font=\sffamily\fontsize{6.5}{8}\selectfont,anchor=north] at (4.55,0.06){Detector AUC with bootstrap 95\% CI};
\node[font=\sffamily\fontsize{5.8}{7}\selectfont,anchor=west,text=npgreen!55!black] at (0.30,-0.42){\rule{6pt}{1.5pt}\,CI above chance};
\node[font=\sffamily\fontsize{5.8}{7}\selectfont,anchor=west,text=nporange!80!black] at (2.55,-0.42){\rule{6pt}{1.5pt}\,nominal CI only};
\node[font=\sffamily\fontsize{5.8}{7}\selectfont,anchor=west,text=npslate] at (4.65,-0.42){\rule{6pt}{1.5pt}\,not above chance};
\end{scope}
\end{tikzpicture}}
\caption{\justifying\textbf{Physical consistency, oracle-driven correction, and violation detection.} \textbf{a)} Physical-Consistency Rate (PCR) for nine models across the six label-free constraint types and overall, ordered by overall PCR; shading is linear in PCR. Stability uses the final oracle, which does not score near-hull materials. Cells hold 44--90 checkable records; intervals and coverage are in Supplementary Tables~S1 and S2. $^{\dagger}$Later-generation refresh models; $^{\ddagger}$flagship-tier addition. \textbf{b)} PCR before and after Constraint-Grounded Self-Verification on identical items, with the two matched controls, self-refinement and equal-compute re-prompting (grey, dashed) (Tables~\ref{tab:cgsv} and~\ref{tab:cgsvacc}). \textbf{c)} Detector AUC with bootstrap 95\% intervals for the released verifier, two baselines and the six held-out constraint types, on source-disjoint splits (Table~\ref{tab:verifier}).}
\label{fig:results}
\end{figure}

\FloatBarrier

\section{Discussion}
Across nine models spanning two capability generations and six label-free constraint
types, current models are far from
uniformly consistent (overall PCR 0.556--0.778), ruling out the scenario in which the
benchmark would be uninformative because every model already satisfies the physics we
check; spectra and diffraction, in particular, remain hard and highly model-dependent. The
design separates differences in overall PCR larger than about 0.07 (the smallest difference
detectable with 80\% power at $\sim$500 checkable items per model), and within a single constraint
type only differences of roughly 0.11--0.20. At that resolution, one generational difference is
clear: GPT-5.6-Terra exceeds GPT-4o by 0.072 (95\% CI $[0.018,0.125]$, Holm-corrected $p=0.045$), a
difference that comes from the image constraint types and does not appear on the paired DFT item set.
Three other contrasts we examined, Claude Opus 4.6 versus Claude Sonnet 4.5, an explicit reasoning
mode (Qwen3-VL-8B-Thinking versus its non-reasoning sibling), and flagship versus balanced tier
within one generation (GPT-5.6-Sol versus GPT-5.6-Terra), show no detectable difference. These are
single pairs of models, and a null result here excludes only differences larger than about 0.07,
so we do not draw conclusions about capability scaling, reasoning modes or pricing tiers in general;
testing them would need more items per model, more model pairs, and current frontier releases (the
refresh models are one to three provider releases behind the newest available on the collection
date). We leave these questions to future work. The accuracy-versus-consistency gap (a mean 23.6\%
disagreement across the diffraction and spectra cells) shows
that PCR and answer accuracy measure different properties of the same response, although part of
the gap is definitional because the accuracy criteria are tighter than the oracle tolerances: models are frequently
consistent while measurably wrong on diffraction. For stability, the strict and the operational
definitions of "stable" differ only on near-hull materials, where neither answer is unambiguously right;
excluding that band, as the final oracle does, makes the two definitions agree and also removes a conflict
between the original oracle and its own prompt (Methods); separately, one model's near-ceiling raw stability
accuracy is a base-rate artifact of the sampled materials, not discriminative skill, once
balanced accuracy is computed. The same base-rate effect appears in PCR itself: on the DFT
constraint types, the ordering of models differs between pooled and class-balanced PCR, so both
should be reported. CGSV reduces oracle-detected violations for all four models tested, an effect
that unaided self-refinement and equal-compute re-prompting do not reproduce. Its absence of
consistency regressions is a property of the stopping rule rather than a finding. Paired
before--after scoring shows a small but significant net gain in exact correctness, but most
revisions that clear the oracle move the stated value into the admissible region rather than onto
the exact one; CGSV is best understood as a filter that removes physically inadmissible claims, not
as a general corrector of physical reasoning. The released verifier is a useful in-distribution
detector for microscopy and stability claims, weak for magnetism and diffraction, and not useful
for metallicity or spectra; part of its pooled in-distribution AUC ($0.850$) reflects differences in
violation base rate between families (a family-only baseline reaches $0.759$). Its
cross-constraint-type generalization is the paper's clearest
negative result: in the reported runs five of the six held-out constraint types are indistinguishable from
chance, and the sixth (metallicity) separates from chance only before correcting for the
six comparisons; in the released, retrained adapters held-out microscopy is weakly above chance
(AUC $0.602$) and none of the six reaches a useful level, consistent with our stated position that a verifier trained on one
constraint type's surface patterns should not be assumed to transfer to another
without direct evidence. Comparator baselines sharpen this picture. In distribution
neither is competitive: a zero-shot LLM-as-judge reaches AUC $0.447$ and the model's own
token-level self-confidence $0.315$, the latter significantly below chance, against the
fine-tuned verifier's $0.850$. Held out, all three are weak, but not uniformly: the
LLM-judge reaches AUC $0.735$ (95\% CI $[0.682,0.790]$) on microscopy, above anything the
fine-tuned verifier achieves on any held-out constraint type, which suggests the fine-tuned
verifier's advantage is specific to the distribution it was trained on rather than a general
superiority over prompting. The central limitation is intrinsic to the approach, not
particular to these results: the oracles bound a class of physically checkable
errors, so PCR is a necessary-not-sufficient reliability signal, its recall on real responses
is not established by the positive control, and in the present benchmark it scores one stated
quantity per chain rather than every inferential step, and the released
verifier predicts the oracle signal rather than replacing first-principles ground
truth. As a clearly scoped extension, the verifier could pre-filter candidate
structures before DFT in a small screening study; a first best-of-3 resampling check finds
real but modest headroom over a random pick (PCR 0.363 random-pick vs.\ 0.419 oracle-oracle
ceiling), though we did not build the verifier-as-selector comparison this ceiling would
need to be practically useful, and report it only as a bounded, partial result. A further implication,
which we make possible but do not claim to demonstrate here, is that a label-free,
physics-grounded consistency signal can serve as a verifiable reward for training
materials-reasoning models, extending reinforcement learning from verifiable rewards
to a domain where verification has been under-specified.\cite{scirlvr2026} We also ran a
domain-specific structure-aware model (MatterChat) on the DFT arms as a materials-domain
evaluee (Baselines subsection): it reaches PCR 0.783/0.795/0.850 on stability/metallicity/
magnetism on a shuffled sample of 120 real materials. These rates sit at or above the
general-purpose models' rates on the same three constraint types, but the sample is drawn
separately from the main evaluation, so the comparison is indicative rather than paired,
and it rests on one domain model. Finally,
every DFT-grounded number in this paper (the DFT column of Table~\ref{tab:main}, the
stability arm of the accuracy-versus-consistency analysis, and the metallicity and
magnetism arms of the verifier and detector analyses) rests on Materials Project
reference values and oracle tolerances, and the DFT tasks give the model a chemical formula
rather than a structure, so where polymorphs of a composition differ in the checked property the
reference entry is not identified by the prompt (Methods).

\section{Methods}
\subsection{Data sources}
MatPCR draws only on public data for reproducibility and license clarity: experimental
powder diffraction patterns with known structures;\cite{opxrd2025} public scanning electron microscopy (SEM) images that include scale bars;\cite{aversa2018nffa} open Raman spectra;\cite{lafuente2015rruff} and public first-principles data (computed energies, convex hulls, band gaps) from established databases.\cite{matproj2013} No number reported in this paper rests on a first-principles calculation run by the authors: every DFT-grounded reference value used is Materials Project's own standardized computational result, retrieved via the \texttt{mp-api} client (Methods, DFT-grounded oracles); an authors'-own-DFT fallback for missing or questionable Materials Project entries was implemented but never invoked over the reported sample. We also surveyed the simulated diffraction dataset SimXRD-4M\cite{simxrd2024} as a candidate diffraction source, but its schema (units of the tabulated lattice-distance field; the meaning of its integer label) could not be resolved against its documentation with confidence, so it was not used in any reported result; all diffraction numbers in this paper are from real experimental
patterns only.

\subsection{Tasks, modalities, and exposed constraints}
Each item pairs an input with a question whose faithful answer requires at least one
physically checkable quantity (Fig.~\ref{fig:pipeline}a). Inputs span image items (VLM), such as micrographs and diffraction images, and structure or text items (LLM), such as a crystal structure with a property question. We parse the reasoning chain for stated quantities (peak positions, scale, feature sizes, stability claims, property values) and test each against the matching oracle. An item is checkable if its chain exposes at least one such quantity.

\subsection{Programmatic geometric oracles}
For diffraction we apply Bragg's law, $\lambda = 2 d \sin\theta$, converting any
claimed $2\theta$ to a $d$-spacing, testing claimed peaks against the reference
pattern of the stated structure within tolerance, and rejecting claims outside the
realizable $2\theta$ window. For microscopy we recover nanometres per pixel from the
detected scale bar and reject any feature larger than the field of view. For spectra
we accept a claimed peak only if a genuine interior local maximum (higher than both
adjacent flanks down to their nearest valleys, rejecting a monotonic rise or a shoulder on a nearby peak) exists near it with prominence exceeding a fraction of the signal's dynamic range (the development history of this criterion is summarized in Supplementary Note~2).

\emph{Tolerances relative to the measurements.} We compared each tolerance with the sampling
and peak widths of the data it is applied to (all patterns and spectra used in the main evaluation
or the precision audit). The processed RRUFF spectra are sampled every $0.48\,\mathrm{cm}^{-1}$, and
the full width at half maximum of each spectrum's strongest peak has a median of
$16.9\,\mathrm{cm}^{-1}$ (interquartile range 11.6--24.3; $n=383$). The $\pm5\,\mathrm{cm}^{-1}$ spectrum
tolerance is therefore about ten sampling steps but narrower than the half-width of a typical peak,
so a claim that lands on the flank of a real peak rather than its apex can be flagged; this matches
the false positives found in the audit, 34 of 36 of which lie within $10\,\mathrm{cm}^{-1}$ of a real
peak. The $\pm2\,\mathrm{cm}^{-1}$ accuracy criterion is narrower still and targets the apex itself.
The opXRD patterns are sampled every $0.013^\circ$ in $2\theta$ (median), and the full width at half
maximum of each pattern's strongest peak has a median of $0.12^\circ$ (interquartile range 0.06--0.21;
$n=176$). The $0.5^\circ$ diffraction tolerance is therefore about four times a typical peak width,
generous relative to the measurement; it is instead set by what can be read from the rendered plot
shown to the model, whose $x$ axis spans a median $0.14^\circ$ per pixel, so the tolerance corresponds to
about three to four pixels. The $0.1^\circ$ accuracy criterion is close to a typical peak width and to
one pixel of the plot. These comparisons, together with the tolerance-sensitivity analysis (Baselines),
show which conclusions depend on the chosen tolerance: the diffraction model ranking does, the
spectrum ranking does not.

\subsection{DFT-grounded property oracles}
First-principles access extends the oracle from geometry to computed physics, still
without human labels. A stability claim is checked against the energy above the
convex hull computed for the stated composition; a metallicity claim against the
computed band gap; and a magnetism claim against the computed total magnetization,
normalized per atom. Reference values are Materials Project's own standardized
computational references (database version 2026.04.13, retrieved via the
\texttt{mp-api} client as a 3{,}000-material batch on 2026-08-05 for the reported evaluation; the
precision-audit sample was drawn from a retrieval of the same database version on 2026-07-16, and all
210 audited materials appear with the same identifiers in the 2026-08-05 batch; GGA/GGA+U relaxation under Materials Project's published standard workflow), described throughout as standardized computational references, not experimental ground truth. The DFT tasks give the model only a
chemical formula (the prompts are released with the code), and each claim is checked against one
Materials Project entry for that formula. Because polymorphs of one composition can differ in hull
energy, band gap and magnetization, the formula does not always determine the reference class: among
the 540 compositions that occur more than once in the 3{,}000-material batch, the polymorphs disagree
on the scored stability class for 0.8\% (4 of 478 with scorable entries; most stability disagreement
between polymorphs falls in the near-hull band, which is not scored), on the computed band-gap class for
15.6\% and on the net-magnetization class for 16.1\%. To bound the effect on the reported rates, we retrieved every Materials Project
entry for each evaluated composition (6{,}370 entries for 1{,}491 compositions, same database version)
and re-scored each claim against the composition's lowest-energy polymorph. The verdict changes for
12.6\% of stability claims (8.5\% because the ground-state entry lies in the near-hull band and the claim
would be excluded, 4.0\% because the verdict flips) and 10.5\% of metallicity claims (including 2.1\% that
fall in the abstention band under the ground-state entry), and scoring a claim as consistent if any
polymorph supports it changes 3.4\% and 12.0\%. Magnetism claims cannot be re-scored this way, because a
magnetic claim can fail on magnitude while its class matches. A structure-resolved prompt, giving the
model the specific structure rather than the formula, would remove this ambiguity and is a natural
extension. Because DFT-computed references depend on the exchange-correlation functional (semilocal functionals underestimate band gaps), each oracle is designed around that known error rather than
treating the reference as exact. \emph{Stability (near-hull criterion with an abstention band)}: a STABLE claim is consistent
for a material on the computed convex hull ($E_{\mathrm{hull}}=0$, to numerical precision) and an UNSTABLE
claim for a material above it by more than $\epsilon=0.05\,\mathrm{eV\,atom^{-1}}$. Materials with
$0<E_{\mathrm{hull}}\le\epsilon$ are near-hull or potentially metastable; for them neither answer is
unambiguously right, so they form an explicit abstention band and are excluded from PCR, as for
metallicity. This band corrects a conflict in the original oracle, exposed by the second annotation
(Results): the stability prompt instructs the model to treat any borderline or metastable case as
UNSTABLE, while the original oracle required STABLE for every $E_{\mathrm{hull}}\le\epsilon$, so a model
that followed the prompt on a near-hull material was counted as a violation (89 such records in the main
evaluation and 50 of the 107 audited cases). The original rule's results are reported in Supplementary
Table~S7. The threshold is an operational screening criterion rather than a definition of strict
thermodynamic stability. On a random sample of 3,000 Materials Project entries, the threshold and the
Materials Project \texttt{is\_stable} label disagreed for 571 structures (19.0\%), all of them in the
near-hull band (Materials Project labelled them unstable; the reverse did not occur), so under the
final rule the two definitions agree on every scored material. We additionally report sensitivity at
$\epsilon\in\{0,0.025,0.05,0.10\}\,\mathrm{eV\,atom^{-1}}$, varying the width of the abstention band; all
four widths gave 0\% false-negative and false-positive rates in the positive-control audit.
Figure~\ref{fig:dft-oracle-cases}a shows three model claims evaluated against this operational near-hull criterion. 
\emph{Computed band-gap classification}: a Materials Project band gap
$\leq0.05\,\mathrm{eV}$ is assigned to the computed zero/near-zero-gap
class, whereas a gap $\geq0.20\,\mathrm{eV}$ is assigned to the computed
nonzero-gap class. The intervening range
($0.05$--$0.20\,\mathrm{eV}$) is treated as an explicit abstention band
and excluded from PCR rather than forced into either class. These thresholds
were fixed before the reported evaluation runs. Figure~\ref{fig:dft-oracle-cases}b
shows three examples of model classifications that are inconsistent with the
corresponding Materials Project-computed band-gap references.
\emph{Magnetism (net-magnetization check)}: because a global reversal of all
collinear spins changes the sign of the total magnetization without changing the
physical state, the oracle uses the absolute total magnetization normalized per
atom. The Materials Project reference is classified as finite or near-zero net
magnetization using a threshold of
$0.05\,\mu_{\rm B}\,\mathrm{atom}^{-1}$, fixed before the reported runs.
Only within the finite-net-magnetization class is a claimed magnitude checked
against the reference within one order of magnitude
(Fig.~\ref{fig:dft-oracle-cases}c). This oracle concerns net magnetization only;
a near-zero net value does not exclude antiferromagnetic or compensated
ferrimagnetic ordering.
Positive-control testing showed all
three thresholds and the stability sensitivity grid give 0\% false-negative and
false-positive rates on real Materials Project structures; the metallicity abstention band
correctly excluded 125 of 3{,}000 sampled materials in that test, and the near-hull stability band 571. These checks make plausible-but-wrong physics detectable.
No step requires a human annotation of the correct answer, which is what lets MatPCR
scale and distinguishes it from answer-accuracy benchmarks and from proposal- or
execution-level verification.

\begin{figure}[!htbp]
\centering

\begin{tikzpicture}[font=\sffamily,scale=0.92,transform shape]
\def\cardw{3.95}
\def\cardh{3.35}
\def\cgap{0.45}

\node[anchor=west,font=\sffamily\bfseries\fontsize{8}{9}\selectfont]
at (0,\cardh+0.18) {(a) Stability oracle};

\foreach \i/\x/\formula/\claim/\trueval/\why in {
  0/0/{SrMoO$_3$}/{UNSTABLE}/{STABLE\\[1pt]{\fontsize{4.6}{5.6}\selectfont$E_{\rm hull}{=}0.0000$}}/{MP hull energy is exactly 0:\\structure on the computed hull\\judged unstable.},
  1/{\cardw+\cgap}/{ThPbAu$_2$}/{UNSTABLE}/{STABLE\\[1pt]{\fontsize{4.6}{5.6}\selectfont$E_{\rm hull}{=}0.0000$}}/{Same error on an unrelated\\composition: hull energy 0,\\still defaulted to UNSTABLE.},
  2/{2*\cardw+2*\cgap}/{EuZrFeBiO$_6$}/{STABLE}/{UNSTABLE\\[1pt]{\fontsize{4.6}{5.6}\selectfont$E_{\rm hull}{=}1.2843$}}/{Opposite error, far larger:\\hull energy is 25$\times$ the\\tolerance, unambiguous.}
}{
  \node[card,minimum width=\cardw cm,minimum height=\cardh cm,anchor=south west]
  at (\x,0){};

  \node[font=\sffamily\bfseries\fontsize{9}{10.5}\selectfont,
        text=npslate,anchor=north]
  at (\x+\cardw/2,\cardh-0.18){\formula};

  \node[font=\sffamily\fontsize{6.0}{7.2}\selectfont,
        black!55,anchor=north]
  at (\x+\cardw/2,\cardh-0.60){real Materials Project entry};

  \node[bad,text width=1.62cm,minimum height=0.80cm,align=center,
        font=\sffamily\bfseries\fontsize{5.6}{6.6}\selectfont,inner sep=1.5pt]
  (scl\i) at (\x+0.98,\cardh-1.62){\claim};

  \node[font=\sffamily\fontsize{5.4}{6.4}\selectfont,
        black!55,anchor=south]
  at (\x+0.98,\cardh-1.22){model claimed};

  \node[good,text width=1.62cm,minimum height=0.80cm,align=center,
        font=\sffamily\bfseries\fontsize{5.6}{6.6}\selectfont,inner sep=1.5pt]
  (srf\i) at (\x+\cardw-0.98,\cardh-1.62){\trueval};

  \node[font=\sffamily\fontsize{5.4}{6.4}\selectfont,
        black!55,anchor=south]
  at (\x+\cardw-0.98,\cardh-1.22){MP-computed reference};

  \draw[flow,black!55] (scl\i.east)--(srf\i.west);

  \node[npred,font=\sffamily\bfseries\fontsize{9}{10}\selectfont]
  at (\x+\cardw/2,\cardh-1.62){$\times$};

  \draw[black!15,line width=0.5pt]
  (\x+0.25,\cardh-2.18)--(\x+\cardw-0.25,\cardh-2.18);

  \node[font=\sffamily\fontsize{6.2}{7.8}\selectfont,
        black!70,anchor=north,align=center,text width=3.65cm]
  at (\x+\cardw/2,\cardh-2.34){\why};
}
\end{tikzpicture}

\vspace{2mm}

\begin{tikzpicture}[font=\sffamily,scale=0.92,transform shape]
\def\cardw{3.95}
\def\cardh{3.35}
\def\cgap{0.45}

\node[anchor=west,font=\sffamily\bfseries\fontsize{8}{9}\selectfont]
at (0,\cardh+0.18) {(b) Computed band-gap oracle};

\foreach \i/\x/\formula/\claim/\trueval/\why in {
  0/0/{KCaTbNbO$_6$}/{NONMETALLIC}/{METALLIC\\[1pt]{\fontsize{4.6}{5.6}\selectfont$E_g{=}0.0000$\,eV}}/{Materials Project reports\\a zero computed gap;\\reasoned from structure alone.},
  1/{\cardw+\cgap}/{Ca(H$_8$O$_5)_2$}/{METALLIC}/{NONMETALLIC\\[1pt]{\fontsize{4.6}{5.6}\selectfont$E_g{=}4.4694$\,eV}}/{Opposite error: the MP-computed gap\\is 4.47\,eV,\\over 20$\times$ the abstention edge.},
  2/{2*\cardw+2*\cgap}/{KMgMn$_{16}$O$_{32}$}/{NONMETALLIC}/{METALLIC\\[1pt]{\fontsize{4.6}{5.6}\selectfont$E_g{=}0.0000$\,eV}}/{Computed gap is exactly\\zero again; reasoning pointed\\the wrong direction.}
}{
  \node[card,minimum width=\cardw cm,minimum height=\cardh cm,anchor=south west]
  at (\x,0){};

  \node[font=\sffamily\bfseries\fontsize{9}{10.5}\selectfont,
        text=npslate,anchor=north]
  at (\x+\cardw/2,\cardh-0.18){\formula};

  \node[font=\sffamily\fontsize{6.0}{7.2}\selectfont,
        black!55,anchor=north]
  at (\x+\cardw/2,\cardh-0.60){real Materials Project entry};

  \node[bad,text width=1.62cm,minimum height=0.80cm,align=center,
        font=\sffamily\bfseries\fontsize{5.4}{6.4}\selectfont,inner sep=1.5pt]
  (mcl\i) at (\x+0.98,\cardh-1.62){\claim};

  \node[font=\sffamily\fontsize{5.4}{6.4}\selectfont,
        black!55,anchor=south]
  at (\x+0.98,\cardh-1.22){model claimed};

  \node[good,text width=1.62cm,minimum height=0.80cm,align=center,
        font=\sffamily\bfseries\fontsize{5.4}{6.4}\selectfont,inner sep=1.5pt]
  (mrf\i) at (\x+\cardw-0.98,\cardh-1.62){\trueval};

  \node[font=\sffamily\fontsize{5.4}{6.4}\selectfont,
        black!55,anchor=south]
  at (\x+\cardw-0.98,\cardh-1.22){MP-computed reference};

  \draw[flow,black!55] (mcl\i.east)--(mrf\i.west);

  \node[npred,font=\sffamily\bfseries\fontsize{9}{10}\selectfont]
  at (\x+\cardw/2,\cardh-1.62){$\times$};

  \draw[black!15,line width=0.5pt]
  (\x+0.25,\cardh-2.18)--(\x+\cardw-0.25,\cardh-2.18);

  \node[font=\sffamily\fontsize{6.2}{7.8}\selectfont,
        black!70,anchor=north,align=center,text width=3.65cm]
  at (\x+\cardw/2,\cardh-2.34){\why};
}
\end{tikzpicture}

\vspace{2mm}

\begin{tikzpicture}[font=\sffamily,scale=0.92,transform shape]
\def\cardw{3.95}
\def\cardh{3.35}
\def\cgap{0.45}

\node[anchor=west,font=\sffamily\bfseries\fontsize{8}{9}\selectfont]
at (0,\cardh+0.18) {(c) Net-magnetization oracle};

\foreach \i/\x/\formula/\claim/\trueval/\why in {
  0/0/{C}/{MAGNETIC\\[1pt]{\fontsize{4.6}{5.6}\selectfont 2.2\,\textmu B/atom}}/{NONMAGNETIC\\[1pt]{\fontsize{4.6}{5.6}\selectfont 0.0000\,\textmu B/atom}}/{Real moment is exactly zero;\\reasoning admits it never\\identified the formula.},
  1/{\cardw+\cgap}/{Eu$_2$Sb$_3$}/{NONMAGNETIC}/{MAGNETIC\\[1pt]{\fontsize{4.6}{5.6}\selectfont 2.76\,\textmu B/atom}}/{Real moment is large and\\unambiguous; reasoned Eu/Sb\\moments cancel, but they do not.},
  2/{2*\cardw+2*\cgap}/{Na$_2$FeO$_3$}/{NONMAGNETIC}/{MAGNETIC\\[1pt]{\fontsize{4.6}{5.6}\selectfont 0.667\,\textmu B/atom}}/{Real moment is well above\\threshold; own reasoning called\\this ``close to zero.''}
}{
  \node[card,minimum width=\cardw cm,minimum height=\cardh cm,anchor=south west]
  at (\x,0){};

  \node[font=\sffamily\bfseries\fontsize{9}{10.5}\selectfont,
        text=npslate,anchor=north]
  at (\x+\cardw/2,\cardh-0.18){\formula};

  \node[font=\sffamily\fontsize{6.0}{7.2}\selectfont,
        black!55,anchor=north]
  at (\x+\cardw/2,\cardh-0.60){real Materials Project entry};

  \node[bad,text width=1.62cm,minimum height=0.80cm,align=center,
        font=\sffamily\bfseries\fontsize{5.4}{6.4}\selectfont,inner sep=1.5pt]
  (gcl\i) at (\x+0.98,\cardh-1.62){\claim};

  \node[font=\sffamily\fontsize{5.4}{6.4}\selectfont,
        black!55,anchor=south]
  at (\x+0.98,\cardh-1.22){model claimed};

  \node[good,text width=1.62cm,minimum height=0.80cm,align=center,
        font=\sffamily\bfseries\fontsize{5.4}{6.4}\selectfont,inner sep=1.5pt]
  (grf\i) at (\x+\cardw-0.98,\cardh-1.62){\trueval};

  \node[font=\sffamily\fontsize{5.4}{6.4}\selectfont,
        black!55,anchor=south]
  at (\x+\cardw-0.98,\cardh-1.22){MP-computed reference};

  \draw[flow,black!55] (gcl\i.east)--(grf\i.west);

  \node[npred,font=\sffamily\bfseries\fontsize{9}{10}\selectfont]
  at (\x+\cardw/2,\cardh-1.62){$\times$};

  \draw[black!15,line width=0.5pt]
  (\x+0.25,\cardh-2.18)--(\x+\cardw-0.25,\cardh-2.18);

  \node[font=\sffamily\fontsize{6.2}{7.8}\selectfont,
        black!70,anchor=north,align=center,text width=3.65cm]
  at (\x+\cardw/2,\cardh-2.34){\why};
}
\end{tikzpicture}

\caption{\justifying\textbf{Representative Materials Project-grounded oracle
violations.}
\textbf{a}, Stability-oracle cases comparing model STABLE/UNSTABLE claims with
Materials Project energies above the convex hull, $E_{\rm hull}$. The first two
structures lie on the Materials Project computed convex hull, whereas
EuZrFeBiO$_6$ lies far above the
$0.05\,\mathrm{eV\,atom^{-1}}$ operational near-hull threshold.
\textbf{b}, Computed band-gap cases comparing model METALLIC/NONMETALLIC claims
with the operational zero/near-zero-gap and nonzero-gap classes derived from
Materials Project band gaps, $E_g$. These are computational reference
classifications, not experimental transport assignments.
\textbf{c}, Net-magnetization cases comparing model MAGNETIC/NONMAGNETIC claims
with the operational near-zero and finite classes derived from Materials Project
total magnetization per atom. This panel evaluates net magnetization only and
does not determine the complete magnetic ordering. All examples are real,
human-confirmed oracle activations from the DFT-grounded precision audit (Table~\ref{tab:exp1});
values are reported exactly as recorded.}
\label{fig:dft-oracle-cases}
\end{figure}

\subsection{The Physical-Consistency Rate}
Let items be indexed $i=1,\dots,N$, let $k_i$ be the number of physically checkable
constraints the chain exposes on item $i$, and $v_i \le k_i$ the number the oracles
find violated. The item is consistent iff $v_i=0$. Over the checkable subset
$\mathcal{C}=\{i:k_i>0\}$,
\begin{equation}
\mathrm{PCR} \;=\; \frac{1}{|\mathcal{C}|}\sum_{i\in\mathcal{C}} \mathbf{1}[v_i = 0].
\end{equation}
We report per-constraint-type violation rates
$\rho_t=\big(\sum_i v_{i,t}\big)/\big(\sum_i k_{i,t}\big)$ and per-modality rates with
percentile bootstrap confidence intervals. The definition admits several checkable claims per
item, but every item in this benchmark exposes exactly one claim of its own constraint type,
so $k_i=1$ throughout and $\rho_t$ reduces to $1-\mathrm{PCR}$ within a constraint type; we
keep the general form because the metric is intended for chains that expose more. Items fall
outside $\mathcal{C}$ for two different reasons, which we keep apart. An item is \emph{not evaluable} if the
oracle has no reference for it (a diffraction pattern without a usable structure, a micrograph without
a readable scale bar, a material in the metallicity abstention band or the near-hull stability band); this is a property of the item,
independent of the model. A response to an evaluable item is \emph{not checkable} if it states no
quantity the parser can extract. We report coverage, $|\mathcal{C}|$ divided by the number of
evaluable items, for every (model, constraint-type) cell, and because excluding non-checkable responses
could reward a model that states less, we also report PCR with non-checkable responses counted as
violations (penalised) and as consistent (lenient), together with the resulting model ordering
(Supplementary Table~S2). PCR is distinct from answer accuracy: a
model can be accurate yet inconsistent, or consistent yet wrong, and the gap is
itself reported.

\subsection{Constraint-Grounded Self-Verification}
CGSV uses the same oracles as a corrective signal at inference time with no labels.
Because a model's own self-verification is largely confirmatory rather than
corrective,\cite{selfverify2026} CGSV supplies the correction signal externally, from the
oracle, rather than asking the model to re-examine its answer unaided. The model extracts
the governing constraints, answers, is shown any programmatic or DFT-grounded violation,
and revises, up to $K$ rounds or until no violation remains (Algorithm~\ref{alg:cgsv}).
CGSV is a positive control: if the oracles capture real errors, exposing them should raise
PCR, measured by a paired test, not assumed. Writing $c$ for the items the oracle flags
before but not after revision and $b$ for those flagged after but not before, the change in
the rate is exactly $\Delta\mathrm{PCR}=(c-b)/n$, so the discordant pair $(b,c)$ carries all
of the effect and the exact McNemar test on it is the natural paired test
(Supplementary Note~1.6). The loop halts after at most $K$ rounds by construction. Because it
also halts as soon as no violation remains, an item that is consistent before revision is never
revised, so $b=0$ by construction and the McNemar test reduces to a test of $c$ against zero. A rise
in PCR shows only that revised claims pass the oracle more often; to see whether they are also more
correct we score each item's first- and final-round claim against an independent accuracy criterion
and test the paired change (Results, Table~\ref{tab:cgsvacc}).

\begin{algorithm}[t]
\caption{Constraint-Grounded Self-Verification}
\label{alg:cgsv}
\begin{algorithmic}[1]
\State \textbf{Input:} input $x$, question $q$, oracle set $\mathcal{O}$, max rounds $K$
\State $G \gets \textsc{ExtractConstraints}(x)$
\State $a \gets \textsc{Answer}(x,q,G)$
\For{$r = 1 \dots K$}
  \State $V \gets \{o \in \mathcal{O} : o(\textsc{Claims}(a), G)\ \text{fails}\}$
  \If{$V = \varnothing$} \State \textbf{break} \EndIf
  \State $a \gets \textsc{Revise}(x, q, G, a, V)$
\EndFor
\State \textbf{return} $a$, residual violations $V$
\end{algorithmic}
\end{algorithm}

\subsection{Released physical-consistency verifier}
Because the oracles require reference data (a structure, a scale bar, a computed
value), we fine-tune an open model to predict, from the input and chain alone,
whether a physical violation is present. Training labels come from the oracles, so the
verifier is trained without human annotation. Unlike process-reward and generative
reward models trained on human preference or model-judged
correctness,\cite{processbench2025,compassverifier2025} its labels are a deterministic
physics signal, which reduces the spurious correlations that affect preference-trained
reward models; where a reference exists, the oracle is hard to game by surface-form changes. We release it as the reusable artifact of this work and report precision,
recall, $F_1$, ROC AUC, and calibration against held-out oracle judgments.

\begin{figure}[tbp]
\centering
\adjustbox{max size={\linewidth}{0.66\textheight}}{%
\begin{tikzpicture}[font=\sffamily]
\node[plab,anchor=west] at (-0.10,6.05){a)};
\node[font=\sffamily\small\bfseries,anchor=west] at (0.55,6.05){From public data to a released verifier, with no human annotation};
\node[data,text width=2.25cm,minimum height=1.45cm,align=center,font=\sffamily\fontsize{6.2}{7.4}\selectfont,inner sep=2pt] (pub) at (1.345,3.15){\textbf{Public data}\\opXRD, RRUFF,\\NFFA-EUROPE,\\Materials Project};
\node[proc,text width=2.45cm,minimum height=1.45cm,align=center,font=\sffamily\fontsize{6.2}{7.4}\selectfont,inner sep=2pt] (build) at (4.185,3.15){\textbf{Items and oracles}\\Bragg, scale bar,\\peak, hull, gap,\\magnetization};
\node[proc,text width=1.80cm,minimum height=1.45cm,align=center,font=\sffamily\fontsize{6.2}{7.4}\selectfont,inner sep=2pt] (run) at (6.80,3.15){\textbf{Sample chains}\\9 models\\$S=3$ seeds};
\node[proc,fill=nppurple!12,draw=nppurple!65,text width=2.00cm,minimum height=1.45cm,align=center,font=\sffamily\fontsize{6.2}{7.4}\selectfont,inner sep=2pt] (chk) at (9.19,3.15){\textbf{Parse claims,}\\\textbf{apply oracles}\\consistent\\or violation};
\node[good,text width=2.15cm,minimum height=1.45cm,align=center,font=\sffamily\fontsize{6.2}{7.4}\selectfont,inner sep=2pt] (pcr) at (11.755,3.15){\textbf{PCR} and $\rho_t$\\bootstrap 95\% CI\\(Prop. 1 bounds\\the oracle bias)};
\node[acc,text width=5.60cm,minimum height=1.15cm,align=center,font=\sffamily\fontsize{6.2}{7.4}\selectfont,inner sep=2pt] (gate) at (9.19,5.15){\textbf{Oracle-validity gate}, run before any model\\planted violations recovered at 10\%, 30\%, 50\%\\(FN rate 0.0000); precision audited by hand};
\draw[flow](gate.south)--node[right,font=\sffamily\fontsize{5.8}{7}\selectfont,black,xshift=-1pt]{validates}(chk.north);
\node[acc,text width=2.30cm,minimum height=1.20cm,align=center,font=\sffamily\fontsize{6.2}{7.4}\selectfont,inner sep=2pt] (cgsv) at (6.80,1.05){\textbf{CGSV}: return the\\violation, revise ($K=2$)\\4 models, diffraction\\and spectra};
\node[proc,fill=npgreen!12,draw=npgreen!65,text width=2.15cm,minimum height=1.20cm,align=center,font=\sffamily\fontsize{6.2}{7.4}\selectfont,inner sep=2pt] (ver) at (11.755,1.05){Fine-tune on\\oracle labels\\$\rightarrow$ \textbf{released}\\\textbf{verifier}};
\node[font=\sffamily\fontsize{5.5}{6}\selectfont,npgreen!45!black,fill=white,draw=npgreen!55,rounded corners=1pt,inner sep=1pt] at ($(ver.north east)+(-0.18,0.02)$){open};
\draw[flow](pub)--(build);
\draw[flow](build)--(run);
\draw[flow](run)--(chk);
\draw[flow](chk)--(pcr);
\draw[flow](chk.south)--(9.19,2.05)--(8.40,2.05)--(8.40,1.05)--(cgsv.east);
\draw[flow](chk.south)--(9.19,2.05)--(11.755,2.05)--(ver.north);
\node[font=\sffamily\fontsize{5.8}{7}\selectfont,black!70,anchor=south] at (10.45,2.09){oracle judgments};
\draw[flow,nporange](cgsv.north)--node[left,font=\sffamily\fontsize{5.8}{7}\selectfont,black,align=right]{label-free\\feedback}(run.south);
\draw[black!25,line width=0.5pt,dash pattern=on 2pt off 1.5pt](0.12,0.20)--(13.00,0.20);
\node[font=\sffamily\fontsize{6.2}{7}\selectfont,black!60,anchor=west] at (0.12,-0.02){No step requires a human annotation of the correct answer; every training and scoring label above is derived from physics.};
\begin{scope}[shift={(0,-6.30)}]
\node[plab,anchor=west] at (-0.10,5.05){b)};
\node[font=\sffamily\fontsize{7.6}{9}\selectfont\bfseries,anchor=west] at (0.52,5.05){What PCR adds to accuracy};
\fill[npgreen!16](1.00,2.27) rectangle (3.60,3.82);
\node[align=center,font=\sffamily\fontsize{6.4}{7.6}\selectfont,text=npgreen!45!black] at (2.30,3.35){\textbf{right for the}\\\textbf{right reasons}};
\node[align=center,font=\sffamily\fontsize{5.7}{6.8}\selectfont,text=black!62] at (2.30,2.70){the target};
\fill[nporange!16](3.60,2.27) rectangle (6.20,3.82);
\node[align=center,font=\sffamily\fontsize{6.4}{7.6}\selectfont,text=nporange!60!black] at (4.90,3.35){\textbf{right despite a}\\\textbf{physics violation}};
\node[align=center,font=\sffamily\fontsize{5.7}{6.8}\selectfont,text=black!62] at (4.90,2.70){rare; none on stability\\under the final rule};
\fill[npblue!13](1.00,0.72) rectangle (3.60,2.27);
\node[align=center,font=\sffamily\fontsize{6.4}{7.6}\selectfont,text=npblue!55!black] at (2.30,1.80){\textbf{wrong but}\\\textbf{consistent}};
\node[align=center,font=\sffamily\fontsize{5.7}{6.8}\selectfont,text=black!62] at (2.30,1.15){dominates diffraction\\(up to 58.1\% of a cell)};
\fill[npred!13](3.60,0.72) rectangle (6.20,2.27);
\node[align=center,font=\sffamily\fontsize{6.4}{7.6}\selectfont,text=npred!70!black] at (4.90,1.80){\textbf{wrong and}\\\textbf{inconsistent}};
\node[align=center,font=\sffamily\fontsize{5.7}{6.8}\selectfont,text=black!62] at (4.90,1.15){dominates spectra};
\draw[black!35,line width=0.5pt](1.00,0.72) rectangle (6.20,3.82);
\draw[black!35,line width=0.5pt](3.60,0.72)--(3.60,3.82);
\draw[black!35,line width=0.5pt](1.00,2.27)--(6.20,2.27);
\node[font=\sffamily\fontsize{6.4}{8}\selectfont,anchor=south] at (2.30,3.92){\textcolor{npgreen!55!black}{\checkmark}\,consistent};
\node[font=\sffamily\fontsize{6.4}{8}\selectfont,anchor=south] at (4.90,3.92){\textcolor{npred}{$\times$}\,inconsistent};
\node[font=\sffamily\fontsize{6.4}{8}\selectfont,rotate=90,anchor=south] at (0.88,3.05){correct};
\node[font=\sffamily\fontsize{6.4}{8}\selectfont,rotate=90,anchor=south] at (0.88,1.50){incorrect};
\draw[decorate,decoration={brace,amplitude=4pt,mirror,raise=1pt},black!50,line width=0.5pt](1.00,0.62)--(6.20,0.62);
\node[font=\sffamily\fontsize{6.2}{7}\selectfont,anchor=north,text=npslate] at (3.60,0.42){resolved by PCR, without labels};
\draw[decorate,decoration={brace,amplitude=4pt,raise=1pt},black!50,line width=0.5pt](0.38,0.72)--(0.38,3.82);
\node[font=\sffamily\fontsize{6.2}{7}\selectfont,rotate=90,anchor=south,text=black!55] at (0.20,2.27){resolved by accuracy, needs labels};
\end{scope}
\begin{scope}[shift={(6.90,-6.30)}]
\node[plab,anchor=west] at (-0.10,5.05){c)};
\node[font=\sffamily\fontsize{7.6}{9}\selectfont\bfseries,anchor=west] at (0.52,5.05){Where the two axes disagree};
\def\xg#1{2.35 + #1/60*3.45}
\draw[npslate!70,dash pattern=on 2pt off 1.6pt,line width=0.6pt]({\xg{23.6}},0.95)--({\xg{23.6}},4.30);
\node[font=\sffamily\fontsize{5.8}{7}\selectfont,text=npslate,anchor=south,align=center] at ({\xg{23.6}},4.34){mean 23.6\%\\(8 cells)};
\node[font=\sffamily\fontsize{6.4}{8}\selectfont,anchor=east] at (2.22,3.7){Diffraction (XRD)};
\draw[npblue,line width=1.0pt,dash pattern=on 1.6pt off 1.4pt]({\xg{0}},3.7)--({\xg{58.1}},3.7);
\fill[npblue]({\xg{58.1}},3.7)circle(2.6pt);\draw[white,line width=0.5pt]({\xg{58.1}},3.7)circle(2.6pt);
\node[font=\sffamily\fontsize{5.9}{7}\selectfont,anchor=south east,text=black!65] at ({\xg{58.1}+0.16},3.8600000000000003){58.1\% largest cell};
\node[font=\sffamily\fontsize{6.4}{8}\selectfont,anchor=east] at (2.22,2.95){Stability$^{*}$};
\fill[npgreen]({\xg{0}},2.95)circle(2.4pt);
\node[font=\sffamily\fontsize{5.9}{7}\selectfont,anchor=west,text=black!65] at ({\xg{0}+0.12},2.95){0\%: the two coincide under the final rule};
\node[font=\sffamily\fontsize{6.4}{8}\selectfont,anchor=east] at (2.22,2.2){Spectra (Raman)};
\draw[nppurple,line width=2.6pt]({\xg{4.2}},2.2)--({\xg{11.2}},2.2);
\node[font=\sffamily\fontsize{5.9}{7}\selectfont,anchor=south,text=black!65] at ({\xg{7.699999999999999}},2.33){4.2--11.2\%};
\draw[black!45,line width=0.5pt]({\xg{0}},1.75)--({\xg{60}},1.75);
\draw[black!45,line width=0.4pt]({\xg{0}},1.75)--({\xg{0}},1.68);
\node[font=\sffamily\fontsize{6.2}{7}\selectfont,anchor=north] at ({\xg{0}},1.66){0\%};
\draw[black!45,line width=0.4pt]({\xg{20}},1.75)--({\xg{20}},1.68);
\node[font=\sffamily\fontsize{6.2}{7}\selectfont,anchor=north] at ({\xg{20}},1.66){20\%};
\draw[black!45,line width=0.4pt]({\xg{40}},1.75)--({\xg{40}},1.68);
\node[font=\sffamily\fontsize{6.2}{7}\selectfont,anchor=north] at ({\xg{40}},1.66){40\%};
\draw[black!45,line width=0.4pt]({\xg{60}},1.75)--({\xg{60}},1.68);
\node[font=\sffamily\fontsize{6.2}{7}\selectfont,anchor=north] at ({\xg{60}},1.66){60\%};
\node[font=\sffamily\fontsize{6.4}{8}\selectfont,anchor=north] at ({\xg{30}},1.36){items on which accuracy and PCR disagree};
\node[font=\sffamily\fontsize{5.8}{7}\selectfont,text=black!58,anchor=west,align=left] at (0.02,0.62){Microscopy is excluded: its prompt elicits only a feature-size claim,\\so no non-circular accuracy criterion is available.};
\end{scope}
\end{tikzpicture}}
\caption{\justifying\textbf{The MatPCR framework and what its metric adds to accuracy.} \textbf{a)} Items and oracles are built from public characterization data and Materials Project references; the oracles are validated before any model is run, by a positive control that plants violations at three known rates and a human audit of oracle-fired cases; reasoning chains are then parsed into claims, checked, and summarized as the Physical-Consistency Rate with bootstrap intervals, whose displacement from the true rate Proposition 1 makes exact. The same oracle judgments drive Constraint-Grounded Self-Verification, which returns each violation to the model for revision, and supply the training labels for the released verifier. No step uses a human annotation of the correct answer. \textbf{b)} A reasoning chain occupies one of four states. Accuracy resolves the rows and needs labels; PCR resolves the columns and does not. Models are frequently close enough to satisfy the diffraction oracle while measurably off the exact peak, and wrong on both counts for spectra; on stability the two axes coincide under the final rule, because the strict and operational definitions of stability differ only in the near-hull band, which is not scored. \textbf{c)} The resulting disagreement, as the percentage of items on which the two axes differ, across the four originally evaluated models. The mean is 23.6\% over the eight diffraction and spectra cells; the spread across constraint types shows that the two axes come apart for different reasons in each.}
\label{fig:pipeline}
\end{figure}

\subsection{What oracle error does to the reported rate}
An oracle that misses genuine violations or fires on clean chains moves the measured PCR away
from the rate we want to report. The result below makes that displacement exact, and is the
formal reason the oracle-validity gate is run before any model. Let
$\mathcal{C}$ be the set of checkable items with $|\mathcal{C}|=n$, where here $n$ denotes
the number of checkable items in the set under analysis (distinct from the per-cell item
count of the evaluation tables). For
$i\in\mathcal{C}$ let $Y_i\in\{0,1\}$ be the latent indicator that the chain contains
at least one genuine physical violation and $\hat Y_i\in\{0,1\}$ the oracle's decision.
The true inconsistency rate is $\tau=\frac1n\sum_i Y_i$, the true rate is
$\mathrm{PCR}^\star=1-\tau$, and the measured rate is
$\mathrm{PCR}=1-\frac1n\sum_i\hat Y_i$. Define the empirical item-level recall and
false-alarm rate of the oracle,
$\hat d=\#\{\hat Y_i{=}1,Y_i{=}1\}/\#\{Y_i{=}1\}$ and
$\hat f=\#\{\hat Y_i{=}1,Y_i{=}0\}/\#\{Y_i{=}0\}$ (a rate multiplies a zero
count to zero when its class is empty).

\begin{proposition}[Bias decomposition]\label{prop:bias}
$\mathrm{PCR}-\mathrm{PCR}^\star=\tau(1-\hat d)-(1-\tau)\hat f.$
\end{proposition}
\begin{proof}
Counting flagged items,
$\sum_i\hat Y_i=\hat d\,\#\{Y_i{=}1\}+\hat f\,\#\{Y_i{=}0\}=n[\tau\hat d+(1-\tau)\hat f]$,
so $\mathrm{PCR}=1-\tau\hat d-(1-\tau)\hat f$. Subtracting $\mathrm{PCR}^\star=1-\tau$
gives the result.
\end{proof}

\begin{corollary}[One-sided guarantees and bias bound]\label{cor:onesided}
If the oracle is complete $(\hat d=1)$ then $\mathrm{PCR}\le\mathrm{PCR}^\star$, so PCR
never overstates consistency; if it is sound $(\hat f=0)$ then
$\mathrm{PCR}\ge\mathrm{PCR}^\star$; if both, $\mathrm{PCR}=\mathrm{PCR}^\star$. In
general $|\mathrm{PCR}-\mathrm{PCR}^\star|\le(1-\hat d)+\hat f$.
\end{corollary}
\begin{proof}
The three cases substitute $\hat d=1$ and $\hat f=0$ into
Proposition~\ref{prop:bias}. For the bound,
$|\tau(1-\hat d)-(1-\tau)\hat f|\le\tau(1-\hat d)+(1-\tau)\hat f\le(1-\hat d)+\hat f$
because $\tau,\,1-\tau\le1$.
\end{proof}

\noindent Corollary~\ref{cor:onesided} is the formal reason to run the positive control
and precision audit first (Table~\ref{tab:exp1}). The positive control measures the
oracle's recall on planted violations, an optimistic proxy for $\hat d$. The precision audit estimates its
precision $\pi$ on fired cases, equivalently the false-discovery rate $1-\pi$; this is not
the false-alarm rate $\hat f$, whose denominator is the truly-clean items, and the two
coincide only through the base rate. The positive control measures recovery of injected
violations, not $\hat d$ on real responses, which would require auditing oracle-silent responses;
the following statements are therefore conditional on $\hat d=1$ holding for real violations.
Under that condition, Corollary~\ref{cor:onesided} gives $\mathrm{PCR}\le\mathrm{PCR}^\star$,
and the understatement is exactly the false-positive fraction $(1-\pi)(1-\mathrm{PCR})$, which the
audited precision estimates directly; estimating $\hat f$ itself, and hence the bias when
$\hat d<1$, additionally requires sampling oracle-silent items to fix the base rate.

Proposition~\ref{prop:bias} and Corollary~\ref{cor:onesided} concern the oracle's bias on
a fixed set of items and are exact: they involve no distributional assumption. The reported
PCR is also a finite sample, so its total error as an estimate of the rate the benchmark's
item-generating process would produce carries a second, statistical part. We report that part
empirically, as the percentile bootstrap interval on every rate, resampling at the source
level where items share a parent source (Supplementary Note~1.4). A distribution-free
alternative, obtained by adding a Hoeffding term to the bias bound above under independent
sampling, is stated in Supplementary Note~1.7; at the cell sizes used here it is looser than
the bootstrap interval, so we use it only to make the decomposition explicit and not as the
reported uncertainty.

\noindent Write $\hat Y^v_i\in\{0,1\}$ for the learned verifier's decision on item $i$, in
place of the oracle's $\hat Y_i$, and $\mathrm{PCR}^v=1-\frac1n\sum_i\hat Y^v_i$ for the rate
it reports. Substituting the verifier for the oracle when reference data are unavailable at
inference costs at most the two decisions' measurable disagreement rate
$\delta=\frac1n\#\{i:\hat Y^v_i\neq\hat Y_i\}$: both rates are means of indicators over the
same items, so $|\mathrm{PCR}^v-\mathrm{PCR}|=|\frac1n\sum_i(\hat Y_i-\hat Y^v_i)|
\le\frac1n\sum_i|\hat Y_i-\hat Y^v_i|=\delta$, and hence
$|\mathrm{PCR}^v-\mathrm{PCR}^\star|\le(1-\hat d)+\hat f+\delta$ by
Corollary~\ref{cor:onesided}. The verifier's precision and recall in
Table~\ref{tab:verifier} summarize that disagreement. Because that disagreement is
small only in distribution and large where transfer fails, the bound is tight in
distribution and loose out of it, which is exactly where reference data are absent; we
therefore present it as a guarantee for in-distribution use, not a license to replace the
oracle across constraint types. Standard background results
(Bragg's law, scale-bar calibration, convex-hull stability, the percentile bootstrap, a
concentration inequality for the empirical rate, and the exact McNemar test) are
collected in Supplementary Note~1.

\subsection{Models and experimental protocol}
We evaluate nine models in total across three passes. The original pass (data collected
2026-07-04 through 2026-07-11, re-run on corrected oracles 2026-07-16/17) covers one open
vision--language model, Qwen2.5-VL-7B-Instruct\cite{qwen25vl} (Hugging Face repository
\texttt{Qwen/Qwen2.5-VL-7B-Instruct}, run locally in \texttt{bfloat16} on an NVIDIA A100
80GB via \texttt{transformers}), and three frontier models accessed through the OpenRouter
API: \texttt{openai/gpt-4o}, \texttt{google/gemini-2.5-flash-lite}, and
\texttt{anthropic/claude-sonnet-4.5}. A later-generation refresh pass (data collected
2026-07-19/20) adds \texttt{anthropic/claude-opus-4.6},
\texttt{openai/gpt-5.6-terra}, \texttt{qwen/qwen3-vl-8b-instruct},
and a reasoning-mode condition, \texttt{qwen/qwen3-vl-8b-thinking}, all four again via
OpenRouter. This refresh is one to three provider releases behind the newest models
available on the collection date: within the same families, Claude Opus 4.7, 4.8 and 5,
and the flagship GPT-5.6-Sol tier, had all been released by 2026-07-20, and GPT-5.6-Terra
is the balanced rather than the flagship tier of its series. A third pass (data collected
2026-08-06) adds \texttt{openai/gpt-5.6-sol}, the flagship sibling of GPT-5.6-Terra, at the
same $n=30$, $S=3$ scale, to test directly whether flagship tier exceeds balanced tier
within the same generation; it does not (Discussion). We therefore describe the second
pass as a later-generation, not a frontier, comparison, and the third as answering the
flagship-tier question the second pass's absence of a flagship model had left open. Model strings, weights, and
API routing are pinned in the released repository; OpenRouter does not expose a
per-request snapshot date distinct from the model string for these providers, so the
listed model identifiers and the stated collection dates are the reproducibility anchor. We additionally evaluate a materials-domain
structure-aware model, MatterChat\cite{matterchat2026} (its released Mistral-7B-based
checkpoint and CHGNet structure encoder, run locally), on the DFT arms only (Baselines
subsection). No LLaVA-class model was run.

For each model and constraint type, we use $S=3$ independent item-sampling
draws of 30 items each. For the image constraint types, each draw is generated from a fixed
seed shared by all models, so all nine models answer the same items in a given seed and model
comparisons are paired. For the DFT constraint types, each (model, seed) cell takes the next
block of 30 materials from a fixed-seed shuffle of the Materials Project batch, so different
models answer different materials and comparisons are unpaired. We apply the oracles and report
PCR with percentile bootstrap confidence intervals ($B=10{,}000$ resamples);
CGSV uses the same decoding configuration and $K=2$ revision rounds. The local model uses greedy decoding (\texttt{do\_sample=False}); the API models are called with no explicit temperature or top-$p$ override, i.e.\ each provider's own default sampling settings for the given model string, which we did not tune. We ran the oracle-validity gate (Table~\ref{tab:exp1}) before any model evaluation: a positive control that injected synthetic violations into otherwise valid chains at known rates and confirmed the oracles recover them, followed by a manual audit of oracle-fired cases that quantified oracle precision on real chains. Tolerances, prompts, and the DFT protocol were fixed before the reported evaluation runs and are released with the code; every reported result uses the final oracle definitions given above (Supplementary Note~2 records how the spectrum and magnetism oracles reached their final form). The draws contain no duplicate source files within
any (constraint-type, model) cell, so pooling records does not duplicate items.

\subsection{Baselines, ablations, and controls}
Three of the study's controls are not consistently reported in comparable
verifier and reasoning-benchmark studies: the oracle positive control with a
precision audit, bootstrap confidence intervals with an exact paired test, and a
held-out-constraint generalization split, to the best of our knowledge as of this
writing. A physical-consistency metric, a released verifier, and a
self-correction loop are nonetheless interpretable only against the further
baselines and controls specified here.

\noindent\emph{Models.} We added the later-generation refresh and reasoning-mode
condition described above (Claude Opus 4.6, GPT-5.6-Terra, Qwen3-VL-8B-Instruct,
Qwen3-VL-8B-Thinking), and a materials-domain evaluee. For the domain evaluee we ran
MatterChat\cite{matterchat2026}, a structure-aware model that aligns a pretrained CHGNet
material encoder with a Mistral-7B language model, on the three DFT constraint types (its
native input, a crystal structure) using its released checkpoint and code, on a shuffled
sample of 120 real Materials Project structures (an unshuffled first sample, drawn in raw
file order, was 3.3\% magnetic versus the corrected sample's 65\%, so we discarded it and
redrew with shuffling before reporting). MatterChat reaches PCR 0.783 (stability), 0.795
(metallicity, $n=112$ checkable of 120), and 0.850 (magnetism, genuinely varying its answer
across materials rather than defaulting to one class, confirmed directly). We did not run a domain-specific model on the image arms.
Microscopy-specific vision--language models for materials do exist; we checked whether the
one cited here, MicroscopyGPT,\cite{choudhary2025microscopygpt} could serve as a domain
evaluee on the scale-bar arm and found its task and interface are incompatible with the
question this arm poses, not merely unevaluated. MicroscopyGPT is trained to caption
simulated scanning transmission electron microscopy images of 2D materials with a full
atomic structure (lattice parameters, element identities, and atomic coordinates), and its
released interface has no scale-bar-reading or feature-size-estimation output at all; forcing
it to answer a size question it was never built to produce would not evaluate the model, it
would evaluate an ad hoc adapter we would have to invent around it. We therefore leave the
image arms without a domain-specific evaluee and report this as an honest coverage gap. On reasoning mode: Qwen3-VL-8B-Thinking does not differ detectably from
Qwen3-VL-8B-Instruct overall (paired difference $+0.021$, 95\% CI $[-0.030,0.071]$, Holm $p=1.0$), so
this single pair gives no evidence that an explicit reasoning mode raises PCR by more than about 0.07 (Results).

\noindent\emph{Verifier baselines.} We compare the released verifier against a zero-shot
LLM-as-judge (Gemini 2.5 Flash Lite, prompted to judge each chain for a physical-consistency
violation with a stated confidence) and the evaluated local model's own token-level
self-confidence (mean log-probability of its stated final answer span), on the same
in-distribution and six held-out-constraint-type splits, 2{,}152 items. Neither baseline
is competitive with the fine-tuned verifier in distribution: the LLM-judge reaches
AUC $=0.447$ (95\% CI $[0.400,0.493]$, indistinguishable from or below chance) and
self-confidence AUC $=0.315$ (95\% CI $[0.264,0.365]$, significantly \emph{below} chance,
an anti-correlation we discuss below), against the fine-tuned verifier's in-distribution
AUC $=0.850$. Held out, all three approaches are weak: the LLM-judge ranges AUC
$0.489$--$0.735$ across the six constraint types (best on microscopy, worst on diffraction)
and self-confidence ranges AUC $0.246$--$0.523$, several below chance; neither baseline
recovers the fine-tuned verifier's in-distribution advantage out of distribution, and
none of the three approaches is a reliable held-out detector. We investigated the
below-chance self-confidence result directly rather than treat it as noise: on a 40-item
sample, chains later judged physically inconsistent had a \emph{higher} mean token
log-probability ($-2.84$) than chains judged consistent ($-3.99$), i.e.\ the evaluated
model is more confident, not less, when it is physically wrong, which is the source of the
below-chance AUC rather than a sign error in our scoring. We separately built reliability
diagrams (binned calibration curves) and a recommended per-split operating threshold for
the released verifier: in distribution its predicted probabilities are reasonably
calibrated (seven of ten reliability bins have a gap under 0.1 between predicted and
empirical violation rate; the two largest gaps are 0.198 in the 0.1--0.2 bin and 0.228 in
the 0.3--0.4 bin) and its default 0.5 threshold is close to optimal (best $F_1=0.765$ at threshold
0.25 versus $F_1=0.735$ at 0.5). Held out, however, every constraint type is severely
miscalibrated in the same direction: the verifier's low-probability bins massively
underestimate the true violation rate (for example, held-out diffraction: predicted
$0.07$--$0.14$ in its lowest bins, actual $0.47$--$0.49$; held-out spectra: predicted
$0.02$--$0.15$, actual $0.84$--$0.93$), so at the default 0.5 threshold the verifier
essentially never fires held out (recall as low as 0.03 for diffraction), and a much lower
threshold (0.05 in five of six held-out splits) recovers usable recall at a real precision
cost (for example, held-out diffraction $F_1$ rises from 0.057 at threshold 0.5 to 0.628 at
threshold 0.05). No off-the-shelf process-reward model (PRM) comparison is reported: we
checked directly and found no general-purpose or physics/materials-science PRM publicly
available; every widely used open PRM (for example Qwen2.5-Math-PRM-7B, PRM800K-tuned
variants, MedPRM) is trained on a specific domain's own step-labeled data (mathematics,
medicine), so applying one to these physics reasoning chains would be an uninterpretable
domain mismatch rather than a meaningful comparator, and we report this as a genuine
absence rather than an oversight.

\noindent\emph{Self-verification controls.} We isolate CGSV's contribution with three
controls, matched to CGSV's items, models, $K$, and decoding: (i) unaided
self-refinement, in which the model is told only to double-check its answer with no oracle
feedback, shows no significant change (PCR $0.349\to0.339$, $n=542$, $b=34$, $c=29$,
$p=0.615$); (ii) generic equal-compute re-prompting, in which the model is simply asked to
reconsider with no critique framing, likewise shows no significant change
($0.358\to0.342$, $n=562$, $b=24$, $c=15$, $p=0.20$); because both controls revise every item, their
nonzero regression counts are not comparable with CGSV's, which cannot regress by design. (iii) Best-of-3 independent resampling (no multi-turn revision) gives a
random-pick PCR of $0.363$ and an oracle-oracle best-of-3 ceiling of $0.419$ ($n=614$),
both below CGSV's after-PCR for every model, by as little as $0.02$ for the weakest model
and on a separately drawn item sample. We did not build a verifier-as-selector
comparison on top of best-of-$N$ (this would need each candidate's chain text persisted for
verifier scoring, which the current best-of-$N$ harness does not retain), so the
best-of-3 result establishes only that resampling alone leaves real headroom over a random
pick, not that the verifier can currently capture it; this is reported as a partial result. Together,
these controls show that CGSV's net gain in PCR is not reproduced by generic self-reflection or by
additional sampling at equal compute; whether the gain reflects better physics is addressed separately
by the paired correctness analysis (Table~\ref{tab:cgsvacc}).

\noindent\emph{Ablations, human agreement, and error analysis.} (i) Tolerance sensitivity:
we re-checked already-collected real model claims against the oracle at multiple tolerance
settings without new model calls. Diffraction's model ranking is \emph{not} robust to
tolerance: at $\mathrm{tol}=0.25^\circ$ the local model's PCR exceeds the API models'
pooled PCR, but the ranking reverses by $\mathrm{tol}=1.0^\circ$ (pairwise rank agreement
$0.0$ against the default $0.5^\circ$ across the tolerances tested); spectrum's ranking is
fully robust across $x_{\rm tol}\in\{10,20,30\}\,\mathrm{cm}^{-1}$ and prominence fraction
$\in\{0.03,0.05,0.08\}$ (rank agreement $1.0$); stability's ranking is fully robust across the
$\{0,0.025,0.05,0.10\}\,\mathrm{eV\,atom^{-1}}$ abstention-band widths stated in Methods (rank agreement
$1.0$; $0.67$ under the original rule). We flag the diffraction result as a real robustness concern
for any model comparison drawn from that single column. We additionally ran an independent
prompt/format sensitivity check: for diffraction, spectra, and stability, we wrote a second,
independently worded prompt template per constraint type (not a minimal edit of the
original) and re-collected real model responses from four of the paper's API models
(Gemini~2.5 Flash Lite, GPT-4o, Claude Sonnet~4.5, GPT-5.6-Terra) under both templates, then
measured the same pairwise rank agreement metric used for the tolerance check above.
Diffraction's ranking is again fragile under this different kind of perturbation (rank
agreement $0.5$, echoing its tolerance-sensitivity fragility rather than resolving it);
spectra's ranking is fairly robust to rewording (rank agreement $0.833$); stability's ranking
is more mixed (rank agreement $0.667$, computed under the original stability rule because the per-item
claims of this check were not retained), though its absolute rates stay tightly clustered
under both wordings ($0.667$--$0.933$ across all four models and both templates), so this
number likely reflects several close scores changing relative order rather than a dramatic
swing in any one model's judgment. Independent of the exact ranking, GPT-5.6-Terra is the
clear standout on spectra under both wordings ($0.867$/$0.933$) while GPT-4o is consistently
the weakest under both ($0.000$/$0.067$). GPT-4o's low coverage on spectra (Results) appears under the
original wording (7 of 15 responses checkable) but not under the paraphrase (15 of 15), so which
responses are scored depends on prompt wording; this is why we report coverage and two alternative
treatments of non-checkable responses alongside PCR (Results, Supplementary Table~S2). (ii) Human agreement: each precision audit was performed by one human
annotator; Table~\ref{tab:exp1} reports Wilson 95\% intervals for every precision and the
sensitivity to cases marked unsure. A second annotator, blind to the first annotator's labels,
re-labelled a random subsample of 30 audited cases per family, stratified by the first annotator's
verdict with false positives over-sampled; the second annotator's precision was weighted back to the
audited population by stratum, and the two annotators adjudicated every disagreement to a consensus label,
whose stratum-weighted precision Table~\ref{tab:exp1} reports beside the single-annotator value. Agreement,
Cohen's $\kappa$ and all precisions are in the Results and Supplementary Table~S8, and the oracle-silent audit used the same guidelines, released with
the code. (iii) Error taxonomy: from every
human-confirmed real violation in the six audits (78 diffraction, 93 microscopy, 457 spectra,
102 stability, 108 metallicity and 106 magnetism cases; Table~\ref{tab:exp1}; the stability count is from the
original audit and includes 45 near-hull cases that the final stability rule does not score), an open-ended per-item categorization followed by a semantic clustering pass
gives a small, legible taxonomy: diffraction failures are dominated by gridline/scale
misreading (70.5\%) over peak misidentification (14.1\%); microscopy failures are
dominated by scale-bar unit and unit-conversion errors (37.6\% and 23.7\%) over scale-factor
arithmetic errors (14.0\%; Fig.~\ref{fig:oracle-examples}c shows three such real cases); spectrum
failures are dominated by identifying the wrong peak (88.6\%) rather than misreading the axis (5.7\%)
or mismeasuring the right peak (5.3\%); and the DFT failures come from composition-level
heuristics: stability failures most often rest on reasoning without supporting data or on analogy to
related compounds (25.5\% and 21.6\% of the 102 real violations of the original audit, which include the
near-hull cases), metallicity failures on overgeneralizing from a known analogue
(25.0\%) or assuming a gap from bond character (10.2\%), and magnetism failures on misassigned
oxidation states (60.4\%; Supplementary Table~S5). The categories are produced by an automatic two-pass labelling with a
language model and are indicative, not a validated coding scheme.

(iv) Contamination: we checked whether any model chain ever
states source-specific metadata it was never shown (institution codes, database sample
IDs, present only in each dataset's own sidecar metadata, never in the rendered plot) --
zero such cases across 2{,}152 checked records, a real if narrow negative signal. We
compared each dataset's public release date against each model's documented training
cutoff where both are known: opXRD's Zenodo record (15298026) was published 2025-04-28,
which postdates GPT-4o's documented cutoff (2023-10, with some sources reporting a later
2024-06 extension), ruling out direct memorization of this specific corpus for that model;
most other model--dataset pairs cannot be so ruled out because frontier providers do not
publish per-snapshot cutoff dates. A local-model perplexity comparison (real opXRD text
mean log-probability $-0.18$ versus procedurally generated control text in the same surface
format at $-1.91$) is reported but explicitly flagged as weak and non-diagnostic: lower
perplexity on real structured data than on random control text is fully explained by real
data being more predictable than gibberish, independent of memorization. Per-family,
per-modality item counts with provenance: 205 (diffraction, opXRD), 481 (microscopy,
NFFA-EUROPE), 575 (spectra, RRUFF), 395 (stability), 255 (metallicity), and 241 (magnetism)
real oracle-checked records make up the 2{,}152-record verifier dataset described above;
they are not the 4{,}334 records of the main evaluation in Supplementary Table~S1, which
pools three seeds over nine models.

\noindent\emph{Statistics.} Rates are reported with percentile bootstrap 95\% intervals. Model
comparisons are tested at the item level, never by interval overlap. On the image constraint types,
where all models answer the same items, we use the exact McNemar test on the items both models
answered checkably; on the DFT constraint types, where models answer different materials, we use
Fisher's exact test; for overall PCR, which pools both, we use a bootstrap of the difference that
resamples paired items jointly and DFT items within model ($B=10{,}000$), with a two-sided bootstrap
$p$-value. Effect sizes are reported as differences in PCR with 95\% intervals, and the minimum
detectable difference as $(z_{0.975}+z_{0.80})$ times the bootstrap standard error of the
difference. Holm correction is applied within each family of comparisons: the comparisons stated in
the text (12 tests), all 36 model pairs on overall PCR, and all 36 pairs within each constraint type
(Supplementary Table~S3). The family of CGSV tests, Table~\ref{tab:cgsv}'s six exact McNemar tests (four per-model CGSV comparisons plus the two
mechanism-isolation controls), is corrected the same way. All four CGSV effects survive
Holm correction (adjusted $p<10^{-3}$ for every model), and
both controls remain non-significant under Holm or Benjamini--Hochberg correction (Holm-adjusted
$p=0.615$ and $0.399$). Paired correctness before and after CGSV uses the exact McNemar test on
the per-item accuracy indicators (Table~\ref{tab:cgsvacc}).

\FloatBarrier

\begin{availability}
The benchmark is derived entirely from public datasets: experimental powder diffraction
patterns,\cite{opxrd2025} scanning electron micrographs,\cite{aversa2018nffa} Raman
spectra,\cite{lafuente2015rruff} and first-principles reference values.\cite{matproj2013} The assembled MatPCR items and the oracle
reference values with their provenance are available at \url{https://zenodo.org/records/21829644}.

All code, including the oracle implementations with unit tests, the chain-parsing
code, the metric code, the DFT protocol settings, the prompts, the analysis scripts for every
reported number, and the released
verifier weights, is available at \url{https://github.com/KurbanIntelligenceLab/matpcr}; the
13 retrained verifier adapters are available at \url{https://huggingface.co/rasulkhanbayov/matpcr-verifier}.
\end{availability}

\begin{contributions}
H.K.\ and M.K.\ conceived the study. H.K., R.K., and M.K.\ designed the methodology.
R.K.\ and H.K.\ implemented the physical oracles, the metric, and the verifier. M.K.\
designed the density functional theory protocol and the physics underlying the oracles
and is responsible for their validation. All authors designed the experimental protocol,
interpreted the analyses, and wrote and approved the manuscript. H.K., Hasan Kurban; R.K., Rasul
Khanbayov; M.K., Mustafa Kurban.
\end{contributions}

\begin{conflicts}
The authors declare no competing financial interest.
In line with the policy that large language models do not satisfy
authorship criteria, we disclose that an AI assistant was used during the preparation of
this manuscript. All content was reviewed and verified by the human authors, who take
full responsibility for the work. No AI system is an author.
\end{conflicts}

\begin{funding}
The authors received no specific funding for this work.
\end{funding}

\begin{acknowledgments}
Computations were run on an NVIDIA A100 (80GB) workstation and on QCRI's Panther cluster (A100 80GB and H200 141GB partitions); analyses added during revision were run on an NVIDIA P100 (16GB) workstation and on CPU. Exact total GPU-hours are not recoverable, because most result logs did not record per-item timing, so we do not report an aggregate figure we cannot verify. We thank the institutions maintaining the public datasets this work depends on (opXRD, NFFA-EUROPE, RRUFF, Materials Project).
\end{acknowledgments}

\bibliography{matpcr_refs}

\clearpage
\setcounter{table}{0}
\setcounter{figure}{0}
\renewcommand{\thetable}{S\arabic{table}}
\renewcommand{\thefigure}{S\arabic{figure}}
\setcounter{equation}{0}
\renewcommand{\theequation}{S\arabic{equation}}

\section*{Supplementary Information}

\noindent
This Supplementary Information contains Supplementary Note 1, which provides
the standard derivations underlying the physical oracles and statistical
estimators used in the Article; Supplementary Note 2, which records how the oracle
definitions and the verifier split reached the final form used for every reported result;
Supplementary Figure S1 (full precision-audit breakdown); and Supplementary Tables S1--S8: per-cell
PCR with bootstrap intervals (S1), coverage and alternative treatments of non-checkable responses
(S2), paired multiplicity-corrected model comparisons (S3), class-conditional DFT rates (S4), the
violation taxonomy for the spectrum and DFT families (S5), the common class-balanced DFT item
set (S6), stability results under the original and final stability rules (S7), and the second-annotator and
oracle-silent audits (S8). Proposition~1 and Corollary~1, which constitute the main
measurement-theoretic result of the study, are stated and proved in the main
Article.

\vspace{0.5em}

\section*{Supplementary Note 1: standard derivations}
\noindent The results collected here are standard and are included for completeness.
\setcounter{subsection}{0}
\renewcommand{\thesubsection}{1.\arabic{subsection}}

\subsection{Bragg's law and \texorpdfstring{$d$}{d}-spacing}
For constructive interference from lattice planes of spacing $d$ at Bragg angle
$\theta$ with a beam of wavelength $\lambda$, the path difference $2d\sin\theta$ equals an
integer number of wavelengths; for first order ($n=1$), $\lambda=2d\sin\theta$. The
measured scattering angle is $2\theta$, so a claimed peak at $2\theta$ corresponds to
$d=\lambda/\!\big(2\sin\theta\big)$ with $\theta$ equal to half the measured angle.
Physical realizability requires $\sin\theta\le1$, equivalently $d\ge\lambda/2$, and the
instrument range further restricts $2\theta$ to a finite window; the diffraction oracle
enforces both and matches claimed peaks to the reference pattern of the stated phase
within tolerance.

\subsection{Scale-bar calibration}
A scale bar of stated length $\ell_{\mathrm{nm}}$ spanning $p_{\mathrm{bar}}$ pixels
fixes the calibration $s=\ell_{\mathrm{nm}}/p_{\mathrm{bar}}$ (nanometres per pixel).
For an image whose long edge is $p_{\mathrm{img}}$ pixels, the field of view is
$s\,p_{\mathrm{img}}$ nanometres, which upper-bounds any reported feature size and
defines the microscopy oracle.

\subsection{Convex-hull stability}
For a given composition, the energy above the convex hull $E_{\mathrm{hull}}\ge0$ is the
formation energy minus that of the lowest-energy combination of stable phases at the
same composition. A phase is thermodynamically stable when $E_{\mathrm{hull}}=0$ and is
treated as within tolerance when $E_{\mathrm{hull}}\le\epsilon$; the DFT-grounded near-hull oracle uses the reported threshold
$\epsilon=0.05\,\mathrm{eV\,atom^{-1}}$.

\subsection{Percentile bootstrap confidence intervals}
For a rate estimated from $n$ items we resample items with replacement $B$ times,
recompute the rate on each resample, and report the empirical $\alpha/2$ and
$1-\alpha/2$ quantiles. Under independent sampling of items the interval is consistent
as $n,B\to\infty$. Where items share a parent source (for example, several questions
about one structure), we resample at the source level to respect that dependence. We
report the interval for all rates.

\subsection{Concentration of the empirical rate}
For a rate $\bar X=\frac1n\sum_{i=1}^n X_i$ that averages independent variables
$X_i\in[0,1]$ with mean $\mathbb{E}[\bar X]=\mu$, Hoeffding's inequality gives
$\Pr[\,|\bar X-\mu|>t\,]\le 2e^{-2nt^2}$ for every $t>0$. Applied to $X_i=1-Y_i$ this
bounds the sampling error of $\mathrm{PCR}^\star$ about the population physical-consistency rate; where items share a parent source, $n$ is the number of independent sources, matching the source-level percentile bootstrap above.

\subsection{Exact McNemar test}
For paired before/after consistency, let $b$ and $c$ be the numbers of discordant items
(consistent-then-inconsistent, i.e.\ regressions, and inconsistent-then-consistent,
i.e.\ fixes). Under the null that each discordant item changes direction with
probability $\tfrac12$, $c\sim\mathrm{Binomial}(b+c,\tfrac12)$, and the two-sided exact
$p$-value is $\min\!\big(1,\,2\sum_{k=0}^{\min(b,c)}\binom{b+c}{k}2^{-(b+c)}\big)$. This
is the test used for CGSV in the main Article.
\subsection{Distribution-free bound on the total error of a reported rate}
The main Article's Proposition~1 and Corollary~1 are exact statements about the oracle's
bias on a fixed set of items. If, in addition, the checkable items are drawn independently
from the benchmark's item-generating process, with
$\mu=\mathbb{E}[1-Y_i]$ that process's physical-consistency rate, then splitting
$\mathrm{PCR}-\mu=(\mathrm{PCR}-\mathrm{PCR}^\star)+(\mathrm{PCR}^\star-\mu)$, bounding the
first term by $(1-\hat d)+\hat f$ (Corollary~1) and the second by the concentration
inequality above gives, for every $t>0$, with probability at least $1-2e^{-2nt^2}$,
\[
|\mathrm{PCR}-\mu|\;\le\;\underbrace{(1-\hat d)+\hat f}_{\text{oracle bias}}\;+\;\underbrace{t}_{\text{sampling}}.
\]
Where items share a parent source the same statement holds with $n$ replaced by the number
$m$ of independent sources under balanced source sizes; unequal sizes give a size-weighted
mean and a correspondingly weaker rate. Both steps are standard, and the independence
assumption is an idealization for a convenience benchmark. At the cell sizes used in the
Article the sampling term is far looser than the percentile bootstrap interval: at $n=90$ and
$95\%$ confidence it is $t=0.143$, against reported per-cell intervals of roughly $\pm0.09$.
The Article therefore reports bootstrap intervals as its uncertainty and uses this bound only
to show that the two sources of error are separately controlled.

\clearpage
\section*{Supplementary Figure S1}
\noindent\textbf{Oracle precision audit, every family, full breakdown.} Referenced from
the main Article's Results (Oracle validity subsection). Each bar is every
human-reviewed oracle-fired case for that family (Article, Table~1), split into real
violations, false positives, and unsure verdicts, as judged by the first annotator. Single-annotator precision
is 77.2--95.3\% (the stability bar shows the original audit of 107 cases, of which the 57 scored under the final
stability rule are all genuine violations, Article Table~1) but not uniform: diffraction's tolerance is visibly the tightest relative
to typical model precision, while the three DFT-grounded families cluster above 92\%. After a second annotation and
adjudication, consensus precision is lower for diffraction and spectra (0.60 and 0.61) and essentially unchanged for
the other families (Article Table~1; Supplementary Table~S8).
\begin{figure}[htbp]
\centering
\includegraphics[width=0.92\linewidth]{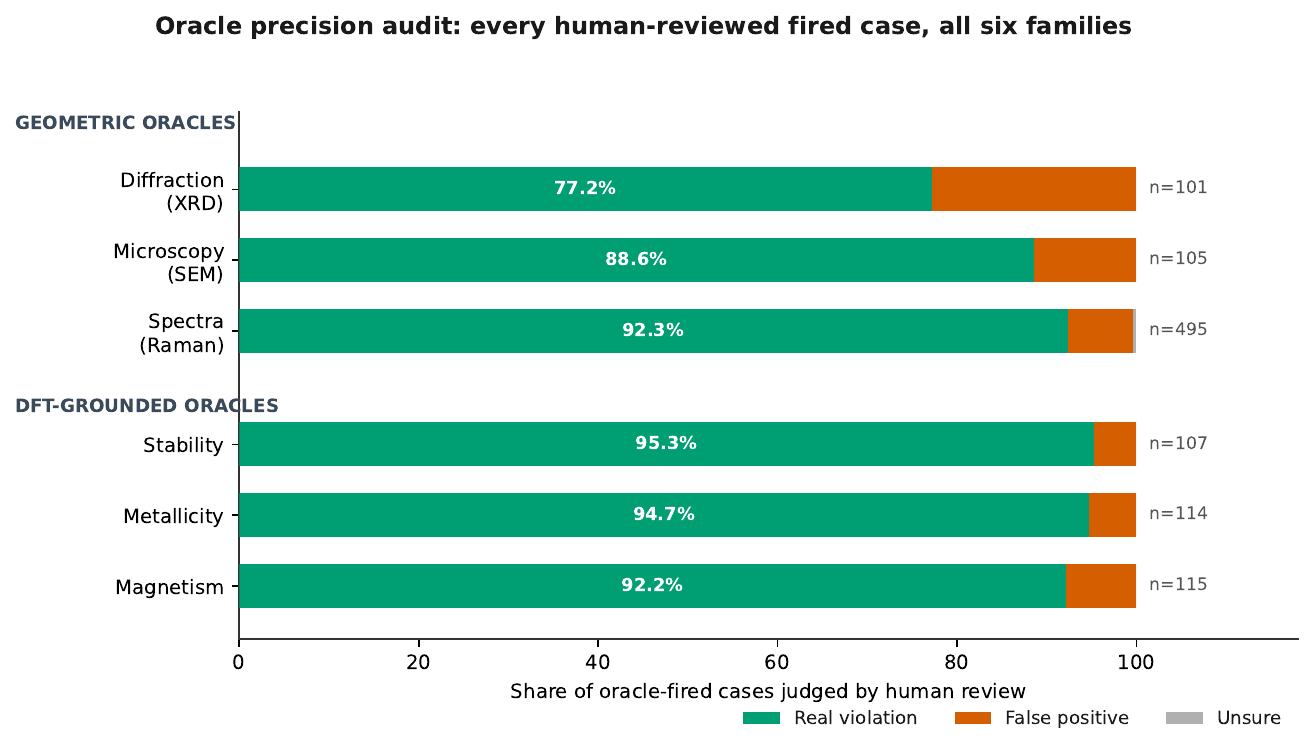}
\label{fig:precision-funnel}
\end{figure}

\clearpage
\section*{Supplementary Table S1}
\noindent\textbf{Per-cell Physical-Consistency Rate with bootstrap 95\% confidence
intervals.} Referenced from the main evaluation table of the Article. $n$ is the number of checkable records in the cell
(pooled across $S=3$ seeds); CI is the percentile bootstrap 95\% interval ($B=10{,}000$
resamples) on the cell's own PCR. The first four models are the originally evaluated set;
the next four are the later-generation refresh; the last (GPT-5.6-Sol) is a flagship-tier
addition testing whether it exceeds the balanced-tier GPT-5.6-Terra within the same
generation (Article, Table~1 note). Each item in this benchmark exposes
exactly one checkable constraint of its own constraint type, so the per-constraint-type
violation rate $\rho_t=1-\mathrm{PCR}$ for every cell below and is not tabulated separately. Stability rows use
the final stability oracle, which does not score near-hull materials; the original rule's values are in
Table~S7.
\begin{table}[htbp]
\centering
{\scriptsize\setlength{\tabcolsep}{4pt}\renewcommand{\arraystretch}{1.05}
\begin{tabular}{l l S[table-format=3.0] S[table-format=1.3] c}
\toprule
Model & Constraint type & {$n$} & {PCR} & {95\% CI} \\
\midrule
Qwen2.5-VL-7B & Diffraction & 64 & 0.500 & [0.375, 0.625] \\
Qwen2.5-VL-7B & Microscopy & 87 & 0.667 & [0.563, 0.759] \\
Qwen2.5-VL-7B & Spectra & 88 & 0.148 & [0.080, 0.227] \\
Qwen2.5-VL-7B & Stability$^{*}$ & 72 & 0.917 & [0.847, 0.972] \\
Qwen2.5-VL-7B & Metallicity$^{*}$ & 87 & 0.540 & [0.437, 0.644] \\
Qwen2.5-VL-7B & Magnetism$^{*}$ & 90 & 0.767 & [0.678, 0.856] \\
\addlinespace
Gemini 2.5 Flash Lite & Diffraction & 60 & 0.517 & [0.400, 0.650] \\
Gemini 2.5 Flash Lite & Microscopy & 86 & 1.000 & [1.000, 1.000] \\
Gemini 2.5 Flash Lite & Spectra & 90 & 0.278 & [0.189, 0.378] \\
Gemini 2.5 Flash Lite & Stability$^{*}$ & 70 & 0.814 & [0.729, 0.900] \\
Gemini 2.5 Flash Lite & Metallicity$^{*}$ & 84 & 0.702 & [0.607, 0.798] \\
Gemini 2.5 Flash Lite & Magnetism$^{*}$ & 89 & 0.494 & [0.382, 0.596] \\
\addlinespace
GPT-4o & Diffraction & 64 & 0.625 & [0.500, 0.750] \\
GPT-4o & Microscopy & 87 & 0.954 & [0.908, 0.989] \\
GPT-4o & Spectra & 44 & 0.068 & [0.000, 0.159] \\
GPT-4o & Stability$^{*}$ & 79 & 0.886 & [0.810, 0.949] \\
GPT-4o & Metallicity$^{*}$ & 85 & 0.671 & [0.565, 0.765] \\
GPT-4o & Magnetism$^{*}$ & 90 & 0.678 & [0.578, 0.778] \\
\addlinespace
Claude Sonnet 4.5 & Diffraction & 64 & 0.547 & [0.422, 0.672] \\
Claude Sonnet 4.5 & Microscopy & 87 & 0.977 & [0.943, 1.000] \\
Claude Sonnet 4.5 & Spectra & 90 & 0.189 & [0.111, 0.267] \\
Claude Sonnet 4.5 & Stability$^{*}$ & 73 & 0.849 & [0.767, 0.931] \\
Claude Sonnet 4.5 & Metallicity$^{*}$ & 85 & 0.694 & [0.600, 0.788] \\
Claude Sonnet 4.5 & Magnetism$^{*}$ & 90 & 0.756 & [0.666, 0.844] \\
\addlinespace
Claude Opus 4.6 & Diffraction & 64 & 0.578 & [0.453, 0.703] \\
Claude Opus 4.6 & Microscopy & 87 & 0.954 & [0.908, 0.989] \\
Claude Opus 4.6 & Spectra & 90 & 0.244 & [0.156, 0.333] \\
Claude Opus 4.6 & Stability$^{*}$ & 76 & 0.789 & [0.697, 0.882] \\
Claude Opus 4.6 & Metallicity$^{*}$ & 86 & 0.721 & [0.628, 0.814] \\
Claude Opus 4.6 & Magnetism$^{*}$ & 90 & 0.700 & [0.600, 0.789] \\
\addlinespace
GPT-5.6-Terra & Diffraction & 64 & 0.797 & [0.703, 0.891] \\
GPT-5.6-Terra & Microscopy & 81 & 0.951 & [0.901, 0.988] \\
GPT-5.6-Terra & Spectra & 90 & 0.833 & [0.756, 0.911] \\
GPT-5.6-Terra & Stability$^{*}$ & 75 & 0.653 & [0.547, 0.760] \\
GPT-5.6-Terra & Metallicity$^{*}$ & 85 & 0.659 & [0.553, 0.753] \\
GPT-5.6-Terra & Magnetism$^{*}$ & 90 & 0.767 & [0.678, 0.844] \\
\addlinespace
Qwen3-VL-8B-Instruct & Diffraction & 62 & 0.468 & [0.355, 0.597] \\
Qwen3-VL-8B-Instruct & Microscopy & 86 & 0.802 & [0.709, 0.884] \\
Qwen3-VL-8B-Instruct & Spectra & 90 & 0.133 & [0.067, 0.200] \\
Qwen3-VL-8B-Instruct & Stability$^{*}$ & 75 & 0.747 & [0.653, 0.840] \\
Qwen3-VL-8B-Instruct & Metallicity$^{*}$ & 77 & 0.688 & [0.584, 0.792] \\
Qwen3-VL-8B-Instruct & Magnetism$^{*}$ & 90 & 0.533 & [0.433, 0.633] \\
\addlinespace
Qwen3-VL-8B-Thinking & Diffraction & 64 & 0.578 & [0.453, 0.703] \\
Qwen3-VL-8B-Thinking & Microscopy & 87 & 0.885 & [0.816, 0.943] \\
Qwen3-VL-8B-Thinking & Spectra & 87 & 0.092 & [0.034, 0.161] \\
Qwen3-VL-8B-Thinking & Stability$^{*}$ & 76 & 0.895 & [0.816, 0.961] \\
Qwen3-VL-8B-Thinking & Metallicity$^{*}$ & 87 & 0.644 & [0.540, 0.736] \\
Qwen3-VL-8B-Thinking & Magnetism$^{*}$ & 89 & 0.506 & [0.404, 0.607] \\
\addlinespace
GPT-5.6-Sol & Diffraction & 61 & 0.820 & [0.721, 0.918] \\
GPT-5.6-Sol & Microscopy & 82 & 0.915 & [0.854, 0.976] \\
GPT-5.6-Sol & Spectra & 89 & 0.764 & [0.674, 0.854] \\
GPT-5.6-Sol & Stability$^{*}$ & 71 & 0.789 & [0.690, 0.887] \\
GPT-5.6-Sol & Metallicity$^{*}$ & 88 & 0.670 & [0.580, 0.761] \\
GPT-5.6-Sol & Magnetism$^{*}$ & 90 & 0.733 & [0.644, 0.822] \\
\bottomrule
\end{tabular}}
\end{table}

\clearpage
\section*{Supplementary Note 2: development history of the oracles and the verifier split}
Every result in the Article uses the final oracle definitions and the source-disjoint verifier
split described in its Methods. This note records the earlier versions, so that numbers in
earlier drafts or in the public repository history can be reconciled with the Article.

\paragraph{Spectrum oracle.} The first spectrum oracle accepted a claimed peak if a windowed maximum
of sufficient prominence lay near it. A first precision audit of that version (72.8\%, $n=103$)
showed that the criterion did not match human judgement of a peak on noisy, baseline-drifting
Raman spectra, because a monotonic rise or a shoulder could satisfy it. The oracle was revised to
require a genuine interior local maximum (Methods). The collected Half-B responses were re-scored
with the revised oracle and the resulting 495 fired cases re-audited (92.7\%; Article, Table~1).
The main evaluation, the accuracy-versus-consistency analysis and the CGSV experiment were re-run
under the revised oracle. Spectra PCR fell for every model, the mean accuracy--consistency
disagreement changed from 41.8\% to 22.2\% (23.6\% on the diffraction and spectra cells under the final
stability rule, on which the stability arm shows no disagreement), and the CGSV experiment changed from 555 to 545
checkable records.

\paragraph{Magnetism oracle.} The first magnetization check compared the computed reference with zero
using exact floating-point equality, which misfired on 82 of 3{,}000 materials whose computed
magnetization was near but not exactly zero. It was first patched with a $10^{-6}\,\mu_{\mathrm{B}}$
tolerance and then, following the domain co-author's review of the DFT protocol, replaced by the
per-atom classification at $0.05\,\mu_{\mathrm{B}}\,\mathrm{atom}^{-1}$ that never compares sign. The same
review introduced the metallicity abstention band and the near-hull wording for stability. The DFT
Half-B responses were re-collected under the final definitions before the audit reported in Table~1.

\paragraph{Stability abstention band.} The original stability oracle required a STABLE claim whenever
$E_{\mathrm{hull}}\le0.05$~eV\,atom$^{-1}$, while the stability prompt tells the model to treat any borderline or
metastable case as UNSTABLE. A model following the prompt on a near-hull material was therefore counted as a
violation. This conflict was found during revision, when a second, independent annotator re-labelled a subsample of
the stability audit: all 13 disagreements with the first annotator were UNSTABLE claims on near-hull materials.
The final oracle does not score $0<E_{\mathrm{hull}}\le0.05$ (abstention band, as for metallicity). The change affects
89 flagged records in the main evaluation and 50 of the 107 audited stability cases; all stability-dependent numbers
in the Article were recomputed from the collected responses, without new model calls, and the original rule's
values are given in Table~S7. The verifier was trained on labels from the original rule; evaluated on the test items
the final rule scores, its in-distribution AUC is 0.871 (stability family 0.911) and its held-out stability AUC 0.372.
The prompt-paraphrase sensitivity check for stability (Article, Methods) did not retain per-item claims and is
reported under the original rule. The stability row of Table~S5 was computed on the original audit's 102 real
violations, 45 of which are near-hull.

\paragraph{Materials Project retrievals.} The precision-audit sample used a retrieval of database
version 2026.04.13 on 2026-07-16; the reported evaluation used a retrieval of the same version on
2026-08-05. All 210 audited materials appear with the same identifiers in the later batch.

\paragraph{Verifier split.} The first in-distribution split was drawn at the record level, so records
from the same source item (seed replicates, different models' claims) could fall on both sides. It
gave an in-distribution AUC of 0.858 and a held-out magnetism AUC of 0.867, which earlier drafts of
this work reported as an exception to the lack of transfer. Under the source-disjoint split used in the
Article, the in-distribution AUC is 0.850 and held-out magnetism is 0.503, at chance; the exception
was an artifact of the leak and is withdrawn. The verifier dataset changed from 2{,}416 to 2{,}152
records when the spectrum records were re-scored with the revised oracle.

\paragraph{Violation taxonomy.} The diffraction and microscopy taxonomy (Article, Fig.~1c) was
computed before the spectrum and DFT audits had been redone under the final oracles; the taxonomy
for those four families (Table~S5) was computed on the final audits during revision.

\paragraph{Verifier weights.} The original training script did not save the adapter weights. The
script was corrected and all 13 adapters were retrained for release (seed 0; float32 on an NVIDIA
P100, because that GPU lacks bfloat16 support); the released adapters reproduce the in-distribution result of the Article's verifier table (Table~5)
exactly (AUC 0.850, 95\% CI [0.814, 0.884]; $F_1$ 0.737 vs.\ 0.735). Held-out AUCs move by up to 0.08:
diffraction 0.532 (reported 0.497), microscopy 0.602 (0.527; Bonferroni-corrected interval
[0.519, 0.682], above chance), spectra 0.536 (0.558), stability 0.383 (0.436; below chance), metallicity
0.542 (0.582) and magnetism 0.502 (0.503). The held-out results are thus sensitive to the training run,
while the conclusion that no held-out family transfers at a useful level holds for both runs.

\clearpage
\section*{Supplementary Table S2}
\noindent\textbf{Coverage of the Physical-Consistency Rate and alternative treatments of non-checkable
responses.} Each (model, constraint-type) cell attempts 90 items (30 per seed, $S=3$). Items the
oracle can score (evaluable) per cell: diffraction 64; microscopy 87; spectra 90; stability 71--79; metallicity 82--89; magnetism 90. Diffraction items without a usable reference
structure, metallicity items in the 0.05--0.20~eV abstention band and near-hull stability materials
($0<E_{\mathrm{hull}}\le0.05$~eV\,atom$^{-1}$) are not evaluable for any model.
For image constraint types an item counts as evaluable if at least one of the nine models produced a
checkable response on it (all models received the same items), which is a lower bound on true
evaluability. Coverage is checkable responses / evaluable items; it absorbs the rare API failures and
time-outs, which the result files do not record separately. Overall PCR (rank in parentheses) under
three conventions: \emph{excluded} (reported in the Article; non-checkable responses are not scored),
\emph{penalised} (a non-checkable response counts as a violation) and \emph{lenient} (it counts as
consistent). Kendall $\tau$ of the model ordering against the reported convention: penalised
0.83, lenient 1.00.
\begin{table}[htbp]
\centering
{\scriptsize\setlength{\tabcolsep}{3.2pt}\renewcommand{\arraystretch}{1.08}
\begin{tabular}{l cccccc ccc}
\toprule
& \multicolumn{6}{c}{Coverage} & \multicolumn{3}{c}{Overall PCR (rank)} \\
\cmidrule(lr){2-7}\cmidrule(lr){8-10}
Model & Diffr. & Micro. & Spectra & Stab. & Metal. & Magn. & excluded & penalised & lenient \\
\midrule
Qwen2.5-VL-7B & 1.00 & 1.00 & 0.98 & 1.00 & 1.00 & 1.00 & 0.584 (8) & 0.582 (8) & 0.586 (8) \\
Gemini 2.5 Flash Lite & 0.94 & 0.99 & 1.00 & 0.99 & 0.97 & 0.99 & 0.630 (6) & 0.618 (6) & 0.638 (6) \\
GPT-4o & 1.00 & 1.00 & 0.49 & 1.00 & 1.00 & 1.00 & 0.699 (3) & 0.634 (5) & 0.727 (3) \\
Claude Sonnet 4.5 & 1.00 & 1.00 & 1.00 & 1.00 & 1.00 & 1.00 & 0.667 (4) & 0.667 (3) & 0.667 (4) \\
Claude Opus 4.6 & 1.00 & 1.00 & 1.00 & 1.00 & 1.00 & 1.00 & 0.663 (5) & 0.663 (4) & 0.663 (5) \\
GPT-5.6-Terra & 1.00 & 0.93 & 1.00 & 1.00 & 1.00 & 1.00 & 0.777 (2) & 0.768 (1) & 0.780 (2) \\
GPT-5.6-Sol & 0.95 & 0.94 & 0.99 & 0.99 & 0.99 & 1.00 & 0.778 (1) & 0.760 (2) & 0.783 (1) \\
Qwen3-VL-8B-Instruct & 0.97 & 0.99 & 1.00 & 0.97 & 0.94 & 1.00 & 0.556 (9) & 0.545 (9) & 0.565 (9) \\
Qwen3-VL-8B-Thinking & 1.00 & 1.00 & 0.97 & 1.00 & 1.00 & 0.99 & 0.594 (7) & 0.589 (7) & 0.597 (7) \\
\bottomrule
\end{tabular}}
\end{table}

\clearpage
\section*{Supplementary Table S3}
\noindent\textbf{Paired, multiplicity-corrected model comparisons.} Differences are PCR of model A minus
model B. Image constraint types: all models answered the same items, so A and B are compared on the
items both answered checkably (exact McNemar test). DFT constraint types: models answered different
materials (Fisher's exact test). Overall: bootstrap of the difference that resamples paired items
jointly and DFT items within model ($B=10{,}000$; two-sided bootstrap $p$). 95\% CI: bootstrap.
MDE: smallest difference detectable with 80\% power at $\alpha=0.05$, $(z_{0.975}+z_{0.80})\times$
bootstrap SE. $p_{\mathrm{Holm}}$: Holm-corrected within the family shown. (a) The 12 comparisons
stated in the Article. (b) Within each constraint type, all 36 pairs: number significant after Holm
and median MDE. (c) All 36 pairs on overall PCR, ordered by $|\Delta|$.
\begin{table}[htbp]
\centering
{\scriptsize\setlength{\tabcolsep}{3pt}\renewcommand{\arraystretch}{1.05}
\adjustbox{max width=\linewidth}{%
\begin{tabular}{l l r c r l r r}
\toprule
Comparison (A vs B) & Scope & $\Delta$ & 95\% CI & MDE & Test & $p$ & $p_{\mathrm{Holm}}$ \\
\midrule
\multicolumn{8}{l}{\itshape (a) Comparisons stated in the Article (Holm over 12)}\\
GPT-5.6-Terra vs GPT-4o & overall & +0.072 & [+0.018, +0.125] & 0.076 & bootstrap & 0.009 & 0.045 \\
Claude Opus 4.6 vs Claude Sonnet 4.5 & overall & -0.003 & [-0.050, +0.044] & 0.067 & bootstrap & 0.890 & 1.000 \\
Qwen3-VL-8B-Thinking vs Qwen3-VL-8B-Instruct & overall & +0.034 & [-0.016, +0.083] & 0.071 & bootstrap & 0.182 & 0.545 \\
GPT-5.6-Sol vs GPT-5.6-Terra & overall & +0.002 & [-0.042, +0.046] & 0.063 & bootstrap & 0.967 & 1.000 \\
Gemini 2.5 Flash Lite vs Qwen2.5-VL-7B & Microscopy & +0.326 & [+0.233, +0.430] & 0.143 & McNemar & $7.5\times10^{-9}$ & $8.9\times10^{-8}$ \\
GPT-4o vs Qwen2.5-VL-7B & Microscopy & +0.287 & [+0.184, +0.391] & 0.145 & McNemar & $4.2\times10^{-7}$ & $4.2\times10^{-6}$ \\
Claude Sonnet 4.5 vs Qwen2.5-VL-7B & Microscopy & +0.310 & [+0.218, +0.414] & 0.138 & McNemar & $1.5\times10^{-8}$ & $1.6\times10^{-7}$ \\
Claude Opus 4.6 vs Qwen2.5-VL-7B & Microscopy & +0.287 & [+0.184, +0.402] & 0.158 & McNemar & $4.6\times10^{-6}$ & $3.7\times10^{-5}$ \\
GPT-5.6-Terra vs Qwen2.5-VL-7B & Microscopy & +0.309 & [+0.198, +0.420] & 0.160 & McNemar & $1.6\times10^{-6}$ & $1.5\times10^{-5}$ \\
GPT-5.6-Sol vs Qwen2.5-VL-7B & Microscopy & +0.244 & [+0.122, +0.366] & 0.169 & McNemar & $3.2\times10^{-4}$ & 0.002 \\
Qwen3-VL-8B-Instruct vs Qwen2.5-VL-7B & Microscopy & +0.140 & [+0.023, +0.256] & 0.162 & McNemar & 0.029 & 0.116 \\
Qwen3-VL-8B-Thinking vs Qwen2.5-VL-7B & Microscopy & +0.218 & [+0.103, +0.333] & 0.162 & McNemar & $5.5\times10^{-4}$ & 0.003 \\
\bottomrule
\end{tabular}%
}}

\vspace{0.8em}
{\scriptsize\setlength{\tabcolsep}{6pt}
\begin{tabular}{l r r}
\toprule
\multicolumn{3}{l}{\itshape (b) All 36 pairs within each constraint type (Holm over 36)}\\
Constraint type & significant pairs & median MDE \\
\midrule
Diffraction & 10 & 0.192 \\
Microscopy & 10 & 0.109 \\
Spectra & 15 & 0.149 \\
Stability & 3 & 0.175 \\
Metallicity & 0 & 0.202 \\
Magnetism & 9 & 0.194 \\
\bottomrule
\end{tabular}}
\end{table}
\begin{table}[htbp]
\centering
{\scriptsize\setlength{\tabcolsep}{3pt}\renewcommand{\arraystretch}{1.02}
\begin{tabular}{l l r c r l r r}
\toprule
\multicolumn{8}{l}{\itshape (c) All 36 pairs on overall PCR (Holm over 36): 18 significant}\\
Comparison (A vs B) & Scope & $\Delta$ & 95\% CI & MDE & Test & $p$ & $p_{\mathrm{Holm}}$ \\
\midrule
GPT-5.6-Terra vs Qwen3-VL-8B-Instruct & overall & +0.223 & [+0.169, +0.278] & 0.078 & bootstrap & $1.0\times10^{-4}$ & 0.004 \\
GPT-5.6-Sol vs Qwen3-VL-8B-Instruct & overall & +0.220 & [+0.165, +0.275] & 0.079 & bootstrap & $1.0\times10^{-4}$ & 0.004 \\
Qwen2.5-VL-7B vs GPT-5.6-Terra & overall & -0.198 & [-0.251, -0.144] & 0.076 & bootstrap & $1.0\times10^{-4}$ & 0.004 \\
GPT-5.6-Terra vs Qwen3-VL-8B-Thinking & overall & +0.189 & [+0.135, +0.243] & 0.077 & bootstrap & $1.0\times10^{-4}$ & 0.004 \\
Qwen2.5-VL-7B vs GPT-5.6-Sol & overall & -0.188 & [-0.242, -0.134] & 0.078 & bootstrap & $1.0\times10^{-4}$ & 0.004 \\
GPT-5.6-Sol vs Qwen3-VL-8B-Thinking & overall & +0.184 & [+0.129, +0.240] & 0.077 & bootstrap & $1.0\times10^{-4}$ & 0.004 \\
Gemini 2.5 Flash Lite vs GPT-5.6-Terra & overall & -0.151 & [-0.205, -0.098] & 0.076 & bootstrap & $1.0\times10^{-4}$ & 0.004 \\
Gemini 2.5 Flash Lite vs GPT-5.6-Sol & overall & -0.145 & [-0.200, -0.092] & 0.077 & bootstrap & $1.0\times10^{-4}$ & 0.004 \\
Claude Opus 4.6 vs GPT-5.6-Terra & overall & -0.118 & [-0.168, -0.067] & 0.073 & bootstrap & $1.0\times10^{-4}$ & 0.004 \\
Claude Opus 4.6 vs GPT-5.6-Sol & overall & -0.116 & [-0.168, -0.065] & 0.074 & bootstrap & $1.0\times10^{-4}$ & 0.004 \\
Claude Sonnet 4.5 vs GPT-5.6-Terra & overall & -0.115 & [-0.166, -0.063] & 0.074 & bootstrap & $1.0\times10^{-4}$ & 0.004 \\
Claude Sonnet 4.5 vs GPT-5.6-Sol & overall & -0.113 & [-0.165, -0.061] & 0.074 & bootstrap & $1.0\times10^{-4}$ & 0.004 \\
Claude Sonnet 4.5 vs Qwen3-VL-8B-Instruct & overall & +0.108 & [+0.059, +0.158] & 0.071 & bootstrap & $1.0\times10^{-4}$ & 0.004 \\
Claude Opus 4.6 vs Qwen3-VL-8B-Instruct & overall & +0.105 & [+0.052, +0.156] & 0.074 & bootstrap & $1.0\times10^{-4}$ & 0.004 \\
GPT-4o vs Qwen3-VL-8B-Instruct & overall & +0.098 & [+0.043, +0.153] & 0.077 & bootstrap & $2.0\times10^{-4}$ & 0.004 \\
Qwen2.5-VL-7B vs Claude Sonnet 4.5 & overall & -0.085 & [-0.137, -0.036] & 0.072 & bootstrap & 0.001 & 0.021 \\
Qwen2.5-VL-7B vs Claude Opus 4.6 & overall & -0.082 & [-0.133, -0.031] & 0.073 & bootstrap & 0.001 & 0.021 \\
GPT-4o vs GPT-5.6-Sol & overall & -0.081 & [-0.136, -0.029] & 0.077 & bootstrap & 0.004 & 0.065 \\
Claude Sonnet 4.5 vs Qwen3-VL-8B-Thinking & overall & +0.077 & [+0.028, +0.126] & 0.069 & bootstrap & 0.003 & 0.049 \\
GPT-4o vs GPT-5.6-Terra & overall & -0.072 & [-0.127, -0.018] & 0.077 & bootstrap & 0.011 & 0.165 \\
Claude Opus 4.6 vs Qwen3-VL-8B-Thinking & overall & +0.071 & [+0.022, +0.122] & 0.071 & bootstrap & 0.006 & 0.109 \\
Qwen2.5-VL-7B vs GPT-4o & overall & -0.071 & [-0.125, -0.017] & 0.078 & bootstrap & 0.013 & 0.176 \\
Gemini 2.5 Flash Lite vs Qwen3-VL-8B-Instruct & overall & +0.070 & [+0.018, +0.123] & 0.075 & bootstrap & 0.008 & 0.125 \\
GPT-4o vs Qwen3-VL-8B-Thinking & overall & +0.062 & [+0.010, +0.115] & 0.075 & bootstrap & 0.021 & 0.276 \\
Qwen2.5-VL-7B vs Gemini 2.5 Flash Lite & overall & -0.043 & [-0.095, +0.009] & 0.075 & bootstrap & 0.112 & 1.000 \\
Gemini 2.5 Flash Lite vs Claude Sonnet 4.5 & overall & -0.039 & [-0.087, +0.009] & 0.069 & bootstrap & 0.119 & 1.000 \\
Gemini 2.5 Flash Lite vs Qwen3-VL-8B-Thinking & overall & +0.039 & [-0.012, +0.089] & 0.072 & bootstrap & 0.140 & 1.000 \\
Gemini 2.5 Flash Lite vs Claude Opus 4.6 & overall & -0.036 & [-0.085, +0.014] & 0.072 & bootstrap & 0.161 & 1.000 \\
Qwen3-VL-8B-Instruct vs Qwen3-VL-8B-Thinking & overall & -0.034 & [-0.084, +0.016] & 0.072 & bootstrap & 0.184 & 1.000 \\
Gemini 2.5 Flash Lite vs GPT-4o & overall & -0.033 & [-0.085, +0.019] & 0.074 & bootstrap & 0.215 & 1.000 \\
Qwen2.5-VL-7B vs Qwen3-VL-8B-Instruct & overall & +0.025 & [-0.027, +0.075] & 0.073 & bootstrap & 0.333 & 1.000 \\
GPT-4o vs Claude Sonnet 4.5 & overall & -0.016 & [-0.065, +0.033] & 0.071 & bootstrap & 0.522 & 1.000 \\
Qwen2.5-VL-7B vs Qwen3-VL-8B-Thinking & overall & -0.009 & [-0.060, +0.043] & 0.073 & bootstrap & 0.735 & 1.000 \\
Claude Sonnet 4.5 vs Claude Opus 4.6 & overall & +0.003 & [-0.043, +0.048] & 0.066 & bootstrap & 0.898 & 1.000 \\
GPT-5.6-Terra vs GPT-5.6-Sol & overall & -0.002 & [-0.046, +0.042] & 0.064 & bootstrap & 0.946 & 1.000 \\
GPT-4o vs Claude Opus 4.6 & overall & +0.001 & [-0.050, +0.053] & 0.073 & bootstrap & 0.975 & 1.000 \\
\bottomrule
\end{tabular}}
\end{table}

\clearpage
\section*{Supplementary Table S4}
\noindent\textbf{Class-conditional and class-balanced PCR on the DFT constraint types.} Reference classes
use each oracle's own definitions (stability: on the hull vs.\ $E_{\mathrm{hull}}>0.05$~eV\,atom$^{-1}$, near-hull
materials not scored; metallicity:
computed gap $\le0.05$~eV vs.\ $\ge0.20$~eV, abstention band excluded; magnetism: $|M|$ per atom
$\ge0.05\,\mu_{\mathrm{B}}$). Pooled: PCR over all checkable responses (as in the Article's main table).
Balanced: mean of the two per-class rates. Main evaluation records (each model answered its own
materials).
\begin{table}[htbp]
\centering
{\scriptsize\setlength{\tabcolsep}{5pt}\renewcommand{\arraystretch}{1.03}
\begin{tabular}{l r r r r r r}
\toprule
Model & $n_{+}$ & PCR$_{+}$ & $n_{-}$ & PCR$_{-}$ & pooled & balanced \\
\midrule
\multicolumn{7}{l}{\itshape Stability: positive class = on hull, negative class = above hull}\\
Qwen2.5-VL-7B & 4 & 0.000 & 68 & 0.971 & 0.917 & 0.485 \\
Gemini 2.5 Flash Lite & 3 & 0.667 & 67 & 0.821 & 0.814 & 0.744 \\
GPT-4o & 9 & 0.222 & 70 & 0.971 & 0.886 & 0.597 \\
Claude Sonnet 4.5 & 6 & 0.333 & 67 & 0.896 & 0.849 & 0.614 \\
Claude Opus 4.6 & 7 & 0.571 & 69 & 0.812 & 0.789 & 0.692 \\
GPT-5.6-Terra & 5 & 1.000 & 70 & 0.629 & 0.653 & 0.814 \\
GPT-5.6-Sol & 4 & 1.000 & 67 & 0.776 & 0.789 & 0.888 \\
Qwen3-VL-8B-Instruct & 8 & 0.625 & 67 & 0.761 & 0.747 & 0.693 \\
Qwen3-VL-8B-Thinking & 7 & 0.143 & 69 & 0.971 & 0.895 & 0.557 \\
\addlinespace
\multicolumn{7}{l}{\itshape Metallicity: positive class = gap $\le$0.05 eV, negative class = gap $\ge$0.20 eV}\\
Qwen2.5-VL-7B & 38 & 0.289 & 49 & 0.735 & 0.540 & 0.512 \\
Gemini 2.5 Flash Lite & 30 & 0.367 & 54 & 0.889 & 0.702 & 0.628 \\
GPT-4o & 34 & 0.265 & 51 & 0.941 & 0.671 & 0.603 \\
Claude Sonnet 4.5 & 36 & 0.444 & 49 & 0.878 & 0.694 & 0.661 \\
Claude Opus 4.6 & 32 & 0.344 & 54 & 0.944 & 0.721 & 0.644 \\
GPT-5.6-Terra & 35 & 0.314 & 50 & 0.900 & 0.659 & 0.607 \\
GPT-5.6-Sol & 34 & 0.206 & 54 & 0.963 & 0.670 & 0.584 \\
Qwen3-VL-8B-Instruct & 34 & 0.412 & 43 & 0.907 & 0.688 & 0.659 \\
Qwen3-VL-8B-Thinking & 35 & 0.229 & 52 & 0.923 & 0.644 & 0.576 \\
\addlinespace
\multicolumn{7}{l}{\itshape Magnetism: positive class = magnetic, negative class = non-magnetic}\\
Qwen2.5-VL-7B & 57 & 0.947 & 33 & 0.455 & 0.767 & 0.701 \\
Gemini 2.5 Flash Lite & 54 & 0.537 & 35 & 0.429 & 0.494 & 0.483 \\
GPT-4o & 56 & 0.732 & 34 & 0.588 & 0.678 & 0.660 \\
Claude Sonnet 4.5 & 52 & 0.788 & 38 & 0.711 & 0.756 & 0.749 \\
Claude Opus 4.6 & 57 & 0.737 & 33 & 0.636 & 0.700 & 0.687 \\
GPT-5.6-Terra & 46 & 0.826 & 44 & 0.705 & 0.767 & 0.765 \\
GPT-5.6-Sol & 50 & 0.840 & 40 & 0.600 & 0.733 & 0.720 \\
Qwen3-VL-8B-Instruct & 57 & 0.526 & 33 & 0.545 & 0.533 & 0.536 \\
Qwen3-VL-8B-Thinking & 54 & 0.407 & 35 & 0.657 & 0.506 & 0.532 \\
\bottomrule
\end{tabular}}
\end{table}

\clearpage
\section*{Supplementary Table S5}
\noindent\textbf{Violation taxonomy for the spectrum and DFT-grounded oracle families.} Every
human-confirmed real violation in the final-oracle precision audits (Article, Table~1) was labelled
open-ended by Gemini 2.5 Flash Lite and the labels were merged by a second, semantic clustering pass
(the same procedure as the diffraction and microscopy taxonomy in the Article's Fig.~1c). Categories
are machine-generated names and are indicative only; the metallicity clustering returned fine-grained
categories, which are shown as returned.
\begin{table}[htbp]
\centering
{\scriptsize\setlength{\tabcolsep}{5pt}\renewcommand{\arraystretch}{1.05}
\begin{tabular}{l l r r}
\toprule
Family & Category & count & share \\
\midrule
Spectra ($n=457$) & Incorrect Peak Identification & 405 & 88.6\% \\
& Axis/Scale Misinterpretation & 26 & 5.7\% \\
& Incorrect Peak Measurement & 24 & 5.3\% \\
& other categories ($<$3\% each) & 2 & 0.4\% \\
\addlinespace
Stability ($n=102$) & Lack of Data/Evidence & 26 & 25.5\% \\
& Overgeneralization/Analogue Reasoning & 22 & 21.6\% \\
& Assumptions based on Complexity/Composition & 15 & 14.7\% \\
& Reliance on Intuition/General Knowledge & 11 & 10.8\% \\
& Ignoring Specific Chemical Details & 10 & 9.8\% \\
& Ignoring Known Phases/Structures & 8 & 7.8\% \\
& Assumptions about Rarity/Novelty & 5 & 4.9\% \\
& other categories ($<$3\% each) & 5 & 4.9\% \\
\addlinespace
Metallicity ($n=108$) & Overgeneralization & 27 & 25.0\% \\
& Incorrect Bond/Character Assumption & 11 & 10.2\% \\
& Overestimation of Electronegativity Effect & 7 & 6.5\% \\
& Overestimation of Transition Metal Contribution & 6 & 5.6\% \\
& Ignoring Electronic Structure & 6 & 5.6\% \\
& Ignoring d-orbital Contribution & 6 & 5.6\% \\
& Ignoring Magnetism/Electronic Structure & 4 & 3.7\% \\
& other categories ($<$3\% each) & 41 & 38.0\% \\
\addlinespace
Magnetism ($n=106$) & Oxidation State Errors & 64 & 60.4\% \\
& Moment/Magnetism Calculation Errors & 15 & 14.2\% \\
& Electronic Effects Errors & 6 & 5.7\% \\
& Element Type/Property Errors & 4 & 3.8\% \\
& Spin State Errors & 4 & 3.8\% \\
& Unpaired Electron Errors & 4 & 3.8\% \\
& other categories ($<$3\% each) & 9 & 8.5\% \\
\bottomrule
\end{tabular}}
\end{table}

\clearpage
\section*{Supplementary Table S6}
\noindent\textbf{Common, class-balanced DFT item set.} Every model answered the same 90 materials per DFT
constraint type (45 per reference class, drawn once with a fixed seed from the archived
3{,}000-material batch; the item list is released with the code). For stability, 34 of the 45 near-hull draws
fall in the band the final oracle does not score, leaving 11 on-hull and 45 above-hull materials. Coverage: checkable / 90.
PCR$_{+}$, PCR$_{-}$: per reference class (classes as in Table~S4). Qwen3-VL-8B-Thinking is absent because
its provider endpoint had been withdrawn when this set was run (2026-10-15); Qwen2.5-VL-7B was run on CPU in float32 for this set. Paired comparisons pooled over the three types (exact McNemar on shared materials, bootstrap 95\% CI, Holm): GPT-5.6-Terra $-$ GPT-4o: $+0.021$ [-0.042, +0.085], Holm $p=1.000$; Claude Opus 4.6 $-$ Claude Sonnet 4.5: $+0.021$ [-0.030, +0.072], Holm $p=1.000$; GPT-5.6-Sol $-$ GPT-5.6-Terra: $+0.030$ [-0.013, +0.072], Holm $p=0.743$. Across all 28 pairs, 3 differ after Holm correction.
\begin{table}[htbp]
\centering
{\scriptsize\setlength{\tabcolsep}{5pt}\renewcommand{\arraystretch}{1.03}
\begin{tabular}{l r r r r r r}
\toprule
Model & $n$ & coverage & PCR & PCR$_{+}$ & PCR$_{-}$ & balanced \\
\midrule
\multicolumn{7}{l}{\itshape Stability}\\
Qwen2.5-VL-7B & 56 & 1.00 & 0.804 & 0.000 & 1.000 & 0.500 \\
Gemini 2.5 Flash Lite & 56 & 1.00 & 0.625 & 0.273 & 0.711 & 0.492 \\
GPT-4o & 56 & 1.00 & 0.804 & 0.091 & 0.978 & 0.534 \\
Claude Sonnet 4.5 & 56 & 1.00 & 0.821 & 0.545 & 0.889 & 0.717 \\
Claude Opus 4.6 & 56 & 1.00 & 0.839 & 0.818 & 0.844 & 0.831 \\
GPT-5.6-Terra & 56 & 1.00 & 0.714 & 0.909 & 0.667 & 0.788 \\
GPT-5.6-Sol & 56 & 1.00 & 0.839 & 0.727 & 0.867 & 0.797 \\
Qwen3-VL-8B-Instruct & 56 & 1.00 & 0.696 & 0.182 & 0.822 & 0.502 \\
\addlinespace
\multicolumn{7}{l}{\itshape Metallicity}\\
Qwen2.5-VL-7B & 90 & 1.00 & 0.622 & 0.400 & 0.844 & 0.622 \\
Gemini 2.5 Flash Lite & 90 & 0.99 & 0.551 & 0.205 & 0.889 & 0.547 \\
GPT-4o & 90 & 1.00 & 0.567 & 0.267 & 0.867 & 0.567 \\
Claude Sonnet 4.5 & 90 & 1.00 & 0.611 & 0.356 & 0.867 & 0.611 \\
Claude Opus 4.6 & 90 & 1.00 & 0.633 & 0.400 & 0.867 & 0.633 \\
GPT-5.6-Terra & 90 & 1.00 & 0.633 & 0.333 & 0.933 & 0.633 \\
GPT-5.6-Sol & 90 & 1.00 & 0.656 & 0.356 & 0.956 & 0.656 \\
Qwen3-VL-8B-Instruct & 90 & 1.00 & 0.556 & 0.267 & 0.844 & 0.556 \\
\addlinespace
\multicolumn{7}{l}{\itshape Magnetism}\\
Qwen2.5-VL-7B & 90 & 1.00 & 0.644 & 0.911 & 0.378 & 0.644 \\
Gemini 2.5 Flash Lite & 90 & 1.00 & 0.644 & 0.600 & 0.689 & 0.644 \\
GPT-4o & 90 & 1.00 & 0.711 & 0.644 & 0.778 & 0.711 \\
Claude Sonnet 4.5 & 90 & 1.00 & 0.722 & 0.689 & 0.756 & 0.722 \\
Claude Opus 4.6 & 90 & 1.00 & 0.744 & 0.778 & 0.711 & 0.744 \\
GPT-5.6-Terra & 90 & 1.00 & 0.756 & 0.822 & 0.689 & 0.756 \\
GPT-5.6-Sol & 90 & 1.00 & 0.733 & 0.800 & 0.667 & 0.733 \\
Qwen3-VL-8B-Instruct & 90 & 0.99 & 0.685 & 0.682 & 0.689 & 0.685 \\
\bottomrule
\end{tabular}}
\end{table}

\clearpage
\section*{Supplementary Table S7}
\noindent\textbf{Stability results under the original and the final stability rule.} The original rule
required a STABLE claim whenever $E_{\mathrm{hull}}\le0.05$~eV\,atom$^{-1}$, although the stability prompt
instructs the model to treat borderline or metastable cases as UNSTABLE; a model following the prompt on a
near-hull material was therefore counted as a violation. The final rule does not score near-hull materials
($0<E_{\mathrm{hull}}\le0.05$), as for the metallicity abstention band. Main-evaluation records:
138 of 805 stability records are near-hull,
89 of which the original rule flagged. Overall: PCR over all six constraint
types. In the precision audit, 50 of the 107 audited stability cases are near-hull; the remaining
57 were all confirmed as genuine violations by the first annotator. Kendall $\tau$ of the overall model
ordering between the two rules: 0.94; of the stability ordering:
0.67.
\begin{table}[htbp]
\centering
{\scriptsize\setlength{\tabcolsep}{5pt}\renewcommand{\arraystretch}{1.05}
\begin{tabular}{l r r r r r}
\toprule
& \multicolumn{3}{c}{Stability PCR} & \multicolumn{2}{c}{Overall PCR} \\
\cmidrule(lr){2-4}\cmidrule(lr){5-6}
Model & original & final & change & original & final \\
\midrule
Qwen2.5-VL-7B & 0.744 & 0.917 & +0.172 & 0.565 & 0.584 \\
Gemini 2.5 Flash Lite & 0.730 & 0.814 & +0.084 & 0.622 & 0.630 \\
GPT-4o & 0.800 & 0.886 & +0.086 & 0.687 & 0.699 \\
Claude Sonnet 4.5 & 0.756 & 0.849 & +0.094 & 0.656 & 0.667 \\
Claude Opus 4.6 & 0.722 & 0.789 & +0.067 & 0.655 & 0.663 \\
GPT-5.6-Terra & 0.656 & 0.653 & -0.002 & 0.774 & 0.777 \\
GPT-5.6-Sol & 0.742 & 0.789 & +0.047 & 0.770 & 0.778 \\
Qwen3-VL-8B-Instruct & 0.713 & 0.747 & +0.034 & 0.555 & 0.556 \\
Qwen3-VL-8B-Thinking & 0.767 & 0.895 & +0.128 & 0.579 & 0.594 \\
\bottomrule
\end{tabular}}
\end{table}

\clearpage
\section*{Supplementary Table S8}
\noindent\textbf{Second annotator and oracle-silent audit.} Precision A: first annotator, all audited fired cases
(as in the Article's Table~1). Precision B: second annotator, who independently labelled a random subsample of 30
audited cases per family stratified by A's verdict (false positives over-sampled), weighted back to the audited
population by stratum; 95\% interval from per-stratum Jeffreys posteriors. $n$: subsampled cases scored under the
final oracle (for stability, near-hull cases are not scored). $\kappa$: Cohen's kappa on those cases (undefined when
both annotators gave every case the same label). Oracle-silent audit: 30 randomly sampled responses the oracle
judged consistent per family; missed = cases an annotator judged a genuine violation the oracle missed (Wilson 95\%
interval); implied recall scales the miss rate to the oracle-silent share of the audited population. Consensus:
the shared label where A and B agree, otherwise the label both annotators chose when adjudicating the case
together (adj.: number adjudicated; 24 in total, none in stability or metallicity); consensus precision is
weighted back by stratum like precision B. The consensus sided with B in every diffraction, spectrum and
magnetism dispute and with A in both microscopy disputes.
\begin{table}[htbp]
\centering
{\scriptsize\setlength{\tabcolsep}{2.5pt}\renewcommand{\arraystretch}{1.05}
\begin{tabular}{l r l l r r r r r l r}
\toprule
& \multicolumn{7}{c}{Precision audit} & \multicolumn{3}{c}{Oracle-silent audit} \\
\cmidrule(lr){2-8}\cmidrule(lr){9-11}
Family & prec.\ A & prec.\ B [95\% CI] & consensus [95\% CI] & $n$ & agree & $\kappa$ & adj. & missed & 95\% CI & recall \\
\midrule
Diffraction & 0.772 & 0.602 [0.44, 0.75] & 0.602 [0.44, 0.75] & 30 & 0.70 & 0.33 & 9 & 1/30 & [0.01, 0.17] & 0.967 \\
Microscopy & 0.886 & 0.820 [0.70, 0.90] & 0.886 [0.82, 0.94] & 30 & 0.93 & 0.71 & 2 & 2/30 & [0.02, 0.21] & 0.807 \\
Spectra & 0.927 & 0.608 [0.43, 0.76] & 0.608 [0.44, 0.76] & 30 & 0.63 & 0.13 & 11 & 1/30 & [0.01, 0.17] & 0.995 \\
Stability & 1.000 & 1.000 [0.84, 1.00] & 1.000 [0.84, 1.00] & 14 & 1.00 & -- & 0 & 0/29 & [0.00, 0.12] & 1.000 \\
Metallicity & 0.947 & 0.947 [0.87, 0.97] & 0.947 [0.87, 0.97] & 30 & 1.00 & 1.00 & 0 & 0/30 & [0.00, 0.11] & 1.000 \\
Magnetism & 0.922 & 0.914 [0.80, 0.97] & 0.914 [0.80, 0.96] & 30 & 0.93 & 0.63 & 2 & 0/29 & [0.00, 0.12] & 1.000 \\
\bottomrule
\end{tabular}}
\end{table}

\end{document}